\documentclass{article}

\usepackage[margin=1in]{geometry}
\usepackage{bm}
\usepackage{amsmath,amssymb}
\usepackage{amsthm}
\usepackage{graphicx,psfrag,epsf}
\usepackage{subcaption}
\usepackage{url,bm}
\usepackage{enumerate}
\usepackage{multirow}
\usepackage{color}
\usepackage[colorlinks=true,linkcolor=blue]{hyperref}
\usepackage{algorithm,algorithmicx,algpseudocode}
\usepackage{amssymb, amsmath, amsthm,mathtools}
\usepackage{cite}
\usepackage{authblk}

\newtheorem{theorem}{Theorem}[section]

\newtheorem{corollary}[theorem]{Corollary}

\newtheorem{lemma}[theorem]{Lemma}

\newtheorem{definition}{Definition}
\newtheorem{assumption}{Assumption}

\theoremstyle{remark}
\newtheorem{remark}{Remark}[section]
\newtheorem{example}{Example}
\newtheorem{conjecture*}{Conjecture}
\theoremstyle{plain}

\DeclareMathOperator*{\argmax}{argmax}

\usepackage{booktabs}

\newcommand{\R}{\mathbb{R}} 
 
\newcommand{\E}{\mathbb{E}}

\newcommand{\calP}{{\cal P}}

\newcommand{\calL}{{\cal L}}

\newcommand{\calB}{{\cal B}}
\newcommand{\calT}{{\cal T}}

\newcommand{\W}[0]{\mathcal{W}_2}

\usepackage[dvipsnames]{xcolor}

\title{Worst-case generation via minimax optimization \\
in Wasserstein space}

\author[1]{Xiuyuan Cheng}
\author[2]{Yao Xie\thanks{Corresponding author: yao.xie@isye.gatech.edu.
Authors listed alphabetically.}}
\author[2]{Linglingzhi Zhu}
\author[2]{Yunqin Zhu}

\affil[1]{{\small Department of Mathematics, Duke University}}
\affil[2]{\small H. Milton Stewart School of Industrial and Systems Engineering, Georgia Institute of Technology}

\date{\vspace{-10pt}}

\begin{document}

\maketitle

\begin{abstract}
    Worst-case generation plays a critical role in evaluating robustness and stress-testing systems under distribution shifts, in applications ranging from machine learning models to power grids and medical prediction systems. We develop a generative modeling framework for worst-case generation for a pre-specified risk, based on min-max optimization over continuous probability distributions, namely the Wasserstein space. Unlike traditional discrete distributionally robust optimization approaches, which often suffer from scalability issues, limited generalization, and costly worst-case inference, our framework exploits the Brenier theorem to characterize the least favorable (worst-case) distribution as the pushforward of a transport map from a continuous reference measure, enabling a continuous and expressive notion of risk-induced generation beyond classical discrete DRO formulations. Based on the min-max formulation, we propose a Gradient Descent Ascent (GDA)-type scheme that updates the decision model and the transport map in a single loop, establishing global convergence guarantees under mild regularity assumptions and possibly without convexity-concavity. We also propose to parameterize the transport map using a neural network that can be trained simultaneously with the GDA iterations by matching the transported training samples, thereby achieving a simulation-free approach. The efficiency of the proposed method as a risk-induced worst-case generator is validated by numerical experiments on synthetic and image data. 
    
\end{abstract}

\section{Introduction}

In many high-stakes engineering and operational systems, such as autonomous vehicles \cite{wang2021advsim}, power systems \cite{olowononi2020resilient}, and healthcare \cite{finlayson2019adversarial}, the outcomes that are most consequential for reliability, risk management, and decision-making arise not from typical observations, but from low-probability, high-impact scenarios situated at the periphery of the data-generating distribution. Such scenarios—stemming from atypical environmental conditions, irregular system configurations, or broader distributional shifts—often exert disproportionate influence on system performance. Yet, they are intrinsically difficult to observe, anticipate, or replicate through standard data-driven methods. Consequently, the ability to systematically construct informative worst-case samples is essential for rigorous stress testing, robustness certification, and the evaluation of models under meaningful but unobserved operating regimes. This need aligns with recent developments in generative modeling aimed at extreme-scenario synthesis, thereby enabling the identification of structural vulnerabilities that remain obscured under nominal data conditions.

A natural framework for representing inference or decision-making systems under such worst-case scenarios is distributionally robust optimization (DRO), which minimizes expected loss over an ambiguity set of probability distributions and thereby provides a principled mechanism for modeling distributional uncertainty in stochastic optimization. Within this paradigm, Wasserstein DRO has emerged as a popular approach due to its geometry-aware ambiguity sets, which both capture realistic data perturbations and offer strong out-of-sample guarantees. Although the theoretical properties and applications of Wasserstein DRO have been extensively studied (see, e.g., \cite{mohajerin2018data,gao2023distributionally,blanchet2019quantifying}), much less attention has been devoted to the explicit construction or sampling of worst-case distributions themselves. Motivated by this gap, and by the structure of the Wasserstein DRO formulation, we consider adversarial perturbations drawn from a Wasserstein ambiguity set as a regularized inner maximization. We formulate worst-case sample generation as a minimax optimization problem over the space of probability measures.

\paragraph{The minimax problem}
The minimax problem in Wasserstein DRO involves distributions in the Wasserstein-2 space.
Recall that on $\R^d$, the $\W$ space is denoted as $\calP_2 = \{  P \text{ on $\R^d$}, \, s.t. \int_{\R^d} \| x \|^2 dP(x) < \infty \}$.
Given a {\it reference distribution} $P$, the Wasserstein DRO problem for a radius $\delta > 0$ can be written as  
\begin{equation}\label{eq:wass-minimax-0}
\min_{\theta \in \R^p} \max_{Q \in \calB_\delta} 
  \E_{v \sim Q}  \ell( \theta, v),
  \quad \calB_\delta =\{ Q \in \calP_2,  \W(Q, P) \le \delta \},
\end{equation}
where $\theta \in \R^p$ is the parametrization of the {\it decision model} (classifier),
and  $\ell (\theta, v)$ stands for the loss function of decision model $\theta$ at sample $v$.
We introduce the Lagrangian (penalized) relaxation of the constrained problem \eqref{eq:wass-minimax-0}, 
namely, for a $\gamma > 0$, the minimax problem
\begin{equation}\label{eq:wass-minimax-1}
\min_{\theta \in \R^p} \max_{Q \in \calP_2} 
  \E_{v \sim Q}  \ell( \theta, v) - \frac{1}{ 2 \gamma} \W(P,Q)^2,
\end{equation}
where the parameter $\gamma$  controls the effective radius $\delta$ of the uncertainty set (the $\W$ ball $\calB_\delta$).
It is possible to constrain $\theta \in \Theta$ for a convex subset $\Theta$ of $\R^p$.
In this work, for simplicity, we assume $\Theta = \R^p$ throughout.

In this work, we consider $P \in \calP_2^r $, that is, $P$ is in  $\calP_2$ and has density.  
We further write $Q = T_\# P$, where $T : \R^d \to \R^d$ is the transport map that we will solve for.
Specifically, define $L^2(P):=\{ T : \R^d \to \R^d, \, \E_{x \sim P} \| T(x) \|^2 < \infty  \}$.
As shown in Proposition 2.2 of \cite{xu2024flow}, for each fixed $\theta $, the inner max problem over $Q \in \calP_2$ in \eqref{eq:wass-minimax-1} has an equivalent formulation using the $L^2$ transport map $T$ as
\[
\max_{T \in L^2(P)} \E_{x \sim P} \big[  \ell( \theta, T(x)) - \frac{1}{ 2 \gamma} \| T(x) - x\|^2  \big].
\]
The equivalence is a result of the Brenier Theorem \cite{brenier1991polar},  making use of the fact that the reference distribution $P$ has a density.
As a result, we consider the min-max problem in variables $(\theta, T)$ as
\begin{equation}\label{eq:wass-minimax-2}
\min_{\theta \in \R^p} \max_{T \in L^2(P)} L(\theta, T)
 = \E_{x \sim P} \big[  \ell( \theta, T(x)) - \frac{1}{ 2 \gamma} \| T(x) - x\|^2  \big],
\end{equation}
and we focus on solving \eqref{eq:wass-minimax-2} for a given $\gamma$ in this work.

We use the population expectation $\E_{x \sim P}$ throughout the theoretical part.
In practice, this population expectation is replaced by an empirical average over finite (training) samples, possibly via mini-batches in a stochastic optimization. Our theory will focus on using $T$ to solve Wasserstein minimax optimization and does not address the finite sample effect in the learning (the sampling complexity).

\paragraph{The double-loop approach}
Observe that, for fixed $\theta$,
the maximization over the transport map in \eqref{eq:wass-minimax-2}  separates pointwisely in $x$. 
In other words, the optimal map  $T^*(\theta)$ can be obtained by solving, for each $x \in \R^d$,
 the proximal mapping
\[
T^* (\theta)(x) := \argmax_{v \in \R^d} \big(  \ell(\theta, v) - \frac{1}{2 \gamma} \| v -x \|^2 \big), 
\]
when the maximizer uniquely exists. 
In particular, 
assume that $\ell(\theta, v)$ is $\rho$-weakly convex-weakly concave,
then the above holds at least when $\gamma < 1/\rho$: 
in this case, $T^*(\theta)(x)$ equals the unique minimizer of $
\min_{v \in \R^d} -\ell(\theta, v) + \frac{1}{2 \gamma} \| v -x \|^2,
$ namely the Moreau envelope  of $-\ell(\theta, \cdot)$,
and $\min_v$ is a strongly convex minimization on $\R^d$. 

The argument has been formalized as the {\it dual representation} of the Wasserstein DRO \cite{mohajerin2018data,blanchet2019quantifying}, which holds under more general settings. 
The result implies that \eqref{eq:wass-minimax-1} can be solved equivalently via 
\begin{equation}\label{eq:wass-minimax-moreau}
\min_{\theta} \E_{x \sim P} \max_{v} \big(   \ell(\theta, v) - \frac{1}{2\gamma} \| v -x \|^2  \big).
\end{equation}
Suppose the proximal mapping in $\R^d$ (of the Moreau envelope) can be numerically solved efficiently (by an inner-loop), 
this leads to a double-loop approach to tackle \eqref{eq:wass-minimax-moreau} or \eqref{eq:wass-minimax-1}:
for a given $\theta$, one uses the inner-loop to solve for $ \arg\max_v$ point-wisely for each $x$
(which are finite samples $x_i$ when $P$ is replaced by an empirical distribution in practice),
and then in the outer-loop one addresses the minimization
$
\min_\theta \phi(\theta) := \E_{x \sim P} \max_{v} 
	\big(   \ell(\theta, v) - \frac{1}{2\gamma} \| v -x \|^2  \big)$.
Under proper conditions, say for a small enough $\gamma$, one has
\[
\partial \phi(\theta) = \E_{x \sim P}  \partial_\theta \ell( \theta, v^*(x)),
\]
which provides first-order information about $\theta$ for the outer-loop minimization via the solution of the inner loop. 
Such an approach was investigated in \cite{sinha2018certifying},
and in view of the minimax problem \eqref{eq:wass-minimax-2}, this method ``eliminates'' the variable $T$ in computation and reduces the minimax problem into a minimization. 
 
In practice, as long as the inner loop can be solved efficiently, then one can expect the algorithm to be efficient.
However, there are two potential limitations: 
a) Finding the proximal mapping of the Moreau envelope may be numerically challenging, e.g., there can be multiple minimizers when $\gamma$ is not small. If restricted to a small $\gamma$ in usage, then the method can only handle local/small perturbations of $Q$ around $P$;
b) The method solves the mapping $T(x)$ on (training) samples only, 
and after optimization, the generation of samples from the worst-case distribution $Q$ via $T_\# P$
does not generalize to test samples. 
A related issue is that, for many problems, when the reference distribution $P$ has a continuous density, one would expect the worst-case $Q$ to be also induced from a continuous transport map $T$, and such continuity is not leveraged in this purely on-sample algorithm. 

\paragraph{Goal and contribution}
In this work, we aim to develop single-loop algorithms for solving \eqref{eq:wass-minimax-2} and establish their convergence rates, leveraging the transport map $T$-based formulation.
Our method is inspired by schemes like GDA for the vector-space minimax optimization, and  we will show that 

\vspace{5pt}
a) The proposed method can handle both small and large $\gamma$, and then allow $Q$ to deviate further from $P$. 
Our convergence theory proves finding $\varepsilon$-stationary points (Definition \ref{def:eps-stationary}) of $L(\theta, T)$ in $O(\varepsilon^{-2})$ iterations under various settings, and the rate holds for all values of $\gamma$. 
This type of result means that the algorithm will find a local saddle (stationary) point solution, say when $\gamma$ is large.  

\vspace{5pt}
b) While our theory-backed scheme is still on samples only, we propose a neural network parametrization of the continuous transport map $T$ that can be trained simultaneously with the on-sample minimax optimization. 
This neural transport map is trained using an $L^2$ matching loss in a teacher-student/distillation fashion, which is scalable to high-dimensional data. We demonstrate that our method provides a parametrized $T$ that can generalize to test samples in experiments. 

\vspace{5pt}

In our theory, the decision variable $\theta$ and the transport map $T$ play distinct and asymmetric roles, so we consider the $\min_\theta \max_T$ and the $\max_T \min_\theta$ problems (which differ) respectively and prove the convergence of the proposed schemes in different settings.
Our analysis follows the two-scale GDA framework \cite{lin2020gradient,yang2022faster}, where the ``inner" variable is always the ``fast" variable. 
In the nonconvex-strongly-concave (NC-SC) or nonconvex-PL (NC-PL) case, we will show that GDA converges at $O(\varepsilon^{-2})$ iteration complexity.

We also consider a nonconvex-nonconcave (NC-NC) case where $\theta$ is the fast variable, 
and introduce a one-sided proximal point method (PPM) which achieves $O(\varepsilon^{-2})$ iteration complexity under an interaction dominant condition. 
The PPM scheme is inspired by the analysis for the vector-space min-max problem \cite{grimmer2023landscape}, but our setting differs in that $T$ is in the $L^2$ space and the objective involves expectation over the sample distribution $P$.
In particular, we have a new notion of interaction-dominant condition after averaging over $P$. 
We will always assume that the loss function $\ell( \theta, v)$ is smooth (Assumption \ref{assump:l0-smooth-ell}),
and additional assumptions will be introduced under each setting, respectively. 

Given finite samples, the proposed algorithm can be implemented in a stochastic fashion via batches,
 both in the (on-sample) minimax iteration scheme  (which we call a ``particle optimization'')
and in the training of the neural transport map. 
In experiments, we apply the proposed GDA scheme to solve distribution robust learning on simulated and image data.
On simulated 2D data, we consider a regression loss $\ell(\theta, x)$,
and on image data (MNIST and CIFAR-10), we use cross-entropy loss involving a neural network classifier. 
The proposed GDA method numerically converges in various settings, and our model gives promising performance on generating worst-case distributions.

\subsection{Related works}

\paragraph{Minimax optimization}
The study of minimax problems dates back to von Neumann's minimax theorem for zero-sum games.
A modern convex-analytic perspective on saddle point problems emerged from the foundational work of Rockafellar \cite{rockafellar1970convex} and was later formalized through the framework of monotone variational inequalities (VIs) \cite{rockafellar1976monotone}. Within this framework, the extragradient method~\cite{korpelevich1976extragradient} has become a basic tool for solving monotone variational inequalities and convex-concave saddle-point problems. It was later extended to general settings through the mirror descent \cite{nemirovski2004prox}, and further refined via the dual extrapolation method~\cite{nesterov2007dual}.
These methods treat minimax optimization symmetrically through the lens of monotone operator theory and attain optimal $\mathcal{O}(\varepsilon^{-1})$ iteration complexity for finding the solution of convex-concave problems, matching the lower bound established by \cite{zhang2022lower}.
A closely related approach is the optimistic gradient descent--ascent (OGDA) method~\cite{rakhlin2013online,daskalakis2018training,mertikopoulos2019optimistic}, which can be interpreted as a single-call approximation of the extragradient step and achieves the same optimal rate with only one gradient evaluation per iteration. A sharp and unified analysis in~\cite{mokhtari2020convergence} established optimal $\mathcal{O}(\varepsilon^{-1})$ complexity for both extragradient and OGDA.

For nonconvex and possibly also nonconcave minimax problems, weak or local monotonicity conditions within the VI framework can still restore a form of symmetry, allowing convergence analyses via (weakly) monotone VIs~\cite{diakonikolas2021efficient,pethick2022escaping,bohm2023solving,cai2024accelerated}.
However, beyond this VI-based regime, the minimization and maximization variables become inherently asymmetric, and symmetry-based analyses no longer apply. The literature, therefore, relies on two major classes of asymmetric structural assumptions. The first class of assumptions consists of \emph{one-sided dominance conditions}, such as convexity (resp. concavity) or Polyak--\L{}ojasiewicz (PL) properties imposed on the primal (resp. dual) side. 
For nonconvex-strongly concave problems, vanilla GDA achieves an $\mathcal{O}(\varepsilon^{-2})$ iteration complexity for finding an $\varepsilon$-stationary point~\cite{lin2020gradient}, matching known first-order lower bounds~\cite{carmon2020lower,li2021complexity,zhang2021complexity}.
In contrast, for general nonconvex-concave problems, the inner maximization induces nonsmoothness of the value function, and GDA may oscillate even on bilinear instances. Using diminishing step sizes stabilizes the iterates but results in a suboptimal ${\mathcal{O}}(\varepsilon^{-6})$ complexity~\cite{jin2020local,lin2020gradient,lu2020hybrid}. To improve convergence rates, smoothing and extrapolation techniques were introduced in~\cite{xu2020unified,zhang2020single}, reducing the complexity to ${\mathcal{O}}(\varepsilon^{-4})$ for general nonconvex-concave problems.
A more structured line of nonconvex-nonconcave work imposes the one-sided PL condition~\cite{polyak1964gradient}, a classical tool for proving linear convergence in smooth minimization~\cite{karimi2016linear}. Under this assumption, \cite{nouiehed2019solving} developed a multi-step variant of GDA that explicitly solves the inner maximization more accurately, while \cite{doan2022convergence} introduced a more practical single-loop two-timescale GDA that avoids such inner loops.
Building on these ideas, \cite{yang2022faster} incorporated smoothing into the dual update and demonstrated that the resulting smoothed GDA remains effective under the PL condition, also obtaining $\mathcal{O}(\varepsilon^{-2})$ complexity as the optimal one in the nonconvex-strongly concave case. Recent developments generalize these results to the broader exponent of the PL property, allowing nonunique inner maximizers and nonsmooth objectives~\cite{li2025nonsmooth,zheng2023universal}. 

Another line of work imposes the so-called \emph{$\alpha$-interaction dominant conditions}, which characterize how the curvature of the interaction term shapes the landscape of the saddle envelope; see, e.g.,~\cite{attouch1983convergence,grimmer2023landscape}. 
Under these conditions, the behavior of the minimax problem is governed by two regimes: the interaction-dominant regime and the interaction-weak regime. 
\cite{grimmer2023landscape} analyzed the damped proximal point method and showed that in the interaction dominant regime, convergence can be obtained under a one-sided dominance condition. In contrast, in the interaction weak regime, the method achieves a local rate of $\mathcal{O}(\log(1/\varepsilon))$. 
To improve computational efficiency, \cite{hajizadeh2024linear} replaced the heavy proximal subproblem in~\cite{grimmer2023landscape} with a single damped extragradient step with the same iteration complexity.

\paragraph{Distributionally robust optimization}

Conventional DRO methods define ambiguity sets using parametric constraints, such as moment conditions \cite{bertsimas2000moment,delage2010distributionally} or deviation-based measures \cite{chen2007robust}. More recent works have shifted toward nonparametric, discrepancy-based ambiguity sets \cite{ben2013robust,namkoong2016stochastic,wang2016likelihood,mohajerin2018data,blanchet2019quantifying,gao2023distributionally}, using a tunable radius to quantify the level of uncertainty. These sets more accurately capture distributional deviations and often lead to less conservative and more computationally tractable formulations. Although the theory and applications of DRO have been extensively developed (see, e.g., \cite{shapiro2017distributionally,rahimian2022frameworks,kuhn2024distributionally}) and efficient solvers have recently been proposed \cite{liu2025dro}, comparatively less attention has been given to the explicit generation of worst-case samples. 

To construct worst-case samples, existing methods often rely on dual formulations or perturbations of nominal support points. In particular, by leveraging semi-infinite duality \cite{shapiro2001duality}, many infinite-dimensional minimax problems can be reformulated as finite, tractable programs when the reference distribution is discrete \cite{namkoong2016stochastic,mohajerin2018data}. Such strong duality results have also been extended to general optimal-transport costs \cite{blanchet2019quantifying,gao2023distributionally}, allowing the worst-case distribution to be obtained by solving a dual problem. 
\cite{sinha2018certifying,blanchet2022optimal} proposed to solve a Lagrangian-regularized, optimal-transport-based DRO formulation using stochastic gradient methods, computing pointwise perturbations in the discrete setting. 

However, discrete worst-case distributions are ill-suited for the type of worst-case sample generation we seek.
Their construction is computationally expensive--computing the Wasserstein adversary typically entails solving a large linear program that scales poorly with dataset size $n$, making such methods practical only for small $n$.
Moreover, because the worst-case distribution is supported only on the training data, it lacks generalization ability: the resulting adversarial samples are effectively confined to perturbations of training points. Existing pointwise methods \cite{sinha2018certifying} also require retraining for each new input, making them unsuitable for generating worst-case samples for unseen data.

\paragraph{Optimization in probability space}

The seminal work of Amari \cite{amari2008information,amari2016information} introduced information geometry on manifolds of probability distributions, providing a geometric foundation for optimization and convex analysis in probability space. While early developments primarily focused on continuous-time flows, more recent advances have bridged the gap to discrete-time algorithms by interpreting sampling as optimization over the space of measures \cite{wibisono2018sampling}. This line of work is closely connected to optimization equipped with Wasserstein geometry, whose foundations lie in the theory of gradient flows in metric spaces \cite{jordan1998variational,ambrosio2008gradient}. Building on these geometric insights, a variety of algorithmic frameworks have been proposed, including Stein variational methods, Wasserstein proximal and information-gradient schemes, and mirror-descent-type algorithms \cite{liu2017stein,salim2020wasserstein,wang2020information,kent2021modified,bonet2024mirror}.
In particular, \cite{kent2021modified} analyzed a Frank-Wolfe algorithm in Wasserstein space, where each linear minimization subproblem is insightfully formulated as a linear program over a Wasserstein ball, a class of problems that has been extensively studied in the DRO literature \cite{mohajerin2018data,gao2023distributionally,yue2022linear}. 

Recent works have also operationalized Wasserstein gradient flows for generative modeling, for instance by employing the JKO schemes \cite{cheng2024convergence}. 
In the context of robust optimization, \cite{xu2024flow} proposed a flow-based formulation for DRO that explicitly characterizes the worst-case distribution via transport maps. 
Our work aligns with this transport-based perspective but focuses on a minimax formulation tailored to worst-case sample generation. In a related direction, \cite{zhu2024distributionally} analyzed a regularized Wasserstein DRO formulation which directly optimizes over the distribution variable in the $\W$ space.
In contrast, our work leverages the transport-map formulation
 corresponding to the Monge formulation of the OT. 

Parallel to the deterministic map approach, another line of research tackles distributional minimax problems via mean-field games and Langevin dynamics \cite{domingo2020mean,nitanda2022convex,liu2025convergence}, relying on stochastic particle diffusion. 
Our problem here differs in that the minimax problem is over a decision variable (a vector) 
and a distributional variable (the worst-case distribution).
While our algorithm on finite data samples also utilizes particle optimization (Section \ref{sec:algorithm-practice}), 
the optimization is essentially a GDA dynamic using first-order information of the loss landscape, 
and there is no sampling nor SDE simulation.

\paragraph{Neural adversarial learning} 
Perturbing input samples--either individually or through local adversarial distributions--has long been studied in the adversarial robustness literature. 
Foundational attack methods include the fast gradient method \cite{goodfellow2014explaining} and its multi-step extensions \cite{madry2018towards}, which together motivated framing robustness as a minimax optimization problem over worst-case perturbations.
On the defense side, adversarial training \cite{madry2018towards} remains the most widely used approach. 
Subsequent works have explored enriching the inner maximization by optimizing over distributions of perturbations rather than pointwise attacks, as in adversarial distributional training \cite{dongAdversarialDistributionalTraining2020}. More recent developments introduce deep generative models, 
such as generative adversarial networks (GAN) \cite{goodfellow2014generative}
or diffusion models  \cite{ho2020denoising,song2021scorebased},
to parameterize the space of allowable perturbations and to produce semantically meaningful adversarial variations \cite{songConstructingUnrestrictedAdversarial2018a, xueDiffusionBasedAdversarialSample2023}.
Overall, these works focus on adversarial examples or local perturbation distributions around individual samples, rather than on global adversarial data distributions as considered in DRO. 

Deep models have also been used under the DRO framework,
and several works parametrize the adversarial (worst-case) distributions by neural networks.
For example,  \cite{michelModelingSecondPlayer2021a} represents the adversarial density via a neural generative model obtained by relaxing the inner optimization problem, thereby allowing more flexible ambiguity sets than classical analytical choices such as $f$-divergence balls; \cite{wenDistributionallyRobustOptimization2025} introduces diffusion-model-based ambiguity sets, in which the worst-case distribution is modeled through a diffusion generative process, and analyzes the convergence of the resulting algorithm.
In contrast, in this work, we use a neural network only as an implementation tool to approximate the transport map for out-of-sample generation, 
while the theoretical formulation and convergence analysis are carried out entirely for the Wasserstein uncertainty set without any neural parametrization.
Our analysis takes place in the functional space of distributions (via transport maps), which allows us to establish convergence guarantees for the minimax optimization in Wasserstein space without incurring neural-network-specific complications such as approximation error, architectural bias, or nonconvex training dynamics.

\section{Min-max problem with $T$ fast}\label{sec:T-fast}

In this section, we specify the GDA-type update of the minimax problem $\min_{\theta} \max_{T} L(\theta, T)$ as in \eqref{eq:wass-minimax-2}, and prove the convergence of the GDA scheme.
Under an NC-PL type assumption, GDA finds an $\varepsilon$-stationary point  (Definition \ref{def:eps-stationary})
in $O(\varepsilon^{-2})$ iterations. 
All proofs are in Appendix \ref{app:proofs}.

\subsection{The GDA update}

We want to solve the minimax problem \eqref{eq:wass-minimax-2}, and recall that 
\[
L(\theta, T)
 = \E_{x \sim P} \big[  \ell( \theta, T(x)) - \frac{1}{ 2 \gamma} \| T(x) - x\|^2  \big]. 
\]
To develop a first-order approach to solve the minimax problem, we define the gradient of the variable $\theta$ and $T$ of the objective $L$. 
We assume that the loss function $\ell(\theta, v)$ is differentiable on $\R^p \times \R^d$, see Assumption \ref{assump:l0-smooth-ell} below.
Since $\theta$ is a vector, the gradient of $L$ with respect to $\theta$ is straight forward
\begin{equation}\label{eq:partial_theta_L}
\partial_\theta L( \theta, T) =  \E_{x \sim P} \partial_\theta  \ell( \theta, T(x));
\end{equation}
The gradient with respect to the $L^2$ vector field $T$ calls for considering the metric in the functional space $L^2(P)$, and here a natural choice is the $L^2$ distance. 
We use the notation $\partial_T$ to stand for this functional derivative, and we have 
\begin{equation}\label{eq:partial_T_L}
\partial_T L( \theta, T) (x) =   \partial_v  \ell( \theta, T(x)) - \frac{1}{\gamma} (T(x) - x), \quad P-a.s.
\end{equation}
namely $\partial_T L( \theta, T)$ is also a vector field in $L^2(P)$.
Then, starting from some initial value $(\theta_0, T_0)$, we have the GDA scheme
\begin{equation}\label{eq:GDA-1}
\begin{cases}
\theta_{k+1}  	\leftarrow  \theta_k  - \tau \partial_\theta L( \theta_k, T_k), \\
T_{k+1} 		\leftarrow T_k 	  	+  \eta \partial_T L( \theta_k, T_k), 
\end{cases}
\end{equation}
where $\tau, \eta >0$ are the step sizes.
Substituting the expressions of the gradients \eqref{eq:partial_theta_L}\eqref{eq:partial_T_L}, we have
\begin{equation}\label{eq:GDA-2}
\begin{cases}
\theta_{k+1}  	\leftarrow  \theta_k  - \tau \E_{x \sim P} \partial_\theta  \ell( \theta_k, T_k(x)),\\
T_{k+1}(x) 	\leftarrow T_k(x) 	 +  \eta \big[  \partial_v  \ell( \theta_k, T_k(x)) - \frac{1}{\gamma} (T_k(x) - x) \big]. 
\end{cases}
\end{equation}

\subsection{Assumptions and Polyak-{\L}ojasiewicz condition in $T$}\label{subsec:T-fast}

We will first assume that the loss function $\ell(\theta, x)$ is smooth (meaning having Lipschitz gradient) over the Euclidean space.

\begin{assumption}[$l_0$-smoothness of $\ell$]\label{assump:l0-smooth-ell}
$\ell( \theta, v)$ is $C^1$ on $\R^p \times \R^d$ and  $l_0$-smooth, i.e.,
$\nabla \ell $ is $l_0$-Lipschitz: 
denote $z = (\theta , v)$ and
$\nabla \ell = ( \partial_\theta \ell, \partial_v \ell )$, 
\begin{equation}
\| \nabla \ell( \tilde z) - \nabla \ell( z) \|
\le l_0 \| \tilde z - z \|,
\quad \forall \tilde z, z \in \R^p \times \R^d.
\end{equation}    
\end{assumption}
We will study the convergence of the GDA scheme \eqref{eq:GDA-2} or equivalently \eqref{eq:GDA-1}.
Define 
\[ l: = l_0 + 1/\gamma.
\]
We first establish that $L$ as a functional on $\R^p \times L^2(P)$ is $l$-smooth.
Technically, we introduce the following definition of $\alpha$-smoothness on $ \R^p \times L^2(P) $ in a coordinate sense.

\begin{definition}\label{def:coord-L-smoothness-M}
For $\alpha \ge 0$,
we say a differentiable functional $M(\theta, T)$ is coordinate $\alpha$-smooth on $ \R^p \times L^2(P) $ if 
\begin{align*}
 \| \partial_\theta M (\tilde \theta, T) - \partial_\theta M (\theta, T)  \| & \le \alpha \| \tilde \theta - \theta\|,  
 	& \forall T \in L^2(P), \forall \tilde \theta, \theta \in \R^p, \\
 \| \partial_\theta M ( \theta, \tilde T) - \partial_\theta M (\theta, T)  \| & \le \alpha \| \tilde T - T\|_{L^2(P)}, 
 	& \forall \tilde T, T \in L^2(P), \forall  \theta \in \R^p,  \\
 \| \partial_T M (\tilde \theta, T) - \partial_T M (\theta, T)  \|_{L^2(P)} & \le \alpha \| \tilde \theta - \theta\|, 
 	&  \forall  T \in L^2(P), \forall \tilde \theta, \theta \in \R^p,  \\
 \| \partial_T M ( \theta, \tilde T) - \partial_T M (\theta, T)  \|_{L^2(P)} & \le \alpha \| \tilde T - T\|_{L^2(P)},
 	& \forall \tilde T, T \in L^2(P), \forall  \theta \in \R^p.
\end{align*}
\end{definition}

\begin{remark}\label{rk:l-smoothness}
It is possible to define $\alpha$-smoothness of $M$ via  requiring 
\begin{equation}\label{eq:def-L-smoothness}
\|
 \begin{bmatrix}
 \partial_\theta M ( \tilde \theta, \tilde T ) -  \partial_\theta M ( \theta, T )  \\
 \partial_T M ( \tilde \theta, \tilde T ) -  \partial_T M (  \theta, T) 
\end{bmatrix}\|
\le 
 \alpha
\| \begin{bmatrix}
\tilde \theta - \theta \\
\tilde T- T 
\end{bmatrix}\|, 
\quad 
\forall \tilde T, T \in L^2(P), \forall \tilde \theta, \theta \in \R^p,
\end{equation}
where on $ \R^p  \times L^2(P) $ we define the concatenated 2-norm as
$
 \|
 \begin{bmatrix}
\theta \\
 T
\end{bmatrix}\|
: = (  \| \theta \|^2 + \| T\|_{L^2(P)}^2  )^{1/2}.
$
It is direct to see that $\alpha$-smoothness of $M$ implies coordinate $\alpha$-smoothness,
and on the other hand,
one can verify that when $M$ is coordinate $\alpha$-smooth then it is also $(2\alpha)$-smooth. 
Thus, the notion of coordinate $\alpha$-smoothness in Definition \ref{def:coord-L-smoothness-M} is equivalent to the $\alpha$-smoothness in \eqref{eq:def-L-smoothness}, up to a constant.
\end{remark}

\begin{lemma}\label{lemma:theta-fast-lsmooth-H}
Under Assumption \ref{assump:l0-smooth-ell}, 
$L( \theta, T)$  is  coordinate  $l$-smooth on $\R^p \times L^2(P) $.
\end{lemma}

We now come to the Polyak-{\L}ojasiewicz (PL) condition in $T$.
Pointwisely for $P$-a.s.  $x$, consider
\[ 
\max_v h(\theta, v; x): = \ell(\theta, v) - \frac{1}{2\gamma} \| v-x\|^2.
\]

\begin{assumption}[$\mu$-PL in $T$]\label{assump:PL-ptwise-in-T}
 For $\mu >0$, 
$\forall \theta \in \R^p$, 
for $P$-a.s.  $x$,
$\max_v h(\theta, v; x) = h^*(\theta;x)$ is finite and obtained at a unique maximizer $v^*(\theta; x)$,
and the maximization is $\mu$-PL in $v$, i.e.,
\[
\frac{1}{2}\| \nabla_v h(\theta, v; x) \|^2 
\ge \mu ( h^*(\theta;x)-h(\theta, v; x)),  
\quad \forall v \in \R^d.
\]
\end{assumption}

Here, we introduce a point-wise PL condition and assume a unique maximizer $v^*(\theta; x)$.
See Section \ref{sec:discuss} for a discussion on the rationale of this assumption and possible relaxations.
 Our purpose here is to define $v^*(\theta; x)$ as the maximizer $T^*(\theta)$ (Lemma \ref{lemma:T-fast-Tstar-phi}),
 and in practice we will use a neural network to parametrize the transport map $T$.
This motivates us to focus on the case that $T^*(\theta)$ as a mapping from $\R^d \to \R^d$ is regular,
 and then it is most convenient to assume the unique maximizer and the PL condition in the pointwise sense.
 We thus keep the stronger assumption to illustrate a representative case
 with simple exposition.

The constant $\mu$ in the PL condition potentially depends on $\gamma$.
We define the {\it condition number} $\kappa := l / \mu$,
which then also depends on $\gamma$.
Note that $h(\theta, v; x)$ has  a Lipschitz-continuous gradient with respect to $v$,
and then 
the pointwise $\mu$-PL condition in $v$ as in Assumption \ref{assump:PL-ptwise-in-T} implies the pointwise Error-Bound (EB) \cite[Theorem 2]{karimi2016linear}:
$\forall \theta \in \R^p$,  for $P$-a.s.  $x$,
\begin{flalign}\label{eq:EB-h-theta-v-x}
~~
\rm{(EB)}_{\mathit{v}}
~~~~~~~~~~~~~~~~~~~~~~~~~~~~
&
\| \nabla_v h(\theta, v; x) \| \ge \mu \| v - v^*(\theta; x)\|,  \quad \forall v \in \R^d.
&
\end{flalign}
The inner-maximum  $T$ is then identified as $T^*(\theta)(x) = v^*(\theta; x)$.
We also define 
\[ \phi(\theta) := \max_{T} L(\theta, T),
\]
and $ \phi(\theta) = L(\theta, T^*(\theta))$.
The following lemma summarizes the properties of $v^*(\theta; x)$ and $T^*(\theta)$.
\begin{lemma}\label{lemma:T-fast-Tstar-phi}
Under Assumptions \ref{assump:l0-smooth-ell},\ref{assump:PL-ptwise-in-T},  we have the following

i) For any $\theta \in \R^p$, 
$T^*(\theta)(x) = v^*(\theta; x)$ is 
$(\gamma \mu)^{-1}$-Lipschitz in $x$ (over a $P$-a.s. set) and is 
the unique solution in $L^2(P)$ 
of $\max_{T} L(\theta, T)$. In addition,  
\begin{flalign}\label{eq:EB-T}
~~
\rm{(EB)}_{\mathit{T}}
~~~~~~~~~~~~~~~~~~~~~~~~~~~~
&
\| \partial_T L(\theta, T)\|_{L^2(P)}
\ge 
\mu \| T - T^*(\theta) \|_{L^2(P)},
\quad \forall T \in L^2(P).
&
\end{flalign}

ii)  $\| T^*(\tilde \theta ) - T^*(\theta ) \|_{L^2(P)} \le \kappa \| \tilde \theta - \theta \|$,
 $\forall \tilde \theta, \theta \in \R^p$.

iii) $\phi$ is differentiable and $\partial \phi(\theta)  = \partial_\theta L(\theta, T^*(\theta))$.
In addition, 
let $L_1 = l (1+\kappa)$, $\phi(\theta)$ is $L_1$-smooth, 
that is, $\| \partial \phi (\tilde \theta) - \partial \phi ( \theta) \| \le L_1 \| \tilde \theta -\theta\|$,
$\forall \tilde \theta, \theta \in \R^p$.
\end{lemma}

\subsection{Convergence rate of GDA}

Following the framework of two-scale GDA \cite{lin2020gradient,yang2022faster}, in problem \eqref{eq:wass-minimax-2} $T$ is the ``inner'' variable and will be the fast variable. 
Our setting corresponds to an NC-PL type assumption on $L(\theta, T)$. 

To prove the convergence, we utilize some intermediate estimates, particularly the descent of the primal $L$ and that of $\phi$. 
The latter will follow from the Lipschitz condition of $\partial \phi$ proved in Lemma \ref{lemma:T-fast-Tstar-phi} iii).
For the former, the  coordinate  $l$-smoothness implies lower bounds of coordinate descent of $L$ over the iterations,
which is summarized in the next lemma.
We use the notation of a general functional $M$.

\begin{lemma}\label{lemma:desdent-coordinate-M}
Suppose $M(\theta, T)$ is   coordinate  $l$-smooth on $\R^p \times L^2(P)$, then

i) $\forall T \in L^2(P)$, $\forall \theta, \tilde \theta \in \R^p$, 
$| M( \tilde \theta, T) - M(   \theta, T) - \partial_\theta M(\theta, T) \cdot (\tilde \theta - \theta) |
 \le \frac{l}{2} \| \tilde \theta - \theta\|^2 $;
 
ii)  $\forall \theta \in \R^p$,   $\forall T, \tilde T \in L^2(P)$,
$| M(  \theta, \tilde T) - M(   \theta, T) - \langle \partial_T M(\theta, T), \tilde T - T \rangle_{L^2(P)} |
 \le \frac{l}{2} \| \tilde T - T \|_{L^2(P)}^2 $.
\end{lemma}

We introduce the notion of $\varepsilon$-stationary point.

\begin{definition}\label{def:eps-stationary}
For differentiable function $f(x)$ on $\R^n$, $x$ is an $\varepsilon$-stationary point of $f$ if $\| \partial f(x) \| \le \varepsilon$.
For function $F(x,y)$ on $\R^p \times \R^n$, $(x,y)$ is an $(\varepsilon, \varepsilon')$-stationary point of $F$ if 
$\| \partial_x F(x,y) \| \le \varepsilon$ and $\| \partial_y F(x,y) \| \le \varepsilon'$.
\end{definition}

The main theorem of this section below shows that GDA not only finds an $\varepsilon$-stationary point of $\phi$,
but  it is also an  $(\sqrt{2}\varepsilon, \varepsilon/\kappa)$-stationary point of $L$.

\begin{theorem}[$T$ fast NC-PL]
\label{thm:rate-T-fast}
Under Assumptions \ref{assump:l0-smooth-ell},\ref{assump:PL-ptwise-in-T}, 
suppose $\phi(\theta)= \max_{T \in L^2(P)} L(\theta, T)$ has a finite lower bound.
Then, for some $\eta \sim 1/l$, $\tau \sim 1/\kappa^2 l$, 
the GDA  scheme \eqref{eq:GDA-1}
finds $\theta_k$ 
for some $k \le K = O(\kappa^2 l /\varepsilon^2)$
s.t. 
i) $ \| \partial \phi(\theta_k) \| \le \varepsilon$,
and 
ii)  $\| \partial_\theta L ( \theta_k, T_k) \|\le \sqrt{2} \varepsilon$,
$\| \partial_T L ( \theta_k, T_k) \|_{L^2(P)}  \le \varepsilon/\kappa$.
\end{theorem}

The choice of $\eta, \tau$ as well as the constant factor in the big-O notation will be clarified in the proof.
In particular, the constant in big-O  equals $ \Delta_\phi$ multiplied by an absolute constant,
where  $\Delta_\phi = (\phi(T_0) - \phi^*) + \frac{1}{8}(\phi(\theta_0) - L(\theta_0, T_0))$ depends on the initial value of the iterations, $\phi^*$ being the lower bound of $\phi$.

As shown in Theorem \ref{thm:rate-T-fast} and also later theorems in Section \ref{sec:theta-fast} when $\theta$ is the fast variable, our theoretical choice of step sizes $\eta$ and $\tau$ also sheds light on the choice in practice:
with $1/l = 1/( 1/\gamma + l_0)$,
the fast variable will have a step size $1/l$
and the slow variable will have step size $\lambda/l$, where $\lambda < 1$ is theoretically proportional to $1/\kappa^2$. 
If there is an estimate of $l_0$ of the loss function, one can insert that together with the value $\gamma$ in the expression of $1/l$. This is useful when one is to solve the same robust learning problem with multiple values of $\gamma$, corresponding to different radii of the Wasserstein uncertainty ball. We also note that empirically, the GDA may converge for a range of values of $(\eta, \tau)$, and the decrease of the gradient norm of the ``slow variable" (for which we set the smaller step size) is not necessarily slower.
This suggests that the fast/slow variable separation could be a theoretical construct, and practical GDA may converge without adhering to a strict two-scale regime.

\begin{remark}[Small and large $\gamma$]
Our theory, here and also in Section \ref{sec:theta-fast} below, 
applies to all values of $\gamma$:
when $\gamma $ is small, that is, $1/\gamma > \rho $, the point-wise $\max_v h(\theta, v; x)$ will be strongly concave  and then $\mu = 1/\gamma - \rho$;
When $1/\gamma \le \rho$, we no longer have strong concavity, and our theory applies by assuming the PL condition on the maximization.
In this case $1/\gamma \le \rho \le l_0$ so $l = 1/\gamma + l_0 \le 2l_0$,
that is, $l$ is (up to a constant) of the magnitude of the smoothness of the loss function $\ell(\theta, v)$.
Algorithm-wise, compared to the elimination method that solves for $v^*(\theta, x)$ exactly using an inner-loop, GDA is a single-loop algorithm that updates $\theta$ and $T$ simultaneously (or alternatively). Our theory is for synchronized GDA, but the extension to alternative GDA should be direct with minor modifications. 
\end{remark}

\section{Max-min problem with $\theta$ fast}\label{sec:theta-fast}

In this section, we consider the max-min problem $\max_T \min_\theta L( \theta, T)$, where $\theta$ is the inner (fast) variable. 
We will prove two results: 
one under an NC-SC assumption, where we show convergence of GDA achieving the $O(\varepsilon^{-2})$ iteration complexity;
one under an interaction dominant condition in the NC-NC setting, where we propose a one-sided PPM scheme and prove the convergence at  $O(\varepsilon^{-2})$ number of proximal point iterations. 
All proofs are in Appendix \ref{app:proofs}.

\subsection{Recast as min-max problem}

By taking a negation of the objective, we write it as a min-max problem 
\begin{equation}\label{eq:wass-minimax-3}
\min_{T \in L^2(P)}  \max_{\theta \in \R^p} 
H( T, \theta)
= -L( \theta, T)
 = \E_{x \sim P} \big[ - \ell( \theta, T(x)) + \frac{1}{ 2 \gamma} \| T(x) - x\|^2  \big].
\end{equation}
Because we lack a convex-concave structure and the spaces $L^2(P)$ and $\R^p$ are unbounded,
\eqref{eq:wass-minimax-3} may not be equivalent to the original problem \eqref{eq:wass-minimax-2}.
However, the two share stationary points, which suffices for our purpose since our theory proves convergence to 
an $\varepsilon$-stationary point under our nonconvex-nonconcave setting. 
In applications, the max-min formulation can also have an interpretation, that is, we assume there is a corresponding classifier $\theta$ to each perturbed distribution $Q= T_\# P$, and then we search for the $Q$ within the Wasserstein uncertainty ball.

Here, we use the convention to put the inner variable $\theta$ in the second variable in the expression of $H(T,\theta)$.
This swap of variable ordering does not affect our definition of the (coordinate) $l$-smoothness in the same way as in Section \ref{subsec:T-fast}.
In particular, since $H = -L$,
Lemma \ref{lemma:theta-fast-lsmooth-H} applies to show that $H$ is  coordinate  $l$-smooth on $ L^2(P) \times \R^p$.
As a result, Lemma \ref{lemma:desdent-coordinate-M} also applies to $H$.

\subsection{NC-SC: convergence of GDA}\label{subsec:theta-fast-NC-SC}  

We introduce an NC-SC condition on $H$, namely requiring $H$ to be strongly concave in $\theta$. 
The extension to the NC-NC case will be addressed in Section \ref{subsec:theta-fast-NC-NC}. 

To proceed, we assume that $\ell$ is $C^2 $ for simplicity.
\begin{assumption}[$C^2$ $\ell$]\label{assump:l-C2}
$\ell$ is $C^2$ on $\R^p \times \R^d$.
\end{assumption}
By definition, we have
\[
- \partial_{\theta \theta}^2 H( T, \theta) 
=  \E_{ x \sim P} \partial^2_{\theta \theta} \ell (\theta, T(x)).
\]
The strongly concavity in  $\theta$ asks that 
$ - \partial_{\theta \theta}^2 H( T, \theta) \succeq
 \mu I$ for some $\mu > 0$, 
 which is weaker than $\ell( \cdot, v)$ being strongly convex in $\theta$: 
 it asks the matrix to be positive definite after {\it averaging} over ($T$-transported) data samples $x \sim P$. 
 We define a ``good'' subset of transport maps as 
 \[
\calT_{\mu}^{\rm SC}
:=\{ T \in L^2(P)  \,
s.t. \, 
\forall \theta \in \R^p, \,
\E_{ x \sim P} \partial^2_{\theta \theta} \ell (\theta, T(x)) 
\succeq
 \mu I  \}.
\]
If $T$ belongs to this good set,
then $\partial_{\theta \theta}^2 \ell( \theta, v)$ after averaging on $v \sim T_\# P$ becomes positive definite.
Intuitively, this happens when the transported density $T_\# P$ sufficiently spreads the locations $v$ where  $\partial_{\theta \theta}^2 \ell( \theta, v)$ is positive. 
Thus, the ``good'' class of $\calT_{\mu}^{\rm SC}$ can be interpreted as requiring the transport map $T$ to be not degenerate -- in view of the locations where $\partial_{\theta \theta}^2 \ell$ is positive or near positive.

\begin{example}[$L^2$ regression]\label{ex:l2-regression}
Consider the regression loss
\[
\ell( \theta , x) =  \frac{1}{2}(y(x; \theta) - y^*(x))^2,  \quad x \in \R^d,
\]
where $y^*$ is the true response, and $y_\theta = y(\cdot; \theta)$ is the model function. We consider the general linear model $y_\theta(x) = \sigma( \theta^T x)$ where $\sigma: \R\to \R$ is a nonlinear ``link'' function, such as the logistic function $\sigma(t) =1/(1+e^{-t}) $.
One can compute that 
\[
\partial_{\theta \theta}^2 \ell( \theta, v) = ( (y_\theta(v) - y^*(v)) \sigma''(\theta^T v) + \sigma'(\theta^T v)^2 ) vv^T, 
\quad \forall v \in \R^d.
\]
As a result, the condition $ - \partial_{\theta \theta}^2 H(T, \theta)  \succeq  \mu I$ is equivalent to that
 \[
- \partial_{\theta \theta}^2 H(T, \theta) = 
 \E_{v \sim T_\# P} ( (y_\theta(v) - y^*(v)) \sigma''(\theta^T v) + \sigma'(\theta^T v)^2 ) vv^T
\succeq  \mu I.
 \]
Note that $ \sigma'(\theta^T v)^2 \ge 0$, and $ (y_\theta(v) - y^*(v)) \sigma''(\theta^T v) $ may change signs.
Suppose the some of the two as a function of $v$ is always greater than $\alpha > 0$, then 
$ - \partial_{\theta \theta}^2 H(T, \theta)  \succeq \alpha \E_{v \sim T_\# P }vv^T$,
which is positive definite except for degenerated cases.
\end{example}

\

Because the functional space $L^2(P)$ contains many degenerate mappings, 
unless $\ell (\theta, v)$ is strongly convex in $\theta$ for every $v$,
we expect $\calT_{\mu}^{\rm SC}$ to be a strict subset of $L^2(P)$.
We will assume that the $T_k$ generated by the GDA scheme always lies inside this good subset. 
\begin{assumption}[SC-good-$T$]\label{assump:SC-good-T}
For $\mu > 0$, over the iterations for all $k$, $T_k \in \calT_{\mu}^{\rm SC}$.
\end{assumption}
For any $T \in \calT_{\mu}^{\rm SC}$, there is a unique $\theta^*[T] = \argmax_{\theta} H(T, \theta)$ due to $\mu$-strongly concavity of $H$ in $\theta$. 
We define
$\psi (T) :=  \max_{\theta \in \R^p}  H( T, \theta)$,
and $\psi(T) = H(T, \theta^*[T])$ when $T \in \calT_{\mu}^{\rm SC}$.
To be able to utilize the differential property of $\partial_T H$ and $\partial \psi$, 
we will need $\theta^*[T]$ to be defined on a neighborhood of $T$ in $L^2(P)$.
The set $\calT_{\mu}^{\rm SC}$ is typically closed by continuity; However, to have unique $\theta^*[T]$ it suffices to have $H(T, \cdot)$ to be $\beta$-strongly concave for some $\beta > 0$.
We  introduce another technical condition on the good set:
\begin{assumption}\label{assump:Tgood-in-open-set-SC}
There exists an open set  $\bar \calT$ in $L^2(P)$
that contains $\calT_{\mu}^{\rm SC}$,
where $ T \in \bar \calT$ implies that for some $0 <\beta < \mu $, $T \in \calT_{\beta}^{\rm SC}$.
\end{assumption}
Assumption \ref{assump:Tgood-in-open-set-SC} can be induced 
e.g. by assuming that  $\partial_{\theta \theta}^2 \ell( \theta, v)$ is Lipschitz in $v$. 
We summarize the properties of $\theta^*[T]$ and $\Phi$
 in the following lemma,
 where we use the notations of a general function $M(T,\theta)$ and a good subset $\calT$ of $L^2(P)$.
 
\begin{lemma}\label{lemma:theta-fast-Tgood-Phi}
Suppose $M(T,\theta)$ is  coordinate  $l$-smooth on $L^2(P) \times \R^p$,
and there is a set $\calT \subset L^2(P)$ satisfying that

a) For $\mu > 0$, when $T \in \calT$, $M(T, \cdot)$ is $\mu$-strongly concave on $\R^p$,

b) $\calT \subset \bar \calT$ which is open in $L^2(P)$, 
and $ T \in \bar \calT$ implies that $M(T, \cdot)$ is $\beta$-strongly concave on $\R^p$ for some $ 0 < \beta < \mu$.

 Then,

i) For any $T \in \calT$, there is a unique $\theta^*[T] = \argmax_{\theta} M(T, \theta)$,
thus the functional $\Psi (T) : =  \max_{\theta \in \R^p}  M( T, \theta) $ is finite,
$\Psi (T) = M(T, \theta^*[T])$, and
$\partial \Psi (T )(x) = \partial_T M(T, \theta^*[T])(x)$, $P$-a.s.
In addition, 
\begin{flalign}\label{eq:EB-M-T-theta}
~~
\rm{(EB)}_\theta
~~~~~~~~~~~~~~~~~~~~~~~~~~~~
&
\| \partial_\theta M(T, \theta) \| \ge \mu \| \theta - \theta^*[T]\|, \quad \forall \theta \in \R^p.&
\end{flalign}

ii) Let $\kappa = l/\mu$, $\forall  \tilde T, T  \in  \calT$, 
 $\| \theta^*[\tilde T] - \theta^*[ T ]\| \le \kappa \| \tilde T - T \|_{L^2(P)}$. 

iii) Let $L_2 = l (1+\kappa^2)$,
$\forall \tilde T, T  \in  \calT$, 
$
\Psi (\tilde T) - \Psi (T) \le 
	\langle \partial \Psi (T), \tilde T - T  \rangle_{L^2(P)}
	+ \frac{L_2}{2} \| \tilde T - T\|_{L^2(P)}^2.
$
\end{lemma}

Typically, to derive the descent of $\Psi$ as in Lemma \ref{lemma:theta-fast-Tgood-Phi} iii), 
one first establishes the  $l(1+\kappa)$-Lipschitz of $\partial \Psi$ as in Lemma \ref{lemma:T-fast-Tstar-phi} iii),
see the proof of Theorem \ref{thm:rate-T-fast}.
However, such an argument cannot apply here because we  cannot connect a line between $T$ and $\tilde T$ in $\calT$
(the latter may not be convex).
The argument in Lemma \ref{lemma:theta-fast-Tgood-Phi} iii) applies when $\calT$ can be any subset of $L^2(P)$,
and the  Lipschitz constant contains $\kappa^2$ instead of $\kappa$.

We derive a convergence theorem in the notation of a general objective $M$, that is, the min-max problem is $\min_{T \in L^2(P)} \max_{\theta \in \R^p} M(T, \theta)$, and the GDA scheme is
\begin{equation}\label{eq:GDA-M-theta-fast}
\begin{cases}
T_{k+1} 		\leftarrow T_k 	  	-  \eta \partial_T M( T_k, \theta_k) \\
\theta_{k+1}  	\leftarrow  \theta_k  + \tau \partial_\theta M( T_k, \theta_k).
\end{cases}
\end{equation}

\begin{theorem}\label{thm:theta-fast-NC-SC}
Suppose $M(T,\theta)$ and the good subset $\calT \subset L^2(P)$
satisfy the condition in Lemma \ref{lemma:theta-fast-Tgood-Phi},
and $T_k$ generated by the GDA scheme \eqref{eq:GDA-M-theta-fast} lies inside $\calT$ for all $k$.
Let $\Psi(T) = \max_{\theta \in \R^p} M( T, \theta)$
and assume that $\Psi(T)$ is lower bounded by finite $\Psi^*$.
Then, for some $\tau \sim 1/l$, $\eta \sim 1/\kappa^2 l$, 
the scheme \eqref{eq:GDA-M-theta-fast} finds $(T_k, \theta_k)$
for some $k \le K = O(\kappa^2 l/\varepsilon^2)$
s.t.
i) $\| \partial \Psi ( T_k) \|_{L^2(P)} \le \varepsilon$,
and 
ii)  $\| \partial_T M (T_k, \theta_k) \|_{L^2(P)} \le \sqrt{2} \varepsilon$,
$\| \partial_\theta M (T_k, \theta_k) \| \le \varepsilon/\kappa$.
\end{theorem}

As revealed in the proof, the constant in big-O  equals $ \Delta_0$ multiplied by an absolute constant,
where $\Delta_0 = (\Psi(T_0) - \Psi^*) + \frac{1}{8}(\Psi(T_0) - M(T_0, \theta_0))$.
The convergence of GDA is a direct application of the theorem:

\begin{corollary}[$\theta$ fast NC-SC]
\label{cor:theta-fast-NC-SC}
Under Assumptions 
\ref{assump:l0-smooth-ell},\ref{assump:l-C2},\ref{assump:SC-good-T},\ref{assump:Tgood-in-open-set-SC},
suppose $\psi(T) = \max_{\theta \in \R^p} H(T,\theta)$ has a finite lower bound.
Then, for some $\tau \sim 1/l$, $\eta \sim 1/\kappa^2 l$, 
the GDA scheme \eqref{eq:GDA-1} 
finds $(T_k, \theta_k)$
for some $k \le K = O(\kappa^2 l /\varepsilon^2)$ 
s.t.
i) $\| \partial \psi ( T_k) \|_{L^2(P)} \le \varepsilon$,
and 
ii)  $\| \partial_T H (T_k, \theta_k) \|_{L^2(P)} \le \sqrt{2} \varepsilon$,
$\| \partial_\theta H (T_k, \theta_k) \| \le \varepsilon/\kappa$.
\end{corollary}

\subsection{NC-NC: convergence of one-sided PPM}\label{subsec:theta-fast-NC-NC}

We start by assuming that the loss $\ell(\theta, v)$ is weakly-convex-weakly-concave (WC-WC) in the vector space.

\begin{assumption}[$\rho$-WC-WC of $\ell$]\label{assump:rho-WC-WC-ell}
For some $\rho > 0$,
$\forall \theta \in \R^p$, $\ell( \theta, \cdot)$ is $\rho$-weakly concave on $\R^d$;
$\forall v \in \R^d$, 
$\ell( \cdot, v)$ is $\rho$-weakly convex on $ \R^p$. 
\end{assumption}
We know that $0 < \rho \le l_0$.
Actually, since $\ell$ is $l_0$-smooth, it must be $l_0$-WC-WC.
Assumption \ref{assump:rho-WC-WC-ell} is to impose that potentially where is a better $\rho$ than $l_0$.

To extend the theory to NC-NC type, we propose to modify the $T$-step in GDA \eqref{eq:GDA-1} to be a (damped) proximal point update, while the $\theta$-step will remain explicitly first-order.
We call our scheme a ``one-sided PPM'', and we will show convergence rate under the interaction dominant condition in $\theta$.
The one-sided PPM scheme treats the two variables $\theta$ and $T$ asymmetrically, and we think this is suitable for our problem,
because the transport map $T$ and the decision variable $\theta$ 
lie in different spaces and have different roles in this game. 
For applications in robust learning, leaving the update on $\theta$ explicit will be useful when $\theta$ is a neural network containing a large number of parameters, and $m$ is possibly larger than the data dimension $d$.

Our analysis is much inspired by  \cite{grimmer2023landscape}, which studied the vector-space NC-NC type min-max problem and proved the convergence rate of a full PPM update (of both variables) under interaction-dominant conditions. 
We will only impose the interaction dominant condition on the ``inner'' max variable. In this case, 
our one-sided PPM has the same $O(\varepsilon^{-2})$ iteration complexity as the full PPM \cite{grimmer2023landscape}.
Computation-wise,  while the proximal point of $T$ still needs to be solved, this is simpler than a full PPM update in $T$ and $\theta$ jointly.
One can also approximately solve the proximal point by an extra-gradient method (EGM), see \cite{hajizadeh2024linear} as a follow-up of \cite{grimmer2023landscape}. Here, we focus on the (one-sided) PPM scheme for theoretical purpose, and we expect that computational tractability can be similarly resolved, e.g., by EGM, with a convergence guarantee by combining the techniques in our analysis and those in prior works.

To be specific, we introduce $H_s$ defined as
\[
H_s(T, \theta)  = \min_{V \in L^2(P)} H(V, \theta) + \frac{1}{2s} \| V - T\|_{L^2(P)}^2 := h(V; T, \theta),
\]
where $ 0 < s <   1/\rho$ is the step size (of the proximal map).
$H_s(\cdot, \theta)$ is a Moreau envelope of $H(\cdot, \theta)$ for fixed $\theta$.
The proximal mapping  is defined as
\begin{equation}\label{eq:def-prox-operator}
T^+ ( T ; \theta)
:= \arg  \min_{V \in L^2(P)}  h(V; T, \theta).
\end{equation}

The following lemma shows that the proximal mapping is well-defined and characterizes the differential of $H_s$:
\begin{lemma}\label{lemma:Hs-prox-diff}
Under Assumptions \ref{assump:l0-smooth-ell},\ref{assump:l-C2},\ref{assump:rho-WC-WC-ell},
for any $0< s < 1/\rho $, 
$T^+ = T^+(T , \theta)$ is well-defined, and
\begin{align}
\partial_T H_s( T,\theta) & = \frac{1}{s}(T - T^+  )  = \partial_T H( T^+ , \theta), 
	\label{eq:partial_T_Hs} \\
\partial_\theta H_s( T,\theta) &  = \partial_\theta H( T^+ , \theta). 
	\label{eq:partial_theta_Hs}	
\end{align}
\end{lemma}

With GDA step sizes $\eta$ and $\tau$, the proposed one-sided damped PPM scheme is 
\begin{equation}\label{eq:damped-PPM}
\begin{cases}
T_k^+ \leftarrow T^+ (T_k; \theta_k) \\
T_{k+1} 		\leftarrow  (1-  \frac{\eta}{s} ) T_k +  \frac{\eta}{s} T_k^+  \\
\theta_{k+1}  	\leftarrow  \theta_k + \tau \partial_\theta H(T_k^+, \theta_k) , \\
\end{cases}
\end{equation}
where the proximal operator $T^+$ is as in \eqref{eq:def-prox-operator}.
We will see that the scheme is equivalent to GDA applied to $H_s$,
and will apply Theorem \ref{thm:theta-fast-NC-SC} to prove the convergence rate. 
To proceed, we establish the coordinate smoothness of $H_s$
and its strong concavity in $\theta$ under interaction dominant condition.

\begin{lemma}\label{lemma:theta-fast-lsmooth-Hs}
Under Assumptions \ref{assump:l0-smooth-ell},\ref{assump:l-C2},\ref{assump:rho-WC-WC-ell},
for any $0< s < 1/\rho $,
$H_s$ is  coordinate  $\bar l$-smooth on $L^2(P) \times \R^p$, where
\[
\bar l := \max \Big\{ \frac{1}{s}\vee 
        \frac{|\rho - 1/\gamma|}{ 1 - s ( \rho - 1/\gamma) },  \,
	 \frac{l_0}{ 1 - s ( \rho - 1/\gamma) },  \,
	l_0 \big(1 +   \frac{s l_0  }{ 1 - s ( \rho - 1/\gamma) } \big)  \Big\}.
	\]
\end{lemma}

The ``good'' subset of transport maps is now
 \begin{align*}
\calT_{\bar \mu}^{\rm NC}
 :=\{ 
T \in L^2(P)  \,
& s.t. \, 
\forall \theta \in \R^p, \, \text{let } T^+= T^+( T, \theta), \\
& \E_{ x \sim P} \big[ \partial^2_{\theta \theta} \ell 
			+ \partial^2_{\theta v} \ell 
				(  (s^{-1}+\gamma^{-1})I_d - \partial^2_{v v} \ell )^{-1} 
				\partial^2_{v \theta} \ell 
 			\big]|_{(\theta, T^+(x))}
	\succeq
	\bar \mu I  \}.
\end{align*}
The constant $\bar \mu$ potentially depends on $s$.

\begin{lemma}\label{lemma:TNC-SC-Hs}
Under Assumptions  \ref{assump:l0-smooth-ell},\ref{assump:l-C2},\ref{assump:rho-WC-WC-ell},
if for a pair of $0< s < 1/\rho $ and $\bar \mu > 0$,  $T \in \calT_{\bar \mu}^{\rm NC}$,
then $H_s(T, \cdot)$ is $\bar \mu$-strongly concave in $\theta \in \R^p$, that is,
$-\partial_{\theta \theta}^2 H_s( T, \theta) \succeq \bar  \mu I$ for all $\theta \in \R^p$.
\end{lemma}

\begin{assumption}[NC-good-$T$]\label{assump:NC-good-T}
For $\bar \mu > 0$, over the iterations for all $k$, 
$T_k \in \calT_{\bar \mu}^{\rm NC}$.
\end{assumption}

\begin{assumption}\label{assump:Tgood-in-open-set-NC}
There exists an open set  $\bar \calT$ in $L^2(P)$
that contains $\calT_{\bar \mu}^{\rm NC}$,
where $ T \in \bar \calT$ implies that for some $0 < \beta < \bar \mu$,
$T \in \calT_{\beta}^{\rm NC}$.
\end{assumption}

Define $
\bar \kappa : = \bar l / \bar \mu $, we are ready to prove the convergence rate of the one-sided PPM scheme.

\begin{corollary}[$\theta$ fast NC-NC]
\label{cor:theta-fast-NC-NC}
Under Assumptions  \ref{assump:l0-smooth-ell},\ref{assump:l-C2},\ref{assump:rho-WC-WC-ell},\ref{assump:NC-good-T},\ref{assump:Tgood-in-open-set-NC},
suppose  $0 < s < 1/\rho$, and $\psi_s(T) = \max_{\theta \in \R^p} H_s(T,\theta)$ has a finite lower bound.
Then, for some $\tau \sim 1/ \bar l$, $\eta \sim 1/ \bar\kappa^2 \bar l$, 
the T-only damped PPM scheme \eqref{eq:damped-PPM} 
finds $(T_k, \theta_k)$
for some $k \le K = O(\bar \kappa^2 \bar l /\varepsilon^2)$ 
s.t.
i) $\| \partial \psi_s ( T_k) \|_{L^2(P)} \le \varepsilon$,
and 
ii)  $\| \partial_T H_s (T_k, \theta_k) \|_{L^2(P)} \le \sqrt{2} \varepsilon$,
$\| \partial_\theta H_s (T_k, \theta_k) \| \le \varepsilon/\bar \kappa$.
\end{corollary}

Note that when ii) holds, i.e.,
 $(T_k, \theta_k)$ is an $(\sqrt{2}\varepsilon, \varepsilon/\bar \kappa)$-stationary point of $H_s$,
 then, by Lemma \ref{lemma:Hs-prox-diff}, 
 $(T_k^+, \theta_k)$ is an $(\sqrt{2}\varepsilon, \varepsilon/\bar \kappa)$-stationary point of $H$.

\section{Algorithm in practice}\label{sec:algorithm-practice}

In the previous two sections, 
we proposed GDA-type schemes with convergence analysis,
where we focused on population (infinite-sample) formulation of the loss, and the variable $T$ is in $L^2(P)$ space.
In practice, we would like to solve a finite-sample optimization problem towards inferring the solution $(\theta^*, T^*)$ corresponding to the infinite-sample problem.
While we postpone a theoretical analysis of the finite-sample effect, in this section, we explain how our algorithms are implemented on finite samples.
We also introduce a neural network parametrization of the transport map $T$ that can be efficiently trained simultaneously with our minimax solvers, both of which can be computed on batches of samples.

\subsection{Particle optimization on finite samples}\label{subsec:particle-opt-algo}

Given finite samples $\{ x_i\}_{i=1}^n$, $x_i \sim P$ i.i.d., our algorithm has two components, 
i) solve the minimax problem on the $n$ samples,
and ii) to learn a neural network  $T_\varphi$ that parametrizes $T$.
Because in i), we will solve $T(x_i)$ on the $n$ samples only,
the learned $T_\varphi$ will allow us to extrapolate to a test sample that is not in the $n$ (training) samples. 
Below we introduce the two components, i) and ii) respectively. 

In i), we aim to solve for $\theta$ and $v_i = T(x_i)$, that is, we only solve for the particles $v_i$, $i=1,\cdots, n$, as images of $x_i$ under the transport $T$. 
We call this part the ``particle optimization".
Below we use superscript $^k$ to stand for iteration index $k$,
that is, we will solve for $v_i^k = T^k(x_i)$ and $\theta^k$.

In each iteration, we load a random batch $B$ of samples, $|B| = m$. 
If we load all the samples,
i.e. $B = [n]:=\{1, \cdots, n \}$, then $m=n$, and otherwise, our algorithm is stochastic. 
We give the finite-sample version of our GDA scheme  below:
suppose step sizes $\eta$, $\tau$  are given, 
the GDA scheme \eqref{eq:GDA-1} or equivalently \eqref{eq:GDA-2} becomes
\begin{eqnarray}
\label{eq:vi-GDA-3}
v_i^{k+1} &\leftarrow &v_i^k +  \eta \big(  \partial_v  \ell( \theta^k, v_i^k) - \frac{1}{\gamma} (v_i^k - x_i) \big), \quad i\in B; \\
\label{eq:theta-step-GDA}
\theta^{k+1}&  	\leftarrow & \theta^k  - \tau  \frac{1}{m}\sum_{i \in B} \partial_\theta  \ell( \theta^k, v^k_i).
\end{eqnarray}
Note that in \eqref{eq:vi-GDA-3}, 
we let $v_i^{k+1} \leftarrow v_i^k$ for $i \notin B$,
that is, $v_i^{k}$ is only changed if the $i$-th particle is loaded in batch $B$, because here $k$ stands for the number of batches.
Some popular modifications of GDA can be derived. 
For example, if one replaces $v_i^k$ with $v_i^{k+1}$ in \eqref{eq:theta-step-GDA}, then this gives an Alternating GDA update. 
One can also derive the on-sample version of the one-sided PPM scheme \eqref{eq:damped-PPM}, that is, the update on $v_i^{k+1}$ and $\theta^{k+1}$ given $\{ \theta^k, v_i^k\}$,
and further approximate it by an extra-gradient method, similarly to \cite{hajizadeh2024linear}.
Generally, one may consider any first-order update for vector-space minimax optimization here
by updating $T$ on the particles, i.e., to update $v_i^k$.

\subsection{Neural transport map and learning by matching}

We propose to parametrize $T$ by $T_\varphi: \R^d \to \R^d$, where $\varphi$ is the network parameters.
During the $k$-th iteration of the particle optimization, 
we have updated pairs of $\{x_i, v_i = v_i^k\}_i$ on a batch $B$, 
and then we can compute a ``matching loss'' defined as 
\begin{equation}\label{eq:loss-matching}
\calL( \varphi; \{ x_i, v_i\}) = 
 \frac{1}{m}\sum_{i \in B} \|  T_\varphi(x_i) - v_i\|^2.
\end{equation}
There is flexibility in the training  of $T_\varphi$ in this part:
one can use the same batch $B$ as in the particle (stochastic) optimization,
or, if $m$ is large,
use a smaller batch size $m'$ to update the network $T_\varphi$ in multiple steps on $B$.
We summarize our batch-based algorithm in Algorithm \ref{alg:particle-gda}, where the iteration index $k$ counts the number of batches, and each batch cycles through all $n$ samples without repetition.

Note that $\calL$ is an $L^2$ loss, and intuitively it asks the mapped image $T_\varphi(x_i)$ to match the target particle $v_i$. 
It can be interpreted as a student network to learn the output from teacher-provided pairs $\{ x_i, v_i\}$, which is similar to distillation in generative models \cite{salimans2022progressive,song2023consistency}.
The training of $T_\varphi$ is standard, and one can use SGD with momentum or Adam.
Our algorithm is efficient because the training of $T_\varphi$ does not need to wait for the particle solutions to be perfect: even when $v_i^k$ are not fully convergent yet, they can be useful as teacher output to train the neural transport map $T_\varphi$,
and the particle iterations and the matching net training can converge concurrently.

\begin{algorithm}[t]
\caption{Particle GDA + neural transport matching}\label{alg:particle-gda}
\begin{algorithmic}[1]
\Statex \textbf{Input:}
		Samples $\{x_i\}_{i=1}^n$, 
		initial decision model $\theta^0$, 
		initial transport network ${\varphi^0}$
		
\Statex \textbf{Parameters:}		
		step sizes $\eta, \tau$, 
		batch size $m$ 
\State Initialize particles: $v_i^0 \gets x_i$ for $i=1,\dots,n$
\vspace{0.05in}

\For{$k = 0,1,2,\dots$}
    \State Sample a batch $B$ with $|B| = m$

    \State Compute $v_i^{k+1}$ and $\theta^{k+1}$ 
    via GDA as in \eqref{eq:vi-GDA-3}\eqref{eq:theta-step-GDA}
     
    \State Update $\varphi$ by SGD/Adam  on minimizing
    $\calL(\varphi; \{ x_i, v_i^{k+1}\})$ as in
    \eqref{eq:loss-matching}
    	 to obtain $\varphi^{k+1}$  from $\varphi^k$      
\Statex \Comment{possibly using batches of size $m' < m$}
\EndFor{}
\end{algorithmic}
\end{algorithm}

\section{Experiments}

We will apply the proposed GDA approach to simulated 2D data and image data.
Codes available at \url{https://github.com/VinciZhu/wasserstein_minimax}.
To measure empirical convergence of the algorithm, we compute the gradient norm of $\theta$ and $T$
at the $k$-th iteration on a sample batch $B$ as
\begin{equation}\label{eq:GN-on-batch}
{\rm GN}_\theta^k =  \| \frac{1}{m}\sum_{i \in B} \partial_\theta \ell ( \theta^k, v^k_i) \|, \quad
{\rm GN}_T^k  =   \big( \frac{1}{m}\sum_{i \in B} \|  \partial_v  \ell( \theta^k, v_i^k) - \frac{1}{\gamma} (v_i^k - x_i) \|^2 \big)^{1/2},
\end{equation}
which are the empirical versions of $\| \partial_\theta L (\theta^k, T^k)\|$ and $\|\partial_T L (\theta^k, T^k) \|_{L^2(P)}$ respectively.

\begin{figure}[t]
\centering
\begin{minipage}{0.275\textwidth}
\includegraphics[height=.95\linewidth]{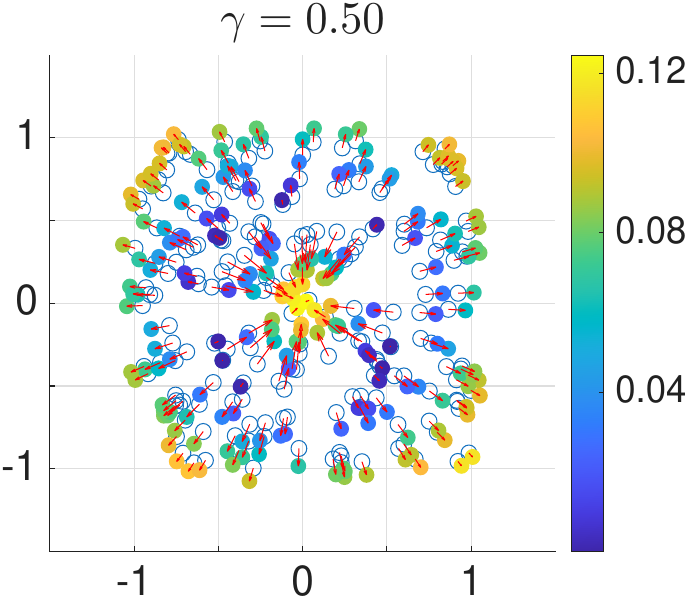}\\[10pt]
\includegraphics[height=.95\linewidth]{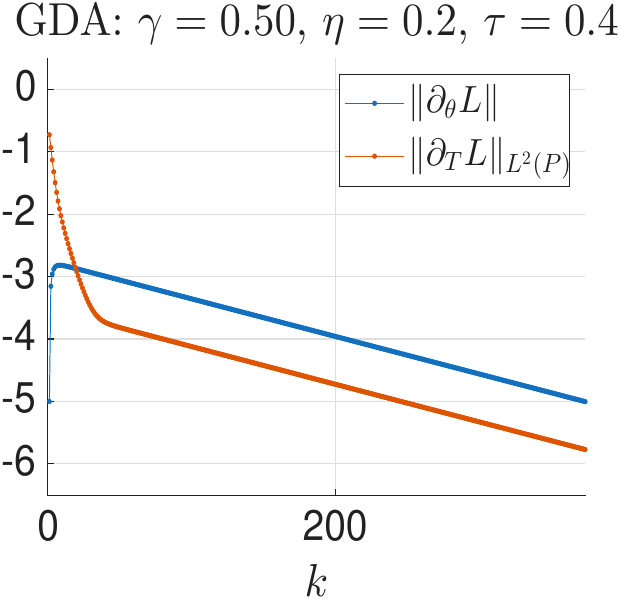}
\end{minipage}\hspace{10pt}
\begin{minipage}{0.275\textwidth}
\includegraphics[height=.95\linewidth]{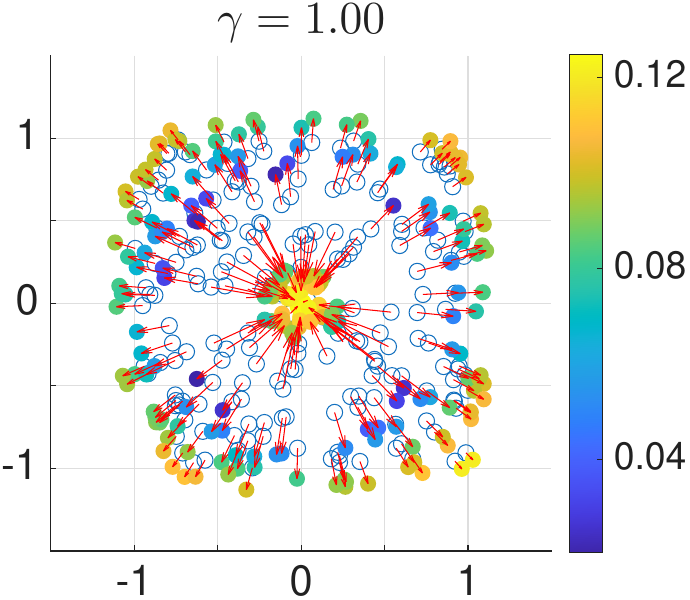}\\[10pt]
\includegraphics[height=.95\linewidth]{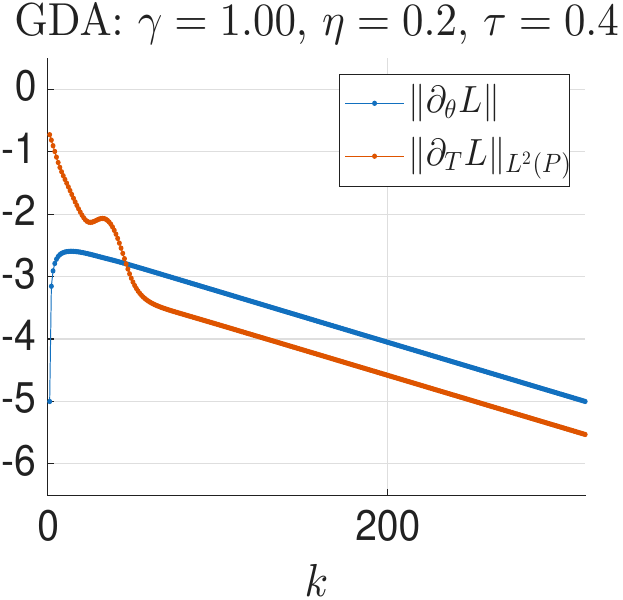}
\end{minipage}\hspace{10pt}
\begin{minipage}{0.275\textwidth}
\includegraphics[height=.95\linewidth]{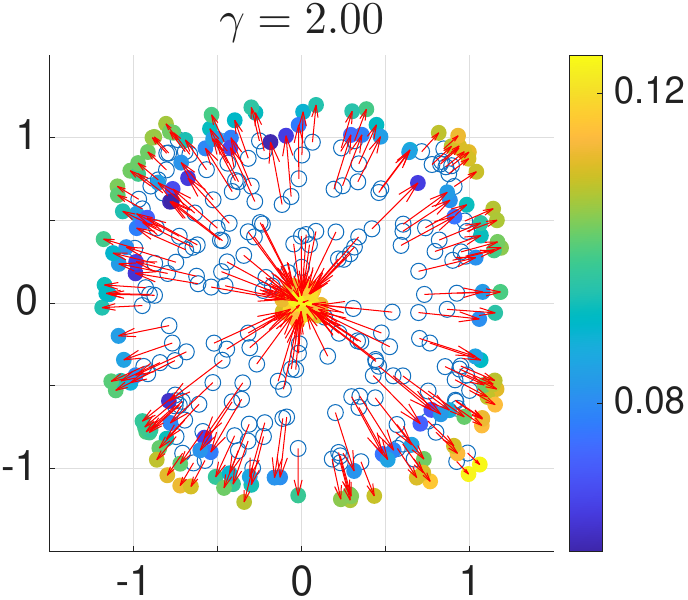}\\[10pt]
\includegraphics[height=.95\linewidth]{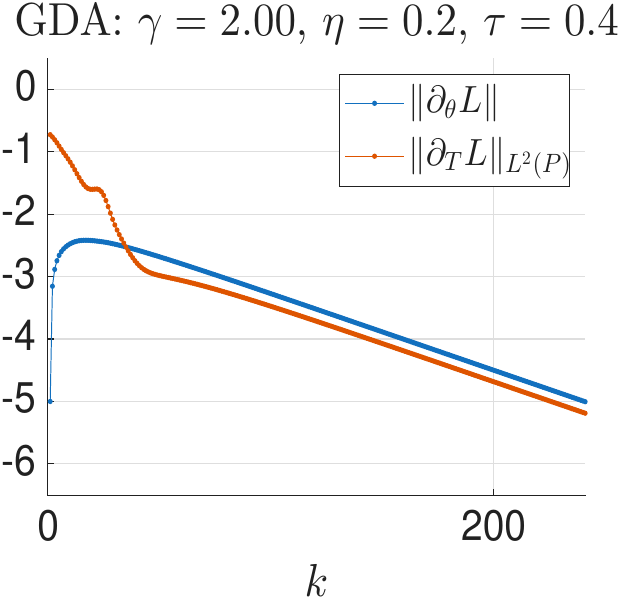}
\end{minipage}
\caption{
$L^2$ regression loss  on 2D data, GDA in the $\theta$-fast setting.
Upper: The saddle point solution found by the algorithm, 
where open circles are original samples $x_i$, solid circles are converged sample locations $v_i = T(x_i)$, and the color bar indicates the loss $\ell(\theta_k, \cdot)$ evaluated on $v_i$ in the last iteration.
Lower: Gradient norms of $\theta$ and $T$ computed on finite samples, see \eqref{eq:GN-on-batch}, along the iterations by GDA.
}
\label{fig:2d-regression-gda-theta-fast}
\end{figure}

\begin{figure}[t]
\centering
\begin{minipage}{0.275\textwidth}
\includegraphics[height=1\linewidth]{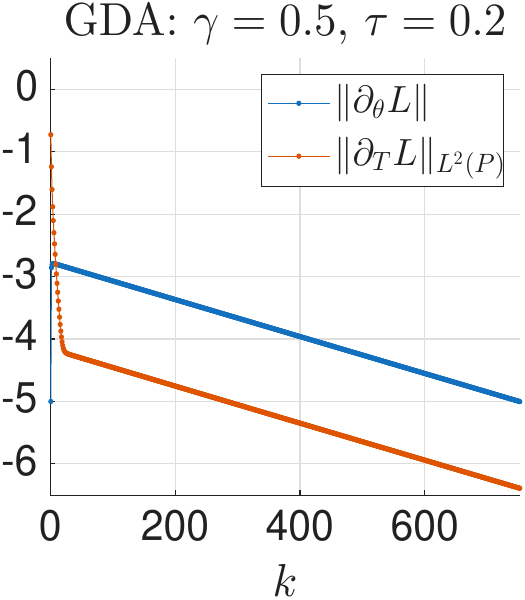} \\[10pt]
\includegraphics[height=1\linewidth]{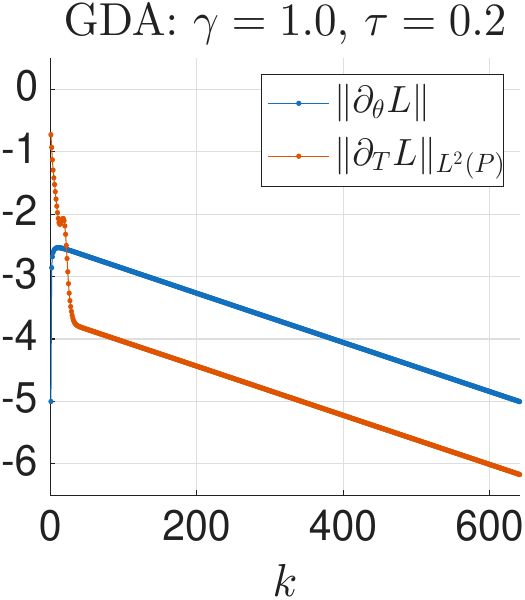}
\subcaption{}
\end{minipage}
\begin{minipage}{0.275\textwidth}
\includegraphics[height=1\linewidth]{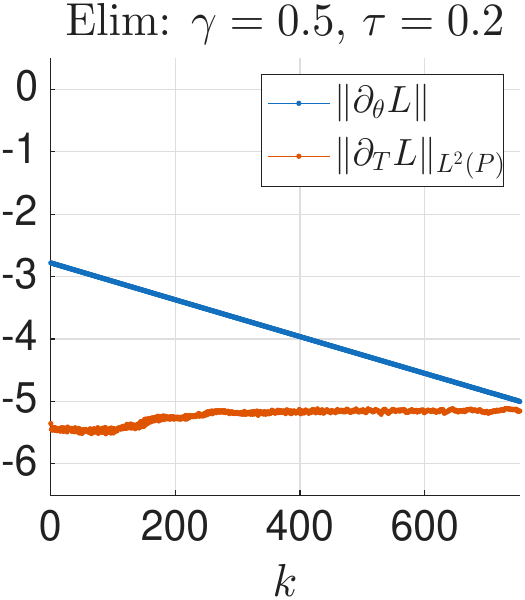} \\[10pt]
\includegraphics[height=1\linewidth]{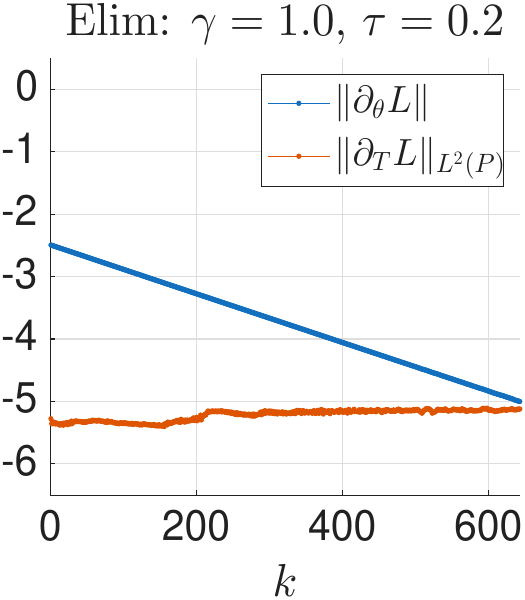}
\subcaption{}
\end{minipage}
\begin{minipage}{0.275\textwidth}
\includegraphics[height=1\linewidth]{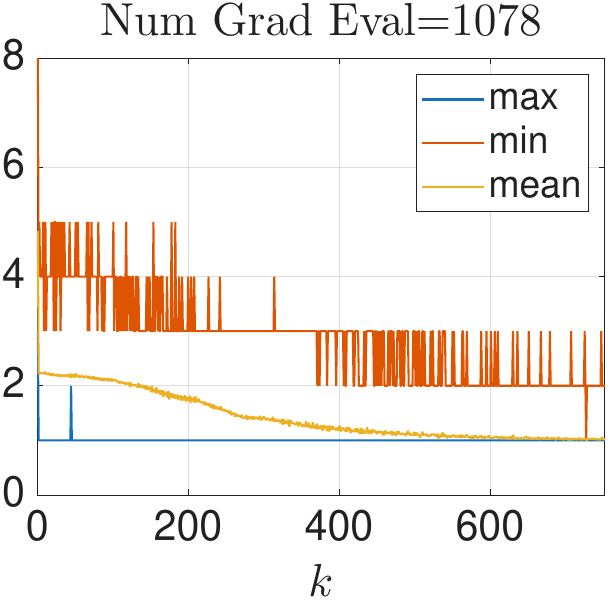} \\[10pt]
\includegraphics[height=1\linewidth]{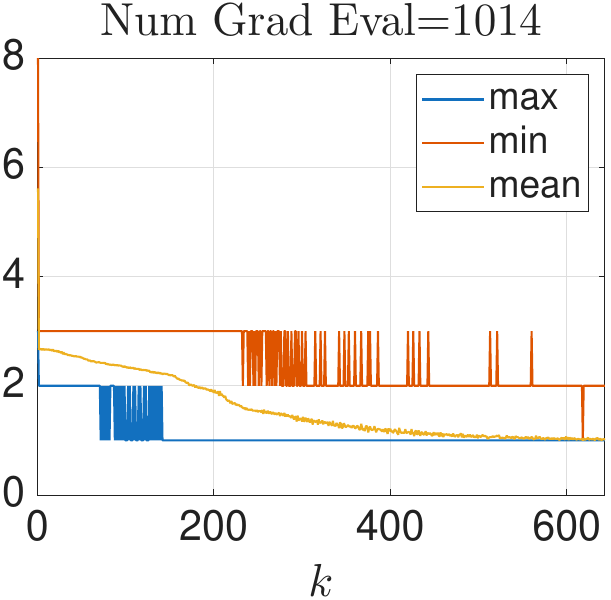}
\subcaption{}
\end{minipage}
\caption{
$L^2$ regression loss on 2D data, comparing GDA and Elim in the $T$-fast setting.
Upper: $\gamma = 0.5$. Lower: $\gamma=1$.
In this example, GDA and Elim start from the same initial value of $\{ \theta^0, v^0_i \}$ and converge to the same local (near) stationary point.
(a) Gradient norms by GDA, same plots in Figure \ref{fig:2d-regression-gda-theta-fast}.
(b) Gradient norms by Elim. Since $T$ is solved via an inner-loop, the gradient norm of $T$ is always below the tolerance 1e-5.
(c) Number of gradient evaluations (NGE) of the inner-loop in the Elim method, showing max, min and average values over the $n$ samples, and the total NGE is shown in the plot title. 
}
\label{fig:2d-regression-gda-elim}
\end{figure}

\subsection{Regression on two-dimensional data}

We consider a two-dimensional example following the setup in Example \ref{ex:l2-regression} (Figure \ref{fig:2d-regression-data}),
and more details can be found in Appendix \ref{app:exp-more}.
We use $n= 200$ samples, and since we focus on the particle optimization convergence on this example, we use all the samples in each iteration, that is, $B= [n]$ and $m=n$.
We implement the following methods:
\begin{enumerate}
\item[(i)] Elim: The elimination approach by solving for $\arg\max_{v} \ell(\theta_k, v) - \frac{1}{2\gamma} \| v -x_i \|^2 $ at each $x_i$ using an inner-loop in each iteration. For the inner-loop, we adopt the standard BFGS algorithm
with optimality tolerance \texttt{1e-5}.

\item[(ii)] GDA: as introduced in Section \ref{subsec:particle-opt-algo}, and we consider both the $T$-fast and $\theta$-fast settings by choosing different step sizes $\eta$, $\tau$.
\end{enumerate}

\paragraph{$\theta$-fast setting}
On this example, we numerically compute the $\partial_{\theta \theta}^2 H$ as explained in Example \ref{ex:l2-regression},
and numerically the Hessians are positive definite in all the cases that we observed.
Thus, we think this problem has almost strongly concavity in $\theta$.
We set  $\eta = 0.2$, $\tau = 0.4$,
and compute for $\gamma$ equals $\{0.5, 1, 2 \}$.
The GDA scheme achieves convergence using 
373, 315, 240 iterations respectively.
The results are shown in Figure \ref{fig:2d-regression-gda-theta-fast}.
When $\gamma$ increases, 
the found solution shows larger movement from $x_i$ to $v_i$.
The decrease of gradient norm in the plots demonstrates exponential convergence,
suggesting a near-strongly convex strongly concave property of this problem (at least locally)
possibly due to the simplicity of this example. 
Note that in this case, though we impose $\theta$ as the fast variable,
the gradient norm in $T$ decreases as fast as that in $\theta$ over the GDA iterations, 
and faster when $\gamma$ is small. 
Thus, it could be reasonable to use larger step size on $T$ as well.
We will examine the $T$-fast setting below.

\paragraph{Comparing GDA with Elim} We consider under the $T$-fast setting.
For GDA, we set $\eta = 0.4$, $\tau = 0.2$.
For Elim, $\tau = 0.2$, and $T^*[\theta]$ is solved by the inner-loop in each iteration. 
We set the stopping criterion as both gradient norms in \eqref{eq:GN-on-batch} less than \texttt{Tol}=1e-5.
We also set \text{Tol} as the optimality tolerance of the inner-loop solver of Elim.
We compute with two values of $\gamma = 0.5, 1$ respectively, and the results are shown in Figure \ref{fig:2d-regression-gda-elim}.

When $\gamma = 0.5$, GDA takes  751 iterations to achieve
$({\rm GN}_\theta, {\rm GN}_T)$ = (9.93e-06, 4.06e-07),
and Elim takes 752 iterations to achieve
$({\rm GN}_\theta, {\rm GN}_T)$ =(9.97e-06, 7.00e-06).
Note that by the inner-loop, ${\rm GN}_T$ is always below \texttt{Tol}, as shown in Figure \ref{fig:2d-regression-gda-elim}(b).
The two methods use the same initial value $\{\theta^0, v^0_i \}$, and converge to the same stationary point in this example.
We also record the number of gradient evaluations (NGE) of computing $\partial_T L$  of both methods: 
in GDA, NGE equals number of iterations,
and in Elim, NGE can be larger due to the inner loop. 
$\partial_T L$ is computed on the samples via the gradient in $v_i$, and we count NGE +1 when gradients on all samples are computed once more.
(Number of evaluations of $\partial_\theta L$ is always smaller: it is the same as that of $\partial_T L$ for GDA, and for Elim, it equals the number of outer iterations.)
In addition, we note that the number of function evaluations in Elim can be larger than NGE, because the BFGS with line search can use more function evaluations than gradient evaluations
(the latter equals the number of inner-loop iterations plus 1).
On this example, in Elim, when $k=1$, it takes up to 8 gradient evaluations for BFGS to converge, and in later phase of the Elim iterations the average NGE gradually drops to 1, which makes the method like a single-loop one (this is possible because we warm-start the BFGS by using $v^{k-1}_i$ as the initial value to solve for $v^k_i$ in $k$-th iteration), see Figure \ref{fig:2d-regression-gda-elim}(c).
Even though, over all the iterations, Elim still uses 1078 NGE, which is significantly more than the NGE of GDA, which is 751.

When increasing $\gamma$ to 1, the behavior of GDA and Elim are similar.
 While both methods use less iterations to converge -- GDA uses 640 iterations, and Elim uses 643 -- 
 the total NGE of Elim is still more than that of GDA.

\subsection{Classification on image data}

We consider the classification task 
where each data sample $x_i$ also has a class label $y_i$ belonging to $K$ classes, namely, $y \in [K]$. 
The decision model $\theta$ is a classifier neural network, and we use the cross-entropy loss.
Strictly speaking, the $K$-classes render the setting here different from before, in that the data distribution $P$ of $x$, the loss function $\ell$, and the transport map $T$ will all depend on the class label $y$,
and we denote the $y$-specific $T$ as $T_y$. 
In our experiments, we apply the minimax optimization in a latent space provided by a pretrained variational autoencoder (VAE) \cite{kingma2013auto}, where the latent representation captures high-level semantic structures in the images and supports smooth interpolation along the data manifold. The Euclidean distance in the latent space, which is the ground distance of the $\W$ space, is more meaningful than that in the image input space. 
Details of the classification network $\theta$, the pre-trained VAE, 
the $y$-specific loss and neural transport map $\varphi$ can be found in Appendix \ref{app:exp-more}.

\begin{figure}[t]
\centering
\begin{subfigure}{0.32\linewidth}
\hspace{-8pt}
\includegraphics[width=\linewidth]{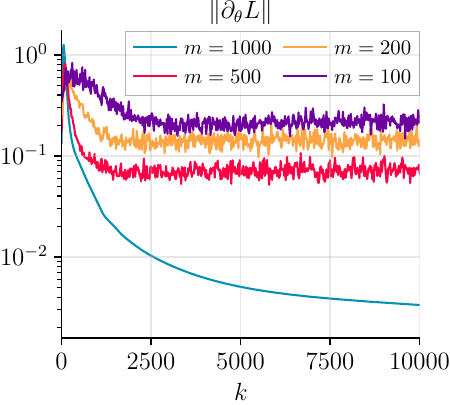}
\subcaption{}
\end{subfigure}
\begin{subfigure}{0.32\linewidth}
\hspace{-8pt}
\includegraphics[width=\linewidth]{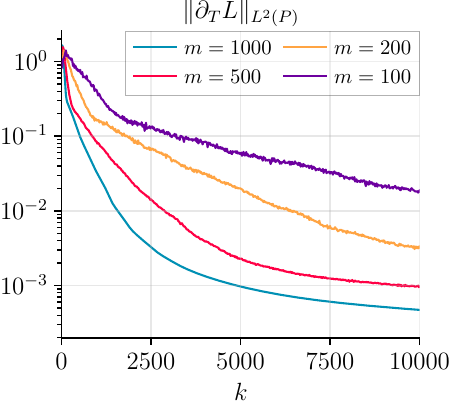}
\subcaption{}
\end{subfigure}
\begin{subfigure}{0.32\linewidth}
\hspace{-8pt}
\includegraphics[width=\linewidth]{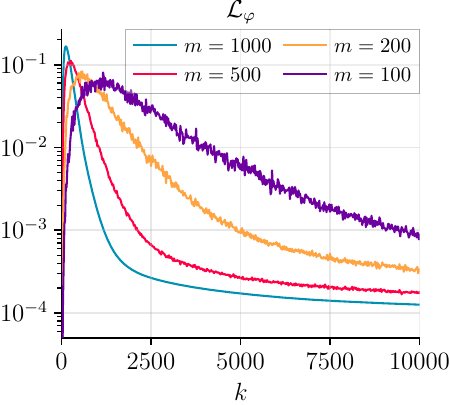}
\subcaption{}
\end{subfigure}
\caption{
GDA with momentum on MNIST
with different batch sizes $m$,
where $k$ stands for the number of batches.
$\gamma=8.0$, $\eta=\tau=0.01$, and momentum 0.7. 
(a) Gradient norm ${\rm GN}_\theta^k$,
and 
(b) gradient norm ${\rm GN}_T^k $, computed on batches as defined in \eqref{eq:GN-on-batch}.
(c) Matching loss $\calL$ for $k$-th batch (average over the $m/m'$ smaller batches, $m'=50$).
All errors are shown in the log scale
and subsampled at intervals of 20 iterations for better visualization.}
\label{fig:mnist-curves-momentum}
\end{figure}

\paragraph{MNIST} 
The dataset consists of 10 classes of grayscale handwritten digits of size $28 \times 28$. 
We use the original training set (having 60,000 samples) to pre-train the VAE with a 32-dimensional latent space.
We then use a subset of $n=1{,}000$ training samples in our minimax optimization, and we implement the GDA scheme with and without momentum, and with varying step sizes $\eta$ and $\tau$.
We use varying batch sizes $m= 1{,}000, 500, 200, 100$ for the particle updates, and the batch size for training $T_\varphi$ is fixed at $m'=50$. See Appendix \ref{app:exp-more} for more details of the setup.

Figure \ref{fig:mnist-curves-momentum} depicts the evolution of the two gradient norms and the matching loss $\calL$ in log scale.
For plain GDA without momentum, the same error plots are shown in Figure \ref{fig:mnist-curves-without-momentum}. 
It can be seen that the gradient norm and matching loss curves all exhibit a steady decreasing trend after a short warm-up period, showing that the mini-batch updates of stochastic GDA are effective. 
While our theory addresses the vanilla GDA, here we observe that the scheme converges after adding momentum, which also accelerates convergence (by comparing  Figure \ref{fig:mnist-curves-momentum} with Figure \ref{fig:mnist-curves-without-momentum}).
With varying step sizes $\eta$ and $\tau$, the scheme remains convergent as shown in Figure \ref{fig:mnist-curves-step-size}, where the convergence of particles (measured by gradient norm in $T$) is faster when using a larger step size $\eta$.
In addition, the matching loss $\calL$ decreases at a pace comparable to that of $\|\partial_T L\|_{L^2(P)}$ and does not exhibit noticeable delay, indicating that the concurrent training of $T_\varphi$ with the particle iterations is efficient.

After the neural transport map $T_\varphi$ is trained, we apply it to the test set of MNIST to compute from a test sample $x_i'$ the transported point $T(x_i')$ (in the latent space). 
To illustrate the transport map, we linearly interpolate between $x_i'$ and $T(x_i')$ and show the images after decoding to the input space. 
In Figure \ref{fig:interpolation-panel}(a), we show the results on the top sample from 5 classes (in the test set) that achieves the largest magnitude of change of pixel values in each class.
The model is trained with $m = 500$ and $k=20{,}000$ iterations (batches),
and the results are almost identical to when $m=n=1{,}000$.
(The results with a smaller batch size like $m=100$ are also similar, but have some visual discrepancy from the $m=1{,}000$ result.)
The panel shows semantically meaningful deformations of the digits towards a different class (labeled by $\theta_0$). 
To further examine the out-of-sample extension performance of our model,
we also plot the interpolation trajectories between $x_i$ and its corresponding $v_i$ (solved by the particle optimization) where $x_i$ is the nearest neighbor of $x_i'$ in the training set (Figure \ref{fig:train_vs_test}). The visual similarity between the trajectory of $x_i$ and $x_i'$ indicates that the neural transport $T_\varphi$ effectively generalizes to the test samples.

\begin{figure}[t]
\centering
\begin{subfigure}{0.48\linewidth}
\includegraphics[width=\linewidth]{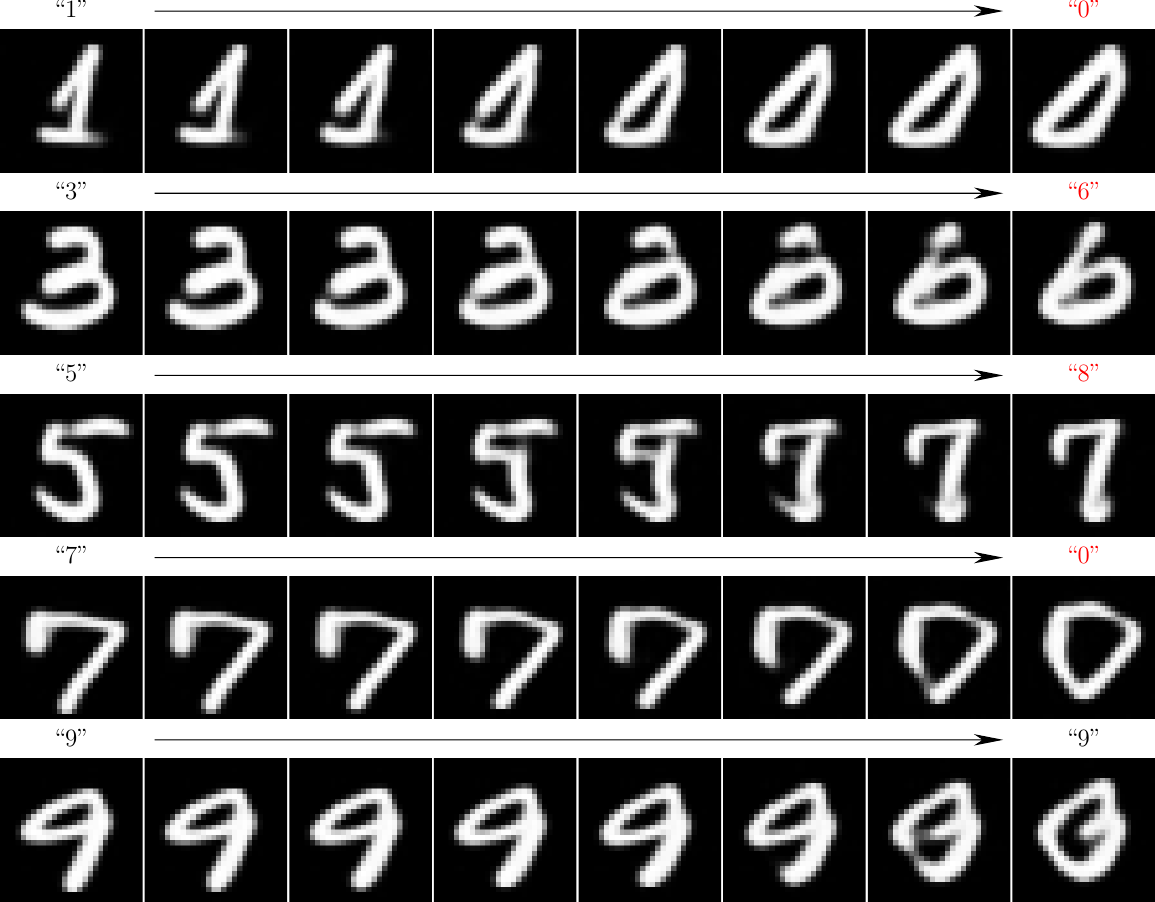}
\caption{MNIST}
\end{subfigure}
\hfill
\begin{subfigure}{0.48\linewidth}
\includegraphics[width=\linewidth]{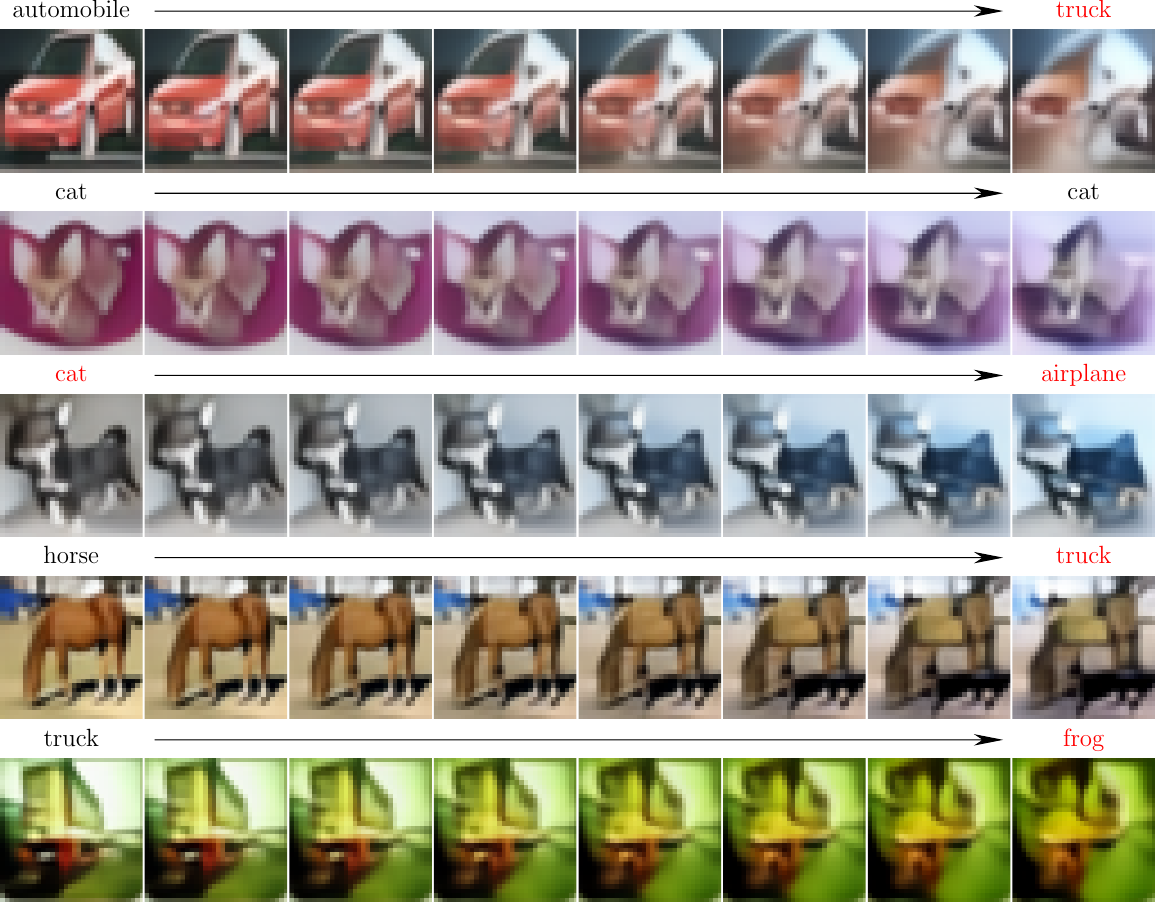}
\caption{CIFAR-10}
\end{subfigure}
\caption{
\label{fig:interpolation-panel}
Interpolation trajectories between $x_i$ and $T(x_i)$
by the learned neural transport $T_\varphi$ on test samples $x_i$. 
Class labels predicted by $\theta^0$ are shown for $x_i$ and $T(x_i)$, and incorrect predictions are colored in red. 
For both MNIST and CIFAR-10,
we apply the GDA with $m=500$ and $k=20{,}000$ batches.}
\label{fig:interpolation-panels}
\end{figure}
 
\paragraph{CIFAR-10} The dataset consists of colored images of size $32 \times 32 $ across 10 object categories. 
We use the original training set of 50,000 images to pretrain the VAE with a 256-dimensional latent space,
and take 200 training images from each class to form a subset of 2,000 samples to use in our minimax optimization.
The overall experimental setup is similar to that of MNIST.
Specifically, we adopt GDA with momentum 0.9, 
step sizes $\eta=\tau=0.001$,
batch size $m=500$ (and $m'=100$ in training $T_\varphi$),
and train for $k=20,000$ batches.
The other details can be found in Appendix \ref{app:exp-more}. 
The convergence of the gradient norms in $\theta$ and $T$ and the matching loss are similar to those on MNIST, and the interpolation panel on the test set is shown in Figure \ref{fig:interpolation-panels}(b).
The learned neural transport $T_\varphi$ successfully generates sample images exhibiting meaningful semantic deformations of the objects, demonstrating the effectiveness of our method on this higher-dimensional example.

\section{Discussion}\label{sec:discuss}

 We begin with a few remarks on the PL condition in Assumption \ref{assump:PL-ptwise-in-T}.
Although somewhat strong, this pointwise PL requirement can be interpreted as imposing a regularizing effect on the transport map $T$, arising fundamentally from the quadratic anchoring term $ -\frac{1}{2 \gamma}\|T(x)-x\|^2$ in the objective.
Geometrically, this term anchors  $T(x)$ to its input $x$ and prevents the map from collapsing, for example, to a single worst-case point.
At the same time, we think it is possible to relax the pointwise PL condition by a one in $T$ that is ``in expectation over $P$'',
 and specifically, the $T$-Error Bound \eqref{eq:EB-T} shown in Lemma \ref{lemma:T-fast-Tstar-phi}
 (together with other needed regularity condition of $T^*(\theta)$).
 It is also possible to relax the uniqueness assumption on $v^*(\theta;x)$, e.g. to allow the maximizers being in a compact set,  following standard analysis in the vector space.
 
This work can also be extended in several other directions. 
 First, we focused on GDA updates, but the framework based on the transport map $T \in L^2(P)$ could accommodate other first-order minimax algorithms to apply to the Wasserstein minimax problem here.
Examples include Optimistic Gradient Descent-Ascent (OGDA) and Extragradient (EG) \cite{nemirovski2004prox,mokhtari2020convergence} methods, which are potentially more stable than GDA and have been studied in the GAN training \cite{daskalakis2018training, chavdarova2019reducing}.
Strengthening the theoretical analysis to establish convergence rates for these methods in our framework is an interesting avenue for future work. At the same time, extending the approach to more high-dimensional data,
including images and beyond, may reveal additional practical insights.

\subsubsection*{Acknowledgments}
The authors thank 
Johannes Milz, 
Jiajin Li
for helpful discussions. 
The work of XC and YX was
partially supported by NSF DMS-2134037.
XC was also partially supported by 
NSF DMS-2237842 
and the Simons Foundation (MPS-MODL-00814643).
The work of 
LZ, YZ and YX 
was partially funded 
by NSF DMS-2134037, CMMI-2112533, and the Coca-Cola Foundation.

\bibliographystyle{plain} 
\bibliography{opt}

\appendix
\setcounter{table}{0}
\setcounter{figure}{0}
\renewcommand{\thetable}{A.\arabic{table}}
\renewcommand{\thefigure}{A.\arabic{figure}}
\renewcommand{\thelemma}{A.\arabic{lemma}}

\setcounter{assumption}{0} \renewcommand{\theassumption}{A.\arabic{assumption}}

\section{Proofs}\label{app:proofs}

\subsection{Auxiliary lemma}

\begin{lemma}\label{lemma:Lip-is-L2}
Suppose $P \in \calP_2(\R^d)$,
a measurable set $A \subset \R^d$ satisfies that $P(A) =1$,
 and $v: A \to \R^d$ is Lipschitz on $A$,
 that is, for some $a \ge 0$,
 $ \| v(\tilde x) - v(x) \| \le a \| \tilde x - x\|$, 
 $\forall \tilde x, x \in A$, 
  then $v \in L^2(P)$.
\end{lemma}
\begin{proof}[Proof of Lemma \ref{lemma:Lip-is-L2}]
$A$ cannot be empty because $P(A)=1$,
then  there is at least one point $x_0 \in A$  and $\| v(x_0)\| < \infty$.
For any $x \in A$, we have that 
\[
\| v(x) \| \le \| v(x_0)  \| + a \| x - x_0\|, \quad \forall x \in A.
\]
This gives that 
\begin{align*}
\E_{x\sim P} \|v(x)\|^2 
&= \int_{A} \|v(x)\|^2 dP(x) \\
&\le \int_A (\| v(x_0)  \| + a \| x\|  + a\| x_0\| )^2 dP(x) \\
& \le  2 (\| v(x_0)  \|+a \| x_0\| )^2 + 2 a^2 \E_{x\sim P}\| x \|^2 < \infty,
\end{align*}
where the first equality is by that $P(A)=1$,
and to show that the last row is finite we used that $\E_{x\sim P} \|x\|^2 < \infty$ due to that $P \in \calP_2(\R^d)$.
\end{proof}

\subsection{Proofs in Section \ref{sec:T-fast}}

\begin{proof}[Proof of Lemma \ref{lemma:theta-fast-lsmooth-H}]
It is possible to show that $L$ is $l$-smooth in the sense of \eqref{eq:def-L-smoothness},
which then implies coordinate $l$-smoothness.
Here, we prove using the definition of coordinate $l$-smoothness as some intermediate estimates will be used in later analysis. 
We claim that 
\begin{align}
\| \partial_\theta L( \tilde \theta,  T) - \partial_\theta L(  \theta, T) \|
&	\le l_0 \| \tilde \theta - \theta \|, & \forall T, \tilde \theta, \theta, \label{eq:coord-smooth-L-11} \\
\| \partial_\theta L( \theta, \tilde T) - \partial_\theta L(  \theta, T) \|
&	\le l_0 \| \tilde T - T\|_{L^2(P)},   & \forall \tilde T,  T, \theta,  \label{eq:coord-smooth-L-12} \\
\| \partial_T L( \tilde \theta,  T) - \partial_T L(  \theta, T) \|_{L^2(P)}
&	\le l_0 \| \tilde \theta - \theta \|,  & \forall T, \tilde \theta, \theta, \label{eq:coord-smooth-L-21} \\
\| \partial_T L(  \theta,  \tilde T) - \partial_T L(  \theta, T) \|_{L^2(P)}
&	\le (l_0 + \frac{1}{\gamma} )\| \tilde T - T\|_{L^2(P)},  & \forall \tilde T,  T, \theta. \label{eq:coord-smooth-L-22}
\end{align}
If true,  this proves the coordinate $l$-smoothness of $L$ by Definition \ref{def:coord-L-smoothness-M}
with $l = l_0 + 1/\gamma$.

Recall that $\forall \theta \in \R^p, T \in L^2(P)$,
\begin{equation*}
\partial_T L(\theta, T)(x)
=   \partial_v \ell( \theta, T(x) ) - \frac{1}{\gamma} (T(x) - x), \quad P-a.s.
\end{equation*}
\begin{equation*}
 \partial_\theta L(  \theta, T) 
= \E_{x \sim P} \partial_\theta \ell( \theta, T(x)).
 \end{equation*}
We then have
\[
\| \partial_\theta L( \tilde \theta,  T) - \partial_\theta L(  \theta, T) \|
\le \E_{x \sim P} 
	\| \partial_\theta \ell( \tilde  \theta, T(x)) - \partial_\theta \ell( \theta, T(x)) \|
\le l_0 \| \tilde \theta - \theta\|,
\]
where the second inequality is by that $\ell( \theta ,v) $ is $l_0$-smooth on $\R^p \times \R^d$.
This proves  \eqref{eq:coord-smooth-L-11}.
Similarly,
\begin{align*}
\| \partial_\theta L( \theta,   \tilde T) - \partial_\theta L(  \theta, T)  \|
& \le   \E_{x \sim P} 
	\|  \partial_\theta \ell(  \theta, \tilde  T(x)) - \partial_\theta \ell( \theta, T(x)) \| \\
& \le l_0  \E_{x \sim P} \| \tilde T(x) - T(x)\|
\le l_0 \| \tilde T - T\|_{L^2(P)},
\end{align*}
where the second inequality is again by the $l_0$-smoothness of $\ell$. 
This proves \eqref{eq:coord-smooth-L-12}.

To see \eqref{eq:coord-smooth-L-21}, observe that for $P$-a.s. $x$,
\[
\| \partial_T L(\tilde \theta, T)(x) - \partial_T L(\theta, T)(x)  \|
= \| \partial_v \ell( \tilde \theta, T(x) ) -   \partial_v \ell( \theta, T(x) ) \|
\le l_0 \| \tilde \theta - \theta\|
\]
by the $l_0$-smoothness of $\ell$,
and then \eqref{eq:coord-smooth-L-21} follows.
Finally, for $P$-a.s. $x$,
\[
\partial_T L(\theta, \tilde T)(x) - \partial_T L(\theta, T)(x)
= (\partial_v \ell( \theta, \tilde T(x) ) 
	- \partial_v \ell( \theta,  T(x) ))
	 - \frac{1}{\gamma} (\tilde T(x) - T(x)),
\]
and then
\begin{align*}
& \| \partial_T L(\theta, \tilde T) - \partial_T L(\theta, T)\|_{L^2(P)}
 \le 
(\E_{x \sim P} \| \partial_v \ell( \theta, \tilde T(x) ) 
	- \partial_v \ell( \theta,  T(x) ) \|^2 )^{1/2} + \frac{1}{\gamma} \| \tilde T - T\|_{L^2(P)} \\
& \le (  \E_{x \sim P} 	l_0^2 \| \tilde T(x) - T(x)\|^2 )^{1/2} + \frac{1}{\gamma} \| \tilde T - T\|_{L^2(P)}
= ( l_0 + \frac{1}{\gamma} ) \| \tilde T - T\|_{L^2(P)},
\end{align*}
where the second equality is again due to the $l_0$-smoothness of $\ell$. 
This proves \eqref{eq:coord-smooth-L-22}.
\end{proof}

\begin{proof}[Proof of Lemma \ref{lemma:T-fast-Tstar-phi}]
i) For each fixed $\theta$, we first show that $v^*(\theta; x)$ is Lipschitz in $x$ and in $L^2(P)$.
Note that $v^*(\theta; x)$ is defined for $P$-a.s. $x$,
thus we are to show that on the set where it is defined, called $A$, the mapping $v^*(\theta; \cdot): A \to \R^d$ is Lipschitz.  
Specifically, Suppose (for this $\theta$) Assumption \ref{assump:PL-ptwise-in-T}  holds for $x \in A \subset \R^d$, where $P(A)=1$, then $v^*(\theta; x) $ is well-defined for $x \in A$. 
For any $\tilde x, x \in A$, 
let $v^* = v^*(\theta; x)$, $\tilde v^* = v^*(\theta;  \tilde x)$. 
We have \eqref{eq:EB-h-theta-v-x} hold with $x$ and $v = \tilde v^*$, and that is $
\| \nabla_v h(\theta, \tilde v^*; x) \| \ge \mu \| \tilde v^* - v^*\|$.
Meanwhile, $\nabla_v h(\theta, \tilde v^*; \tilde x) = 0$ by first-order optimality, and then
\[
 \nabla_v h(\theta, \tilde v^*; x)
 =  \nabla_v h(\theta, \tilde v^*; x) - \nabla_v h(\theta, \tilde v^*; \tilde x) = \frac{1}{\gamma}(x-\tilde x).
\]
Putting together, this shows that
$
 \mu \| \tilde v^* - v^*\| \le \frac{1}{\gamma} \| x - \tilde x\|$,
 which means that $v^*(\theta; x)  $ is $(\gamma \mu)^{-1}$-Lipschitz in $x$ on $A$.
Then,  by Lemma \ref{lemma:Lip-is-L2},
 this also implies that $v^*(\theta; \cdot)$ is in $L^2(P)$.

We now have that $T^*(\theta)(x) = v^*(\theta; x)$ is a member of $L^2(P)$,
and next, we  show that it is the unique maximizer of $\max_T L(\theta, T)$.
By that $v^*(\theta; x) = \argmax_v h(\theta, v; x)$,  we know that for any $T \in L^2 (P)$,
\[
L(\theta, T) 
= \E_{x \sim P} h(\theta, T(x); x)
\le \E_{x \sim P} h(\theta, v^*(\theta; x); x) = L(\theta, T^*(\theta)).
\]
Thus $T^*(\theta)$ is a maximizer. To show the uniqueness, suppose another $\tilde T$ also has $L(\theta, \tilde T) = L(\theta, T^*(\theta))$.
The $\mu$-PL condition implies Quadratic Growth (QG) \cite[Theorem 2]{karimi2016linear}, namely
\[
h( \theta, v^*(\theta;x); x) - h(\theta, v; x) \ge  \frac{\mu}{2} \|v^*(\theta;x) -v \|^2, \quad \forall v \in \R^d.
\]
Letting $v = \tilde T(x)$ and taking expectation over $x \sim P$, we have
\begin{align*}
\frac{\mu}{2}  \E_{x\sim P} \| T^*(\theta)(x) - \tilde T(x) \|^2 
& \le  \E_{x\sim P} (h( \theta, T^*(\theta)(x); x) - h(\theta, \tilde T(x); x) )  \\
& = L(\theta, T^*(\theta)) - L (\theta, \tilde T)
=0.
\end{align*}
This shows that $\| T^*(\theta) - \tilde T \|_{L^2(P)} = 0$, namely $\tilde T = T^*(\theta)$ in $L^2(P)$.

To prove \eqref{eq:EB-T}, for fixed $\theta$ and any $T \in L^2(P)$,
\begin{align*}
\| T - T^*(\theta) \|_{L^2(P)}^2
& = \E_{x \sim P} \| T(x) - v^*(\theta; x) \|^2 \\
& \le   \frac{1}{\mu^2}  \E_{x \sim P} \| \nabla_v h(\theta, T(x); x)\|^2 \\
& =  \frac{1}{\mu^2}  \| \partial_T L(\theta, T)\|_{L^2(P)}^2,
\end{align*}
where the first inequality is by  \eqref{eq:EB-h-theta-v-x}, and in the last equality we used that  
\[
\nabla_v h( \theta, T(x); x) = \partial_v \ell(\theta, T(x)) - \frac{1}{\gamma}(T(x)-x)
= \partial_T L(\theta, T)(x), \quad P-a.s. 
\]
which holds by definition.

ii) For any $\tilde \theta, \theta \in \R^p$, we apply \eqref{eq:EB-T} to $T= T^*(\tilde \theta)$ to obtain 
\[
\mu \| T^*(\tilde \theta) - T^*(\theta) \|_{L^2(P)}
\le 
\| \partial_T L(\theta, T^*(\tilde \theta) )\|_{L^2(P)}.
\]
Meanwhile, the first order optimality gives that $\partial_v h(\tilde \theta, v^*(\tilde \theta; x); x ) =0$, 
which means that $\partial_T L(\tilde \theta, T^*(\tilde \theta))(x) = 0$ for $P$-a.s. $x$.
Putting together, we have
\[
\mu  \| T^*(\tilde \theta) - T^*(\theta) \|_{L^2(P)}
\le 
\| \partial_T L(\theta, T^*(\tilde \theta) ) - \partial_T L(\tilde \theta, T^*(\tilde \theta) )  \|_{L^2(P)}
\le l \| \theta - \tilde \theta \|,
\]
where in the last inequality is by that $L$ is   coordinate  $l$-smooth on $\R^p \times L^2(P) $ 
shown in Lemma \ref{lemma:theta-fast-lsmooth-H}.

iii)  By that $T^*(\theta)= v^*(\theta; \cdot)$ as shown in i), we have
$
\phi(\theta) = L(\theta, T^*(\theta)) =  \E_{x \sim P} h(\theta, v^*(\theta; x); x).
$
By Danskin's theorem, 
\[
\partial \phi(\theta) 
= \E_{x \sim P} \partial_\theta h(\theta, v^*(\theta; x); x)
= \E_{x \sim P} \partial_\theta \ell (\theta, v^*(\theta; x)).
\]
This shows that 
$\partial \phi(\theta)  = \partial_\theta L(\theta, T^*(\theta))$.
Then, we have
\begin{align*}
\| \partial \phi(\tilde \theta)  - \partial \phi(\theta)  \|
& = \| \partial_\theta L(\tilde \theta, T^*(\tilde \theta)) - \partial_\theta L(\theta, T^*(\theta)) \| \\
& \le \| \partial_\theta L(\tilde \theta, T^*(\tilde \theta)) - \partial_\theta L(\tilde \theta, T^*(\theta)) \| 
	+ \| \partial_\theta L(\tilde \theta, T^*( \theta)) - \partial_\theta L(\theta, T^*(\theta)) \|   \\
& \le 	l (\| T^*(\tilde \theta) - T^*( \theta) \| +  \| \tilde \theta - \theta\| ),
\end{align*}
where the second inequality is by that $L$ is   coordinate  $l$-smooth on $\R^p \times L^2(P) $ (Lemma \ref{lemma:theta-fast-lsmooth-H}).
Combined with ii), we have that
$\| \partial \phi(\tilde \theta)  - \partial \phi(\theta)  \| \le l (1+\kappa) \| \tilde \theta - \theta\|$.
\end{proof}

\begin{proof}[Proof of Lemma \ref{lemma:desdent-coordinate-M}]
The proof of i) is standard since $\theta $ is in $\R^p$,
where one uses 
\[
\| \partial_\theta M(\theta' , T) - \partial_\theta M( \theta, T)  \| \le l \| \theta' - \theta\|, 
\quad \forall \theta', \theta \in \R^p,
\]
due to the   coordinate  $l$-smoothness of $M$.

ii) is proved similarly by considering variables in $L^2(P)$.
 Specifically, by considering the line $T(s):= T + s (\tilde T - T)$, $ s \in [0,1] $, in $L^2(P)$, we have
 \begin{align*}
 M(\theta, \tilde T ) - M( \theta, T) - \langle \partial_T M( \theta, T), \tilde T - T \rangle_{L^2(P)} 
& = \int_0^1
 \langle \partial_T M( \theta, T(s) )-   \partial_T M( \theta, T), \tilde T - T \rangle_{L^2(P)}
 ds, 
 \end{align*}
which gives that
 \begin{align*}
 & |  M(\theta, \tilde T) - M(\theta, T) - \langle \partial_T M( \theta, T), \tilde T - T \rangle_{L^2(P)} | \\
 & \le 
 \int_0^1
\| \partial_T M(\theta, T(s))-   \partial_T M( \theta, T)\| \| \tilde T - T \|_{L^2(P)} ds  \\
& \le \int_0^1
l \| T(s) - T\|_{L^2(P)}
 \| \tilde T - T \|_{L^2(P)} ds
 = \frac{l}{2} \| \tilde T - T \|_{L^2(P)}^2, 
 \end{align*}
 where the second inequality is by that 
 \[
\| \partial_T M(\theta, T') - \partial_T M( \theta, T)  \| \le l \| T' - T\|_{L^2(P)}, 
\quad \forall T', T \in L^2(P)
\]
due to the   coordinate  $l$-smoothness.
\end{proof}

 \begin{proof}[Proof of Theorem \ref{thm:rate-T-fast}]
 The proof follows the same strategy as in \cite{yang2022faster} in the deterministic case,
 and the paper considered vector-space min-max problem in the NC-PL setting.
 Here, we consider GDA without alternating and our variable $T$ is in the functional space $L^2(P)$,
 so some adaptation and additional analysis is needed.
 
Recall that $\kappa = l/\mu$, $L_1 = l (1+\kappa)$.
We assume that
\begin{equation}\label{eq:proof-T-fast-condition-eta-tau}
L_1 \tau \le 1,  \quad l \eta \le 1,
\end{equation}
which will be fulfilled later in the proof.

\vspace{2pt}
$\bullet$ Descent of primal: 
 In the $T$ step, by Lemma \ref{lemma:desdent-coordinate-M} ii),
 \begin{align}
L(\theta_{k+1}, T_{k+1} ) - L ( \theta_{k+1}, T_{k})
& \ge \langle \partial_T L(\theta_{k+1}, T_k),  T_{k+1}- T_k \rangle_{L^2(P)}
	 -  \frac{l}{2} \| T_{k+1} - T_k\|_{L^2(P)}^2   \nonumber \\
& = \eta  \langle  \partial_T L(\theta_{k+1}, T_{k}), \partial_T L( \theta_k, T_k)  \rangle_{L^2(P)}
	-  \frac{l \eta^2 }{2}  \|  \partial_T L( \theta_k, T_k) \|_{L^2(P)}^2, 
	\label{eq:proof-T-fast-T-step-1}
\end{align}
 where in the equality we have used that $T_{k+1} =T_k + \eta \partial_T L( \theta_k, T_k)$.
 Define 
\[
e_k: =  \partial_T L( \theta_{k+1}, T_k) - \partial_T L( \theta_k, T_k), 
\]
By that $L$ is   coordinate  $l$-smooth on $\R^p \times L^2(P) $ (Lemma \ref{lemma:theta-fast-lsmooth-H}), we have
\begin{equation}\label{eq:proof-T-fast-eqn5}
\| e_k \|_{L^2(P)} 
\le  l \| \theta_{k+1} - \theta_{k}\| 
= l \tau \| \partial_\theta L(  \theta_k, T_k) \|,
\end{equation}
where in the equality we inserted that $\theta_{k+1} =\theta_k  - \tau \partial_\theta L( \theta_k, T_k) $.

Back to \eqref{eq:proof-T-fast-T-step-1}, we have
\[
L(\theta_{k+1}, T_{k+1}) - L(\theta_{k+1}, T_k)
\ge  \eta (1-\frac{l \eta }{2}) \| \partial_T L( \theta_k, T_k) \|_{L^2(P)}^2 
	+ \eta \langle e_k,  \partial_T L( \theta_k, T_k) \rangle_{L^2(P)}.
\]
Note that 
\begin{align*}
-   \langle e_k,  \partial_T L( \theta_k, T_k) \rangle_{L^2(P)}
& \le  \| e_k \|_{L^2(P)} \| \partial_T L( \theta_k, T_k) \|_{L^2(P)}
 \le l \tau \| \partial_\theta L(  \theta_k, T_k) \|  \| \partial_T L( \theta_k, T_k) \|_{L^2(P)}  \\
& \le \frac{l \tau}{2} ( \| \partial_\theta L(  \theta_k, T_k) \| ^2 + \| \partial_T L( \theta_k, T_k) \|_{L^2(P)}^2),
\end{align*}
where the second inequality is by \eqref{eq:proof-T-fast-eqn5}.
We then have
\begin{align}
L(\theta_{k+1}, T_{k+1}) - L(\theta_{k+1}, T_k)
& \ge  \eta (1-\frac{l \eta }{2} - \frac{l \tau}{2}) \| \partial_T L(\theta_k, T_k) \|_{L^2(P)}^2 
	-  \eta \frac{l \tau}{2}  \| \partial_\theta L( \theta_k, T_k) \|^2 \nonumber \\
& \ge \frac{\eta}{4} 	\| \partial_T L(\theta_k, T_k) \|_{L^2(P)}^2   
 	- \frac{\tau}{2}   \| \partial_\theta L( \theta_k, T_k) \|^2,
	\label{eq:proof-T-fast-T-step-2}
\end{align}
where the second inequality is by that $l \eta \le 1$,
and that $ l \tau \le \frac{1}{ 1+\kappa} \le 1/2$  under \eqref{eq:proof-T-fast-condition-eta-tau}.

In the $\theta$-step, by Lemma \ref{lemma:desdent-coordinate-M} i),
\begin{align}
 L( \theta_{k+1},  T_k) - L( \theta_k,  T_k)
 &\ge   \partial_\theta L(\theta_k, T_k) \cdot (\theta_{k+1} - \theta_k )
 	- \frac{l}{2} \| \theta_{k+1} - \theta_k \|^2 \nonumber \\
& = 	-\tau (1+ \frac{l \tau}{2}) \| \partial_\theta L(\theta_k, T_k) \|^2 \nonumber \\
& \ge - \frac{5 \tau }{4}  \| \partial_\theta L(\theta_k, T_k) \|^2 \nonumber,
 \end{align}
 where the last inequality is by $ l \tau \le 1/2$ again.
 Together with \eqref{eq:proof-T-fast-T-step-2},  this gives 
\begin{equation}\label{eq:proof-T-fast-primal-descent-3}
L(\theta_{k+1}, T_{k+1}) - L( \theta_{k}, T_k)
\ge  \frac{\eta}{4} 	 \| \partial_T L(\theta_k, T_k) \|_{L^2(P)}^2 
	 - \frac{7 \tau }{4}  \| \partial_\theta L(\theta_k, T_k) \|^2.
\end{equation}

\vspace{2pt}
$\bullet$ Descent  of $\phi$: 
Because $\phi(\theta)$ is $L_1$-smooth by  Lemma \ref{lemma:T-fast-Tstar-phi} iii), we have
\begin{align}
\phi( \theta_{k+1}) - \phi(\theta_k) 
& \le 
\partial\phi (\theta_k)\cdot (\theta_{k+1} - \theta_k)
	+ \frac{L_1}{2} \|  \theta_{k+1} - \theta_k \|^2  \nonumber \\
& = - \tau \partial \phi (\theta_k) \cdot  \partial_\theta L( \theta_k, T_k) 
	+ \frac{L_1 \tau^2}{2} \|   \partial_\theta L( \theta_k, T_k) \|^2, 
	\label{eq:proof-T-fast-phi-descent-1}
\end{align}
where we used $\theta_{k+1} =\theta_k  - \tau \partial_\theta L( \theta_k, T_k) $ in the equality.
Define 
\[
\delta_k 
:= \partial_\theta L( \theta_k, T_k) - \partial \phi( \theta_k).
\] 
By Lemma \ref{lemma:T-fast-Tstar-phi} iii), $\delta_k= \partial_\theta L( \theta_k, T_k) - \partial_\theta L( \theta_k, T^*(\theta_k))$, and then by that $L$ is   coordinate  $l$-smooth, we have
\[
\| \delta_k\| \le l \| T_k - T^*(\theta_k) \|_{L^2(P)}.
\]
Meanwhile, by \eqref{eq:EB-T}, $\| T_k - T^*(\theta_k) \|_{L^2(P)} \le \frac{1}{\mu}  \| \partial_T L(\theta_k, T_k)\|_{L^2(P)}$.
Putting together, we have
\begin{equation}\label{eq:proof-T-fast-eqn4}
\| \delta_k\| \le \kappa  \| \partial_T L(\theta_k, T_k)\|_{L^2(P)}.
\end{equation}
Inserting the definition of $\delta_k$ into \eqref{eq:proof-T-fast-phi-descent-1} gives that
\[
\phi( \theta_{k+1}) - \phi(\theta_k) 
\le -\tau \| \partial \phi (\theta_k) \|^2
	- \tau  \partial \phi (\theta_k) \cdot \delta_k 
 	+ \frac{\tau}{2} \| \partial \phi (\theta_k) + \delta_k \|^2,
\]
where for the last term on the r.h.s. we used that $L_1 \tau \le 1$ under \eqref{eq:proof-T-fast-condition-eta-tau}.
Expanding $\| \partial \phi (\theta_k) + \delta_k \|^2 
= \| \partial \phi (\theta_k) \|^2 + \| \delta_k \|^2 + 2 \partial \phi (\theta_k) \cdot \delta_k $ gives that 
\begin{equation}\label{eq:proof-T-fast-phi-descent-3}
\phi( \theta_{k+1}) - \phi(\theta_k) 
\le
- \frac{\tau}{2} \| \partial \phi (\theta_k) \|^2
+ \frac{\tau}{2}\| \delta_k \|^2.
\end{equation}

\vspace{2pt}
$\bullet$ Lyapunov function:
We construct the Lyapunov function as
\[
V_k: = V( \theta_k, T_k) = \phi( \theta_k) + \alpha( \phi( \theta_k) - L(\theta_k, T_k)), 
\]
where $\alpha > 0$ will be chosen later. Inserting \eqref{eq:proof-T-fast-primal-descent-3}\eqref{eq:proof-T-fast-phi-descent-3},
\begin{align*}
& V_k - V_{k+1}
 = (1+\alpha) (\phi( \theta_k) - \phi( \theta_{k+1}) ) - \alpha (L(\theta_k, T_k) - L(\theta_{k+1}, T_{k+1}))  \\
& \ge 
	(1+\alpha) ( \frac{\tau}{2} \| \partial \phi (\theta_k) \|^2
 	- \frac{\tau}{2} \| \delta_k \|)
	+ \alpha (  \frac{\eta}{4} 	 \| \partial_T L(\theta_k, T_k) \|_{L^2(P)}^2 
	 - \frac{7 \tau }{4}  \| \partial_\theta L(\theta_k, T_k) \|^2 ) \\
& \ge 
	\tau \frac{1-6\alpha}{2} \| \partial \phi (\theta_k) \|^2
  	- \tau  \frac{1+8\alpha}{2} \| \delta_k \|^2 
	+ \eta \frac{\alpha}{4} 	\| \partial_T L(\theta_k, T_k) \|_{L^2(P)}^2, 
\end{align*}
where the second inequality is by that
$
 \| \partial_\theta L(\theta_k, T_k) \|^2
 = \| \partial \phi (\theta_k)  + \delta_k \|^2
 \le 2 ( \| \partial \phi (\theta_k) \|^2 + \| \delta_k \|^2 ).
$
Inserting \eqref{eq:proof-T-fast-eqn4}, we then have
\[
V_k - V_{k+1}
\ge  \frac{\tau( 1-6\alpha)}{2} \| \partial \phi (\theta_k) \|^2
	+ (  \frac{\eta \alpha}{4} -  \frac{ \tau \kappa^2(1+8\alpha)}{2}  )
			 \| \partial_T L(\theta_k, T_k) \|_{L^2(P)}^2. 
\]

Now choose $\alpha = 1/8$, we have
\[
V_k - V_{k+1}
\ge \frac{\tau}{8} \| \partial \phi (\theta_k) \|^2 
	+  (  \frac{\eta}{32} -  \tau \kappa^2   )
			 \| \partial_T L(\theta_k, T_k) \|_{L^2(P)}^2.
\]
We set  $\eta = 1/l$, $\tau = {1}/{(40 \kappa^2 l)}$,
 and one can verify that this fulfills \eqref{eq:proof-T-fast-condition-eta-tau} and also $\eta/32 = \frac{5}{4} \kappa^2 \tau$.
 As a result,
\begin{equation}
V_k - V_{k+1}
\ge \frac{\tau}{8} ( \| \partial \phi (\theta_k) \|^2 
	+    {2 \kappa^2}   
			 \| \partial_T L(\theta_k, T_k) \|_{L^2(P)}^2).
\end{equation}
By summing $k$ from 0 to $K-1$, we have
\[
\frac{\tau}{8} 
\sum_{k=0}^{K-1} 
 ( \| \partial \phi (\theta_k) \|^2 
	+    {2 \kappa^2}   
			 \| \partial_T L(\theta_k, T_k) \|_{L^2(P)}^2)
\le V_0-V_K.
\]
 We now show that $V_0-V_K$ is upper bounded by a constant independent of $K$:
 Since $\phi$ has a finite lower bound, denoted as $\phi^*$,
 we have
  \[
  V(\theta, T) = \phi(\theta) + \alpha (\phi(\theta) - L(\theta, T)) 
  \ge \phi(\theta) \ge \phi^*.
  \]
  and then
 \[
 V_0 - V_K
 \le V_0 - \phi^*
 = \phi(\theta_0) - \phi^* + \alpha (\phi(\theta_0) - L(\theta_0, T_0)) ) 
 := \Delta_\phi.
 \]
 The constant $\Delta_\phi$ depends on the initial value $(\theta_0, T_0)$.
 This implies that 
\begin{equation*}
\frac{1}{K} \sum_{k=0}^{K-1} 
 ( \| \partial \phi (\theta_k) \|^2 
	+    {2 \kappa^2}   
			 \| \partial_T L(\theta_k, T_k) \|_{L^2(P)}^2)
\le \frac{8 \Delta_\phi}{\tau K} 
= \frac{320 \kappa^2 l \Delta_\phi}{K}.
\end{equation*} 
As a result, when $\frac{320 \kappa^2 l \Delta_\phi}{K} \le \varepsilon^2$, there must be a $k \le K $ s.t.
\begin{equation}\label{eq:proof-T-fast-find-k}
\| \partial \phi (\theta_k) \|^2 
	+    {2 \kappa^2}   
			 \| \partial_T L(\theta_k, T_k) \|_{L^2(P)}^2 \le \varepsilon^2.
\end{equation}
This $k$ satisfies that  $\| \partial \phi (\theta_k) \|  \le \varepsilon$ as claimed in i).

To show that it also satisfies ii),  we derive an upper bound of $\| \partial_\theta L (\theta_k, T_k) \|$: 
For any $( \theta, T) \in  \R^p \times L^2(P)$,
we have
\[
\| \partial_\theta L(  \theta, T)- \partial_\theta L( \theta, T^*(\theta)) \|
\le l \| T -  T^*(\theta) \|_{L^2(P)}
\le  \kappa \| \partial_T L(\theta, T)\|_{L^2(P)}, 
\]
where the first inequality is by the  coordinate  $l$-smoothness of $L$ and the second inequality is by \eqref{eq:EB-T}.
 Apply to $(\theta, T )=(\theta_k, T_k )$, and by triangle inequality, we have
 \begin{align*}
\| \partial_\theta L (\theta_k, T_k) \|
& \le \| \partial_\theta L( \theta_k, T^*(\theta_k)) \| +  \kappa \| \partial_T L(\theta_k, T_k)\|_{L^2(P)} \\
& = \| \partial \phi(\theta_k)\| + \kappa \| \partial_T L(\theta_k, T_k)\|_{L^2(P)},
 \end{align*}
 which implies that $
\| \partial_\theta L (\theta_k, T_k) \|^2 
\le 2 (\| \partial \phi(\theta_k)\|^2 + \kappa^2 \| \partial_T L(\theta_k, T_k)\|_{L^2(P)}^2) 
 $.
 Comparing to \eqref{eq:proof-T-fast-find-k},  we have
 \[
 \varepsilon^2 
\ge  \frac{1}{2} \| \partial_\theta L( \theta_k, T_k)  \|^2
	+    \kappa^2   
			 \| \partial_T L(\theta_k, T_k) \|_{L^2(P)}^2,
 \]
and this proves ii).
 \end{proof}

\subsection{Proofs in Section \ref{subsec:theta-fast-NC-SC}}

\begin{proof}[Proof of Lemma \ref{lemma:theta-fast-Tgood-Phi}]
In i),
the uniqueness of $\theta^*[T]$ follows the strong concavity of $M(T, \cdot)$ in $\theta$.
  Note that the coordinate $l$-smoothness of $M(T,\theta)$ implies 
Lipschitz-continuous gradient (Remark \ref{rk:l-smoothness})
and, in particular, 
that $M(T,\theta)$ has a Lipschitz-continuous gradient in $\theta$.
Then, by \cite[Theorem 2]{karimi2016linear}, $\mu$-strong concavity implies the Error-Bound (EB) in $\theta$
namely the  inequality \eqref{eq:EB-M-T-theta}.

Because the strong concavity in $\theta$ holds for $M(T, \cdot)$ whenever $T \in \bar \calT$,
we have that for any $T \in \calT$,
 $\theta^*[T]$ is well-defined on a neighborhood of $T$
 and $\Psi(T) = M(T, \theta^*[T])$ on the neighborhood.
 Thus, $\Psi$ is differentiable at $T$
 and the expression of $\partial \Psi(T)$ is by Danskin's type argument.

Now for $\tilde T, T \in \calT$, let $\theta^* = \theta^*[T]$ and $\tilde \theta^* = \theta^*[\tilde T]$.
By \eqref{eq:EB-M-T-theta}, we have
\[
\|\partial_\theta M( T, \tilde \theta^*)\| \ge \mu \| \tilde \theta^* - \theta^*\|.
\]
Meanwhile, by that $\partial_\theta M( \tilde T, \tilde \theta^*) = 0$ (since $\tilde \theta^*$ is the maximizer),
\[
\|\partial_\theta M( T, \tilde \theta^*)\|
= \|\partial_\theta M( T, \tilde \theta^*) - \partial_\theta M( \tilde T, \tilde \theta^*) \|
\le l\| T - \tilde T\|_{L^2(P)},
\]
where the inequality is by that $M$ is  coordinate  $l$-smooth. Putting together, this proves ii).

To prove iii), by definition, 
\begin{align*}
\Psi(\tilde T) - \Psi( T) 
=M( \tilde T, \tilde \theta^*) - M(T, \theta^*)
= ( M( \tilde T, \tilde \theta^*)  - M( \tilde T, \theta^*)  )
   + ( M( \tilde T, \theta^*)  - M(  T, \theta^*) ).
\end{align*}
To bound the first term,
Lemma \ref{lemma:desdent-coordinate-M} i) gives that
$
- M( \tilde T, \tilde \theta^*)  + M( \tilde T, \theta^*)
\ge \partial_\theta M(\tilde T,  \tilde \theta^*) \cdot ( \theta^* - \tilde \theta^*) 
- \frac{l}{2} \| \theta^* - \tilde \theta^* \|^2, 
$
and since $\partial_\theta M(\tilde T,  \tilde \theta^*) = 0$, this gives 
\[
M( \tilde T, \tilde \theta^*)  - M( \tilde T, \theta^*)
\le  \frac{l}{2} \| \theta^* - \tilde \theta^* \|^2.
\]
The second term can be bounded as 
\[
 M( \tilde T,  \theta^*) - M( T,  \theta^*) 
 \le \langle \partial_T M(T, \theta^*), \tilde T - T \rangle_{L^2(P)} + \frac{l}{2} \| \tilde T - T \|_{L^2(P)}^2
\]
by Lemma \ref{lemma:desdent-coordinate-M} ii).
Putting together, and note that $ \partial_T M(T, \theta^*) = \partial \Psi(T) $, we have
\[
\Psi(\tilde T) - \Psi( T) 
\le   \frac{l}{2} \| \theta^* - \tilde \theta^* \|^2
+ \langle \partial \Psi(T)  , \tilde T - T \rangle_{L^2(P)} + \frac{l}{2} \| \tilde T - T \|_{L^2(P)}^2.
\]
Finally, we already have  $\| \theta^* - \tilde \theta^* \| \le \kappa \| \tilde T - T \|_{L^2(P)}$ proved in ii),
and this leads to iii).
\end{proof}

\begin{proof}[Proof of Theorem \ref{thm:theta-fast-NC-SC}]
The proof much follows the steps  in the proof of Theorem \ref{thm:rate-T-fast}.
However, since the roles of $T$ and $\theta$ is switched and these two variables lie in different spaces (one is in $L^2(P)$ and one is in vector space), the setting still differs as reflected by different assumptions.
We include the proof details for completeness.

Under the assumptions of the theorem, Lemma \ref{lemma:theta-fast-Tgood-Phi}  applies.
Recall that $\kappa = l/\mu$, $L_2 = l (1+\kappa^2)$.
We assume that
\begin{equation}\label{eq:proof-theta-fast-NC-SC-condition-eta-tau}
L_2 \eta \le 1,  \quad l \tau \le 1,
\end{equation}
which will be fulfilled later in the proof.

\vspace{2pt}
$\bullet$ Descent  of $\Psi$: 
Since $T_k, T_{k+1} \in \calT$, by Lemma \ref{lemma:theta-fast-Tgood-Phi} iii),
\begin{align}
\Psi( T_{k+1}) - \Psi(T_k) 
& \le 
\langle \partial \Psi (T_k), T_{k+1} - T_k  \rangle_{L^2(P)}
	+ \frac{L_2}{2} \|  T_{k+1} - T_k \|_{L^2(P)}^2  \nonumber \\
& = - \eta \langle \partial \Psi (T_k),  \partial_T M( T_k, \theta_k) \rangle_{L^2(P)}
	+ \frac{L_2 \eta^2}{2} \|   \partial_T M( T_k, \theta_k) \|_{L^2(P)}^2, \label{eq:proof-theta-fast-NC-SC-psi-descent-1}
\end{align}
where in the equality we used that $T_{k+1} = T_k -  \eta \partial_T M( T_k, \theta_k)$.
Define 
\[
\delta_k 
:= \partial_T M( T_k, \theta_k) - \partial \Psi( T_k).
\]
We know that $\delta_k= \partial_T M( T_k, \theta_k) - \partial_T M( T_k, \theta^*[T_k])$ by Lemma \ref{lemma:theta-fast-Tgood-Phi} i).
By the  coordinate  $l$-smoothness of $M$, we have 
\[
\| \delta_k\|_{L^2(P)}
= \| \partial_T M( T_k, \theta_k) - \partial_T M( T_k, \theta^*[T_k]) \|_{L^2(P)}
\le l \| \theta_k - \theta^*[T_k] \|.
\]
By \eqref{eq:EB-M-T-theta}, 
$\| \theta_k - \theta^*[T_k]\| \le \frac{1}{\mu} \| \partial_\theta M( T_k, \theta_k) \| $.
Thus,
\begin{equation}\label{eq:proof-theta-fast-NC-SC-eqn4}
\| \delta_k\|_{L^2(P)} \le \kappa \| \partial_\theta M( T_k, \theta_k) \|.
\end{equation}
Substituting the definition of $\delta_k$ into \eqref{eq:proof-theta-fast-NC-SC-psi-descent-1} gives that
\[
\Psi( T_{k+1}) - \Psi(T_k) 
\le -\eta \| \partial \Psi (T_k) \|_{L^2(P)}^2
	- \eta \langle \partial \Psi (T_k), \delta_k \rangle_{L^2(P)}
 	+ \frac{\eta}{2} \| \partial \Psi (T_k) + \delta_k \|_{L^2(P)}^2,
\]
where for the last term on the r.h.s. we used that $L_2 \eta \le 1$ under \eqref{eq:proof-theta-fast-NC-SC-condition-eta-tau}.
Expanding $\| \partial \Psi (T_k) + \delta_k \|_{L^2(P)}^2 = \| \partial \Psi (T_k) \|_{L^2(P)}^2 + \| \delta_k \|_{L^2(P)}^2 + 2 \langle \partial \Psi (T_k), \delta_k \rangle_{L^2(P)}$ gives that 
\begin{equation}\label{eq:proof-theta-fast-NC-SC-psi-descent-3}
\Psi( T_{k+1}) - \Psi(T_k) 
\le
- \frac{\eta}{2} \| \partial \Psi (T_k) \|_{L^2(P)}^2
+ \frac{\eta}{2}\| \delta_k \|_{L^2(P)}^2.
\end{equation}

\vspace{2pt}
$\bullet$ Descent of primal: 
In the $\theta$ step, by Lemma \ref{lemma:desdent-coordinate-M} i),
\begin{align}
M(T_{k+1}, \theta_{k+1}) - M(T_{k+1}, \theta_k)
& \ge \partial_\theta M(T_{k+1}, \theta_k) \cdot (\theta_{k+1}- \theta_k) -  \frac{l}{2} \| \theta_{k+1} - \theta_k\|^2  \nonumber \\
& = \tau  \partial_\theta M(T_{k+1}, \theta_k) \cdot \partial_\theta M(T_k, \theta_k) 
	-  \frac{l \tau^2 }{2}  \|  \partial_\theta M(T_k, \theta_k) \|^2, 
	\label{eq:proof-theta-fast-NC-SC-theta-step-1}
\end{align}
where in the equality we used that $\theta_{k+1} - \theta_k = \tau \partial_\theta M(T_k, \theta_k)$.
Define 
\[
e_k: =  \partial_\theta M( T_{k+1}, \theta_k) - \partial_\theta M(T_k, \theta_k).
\]
By that $M$ is  coordinate  $l$-smooth on $L^2(P) \times \R^p$
and that $T_{k+1} - T_k = -  \eta \partial_T M( T_k, \theta_k)$,
\begin{equation}\label{eq:proof-theta-fast-NC-SC-eqn5}
\| e_k \| 
\le  l \| T_{k+1} - T_{k}\|_{L^2(P)} 
= l \eta \| \partial_T M( T_k, \theta_k) \|_{L^2(P)}.
\end{equation}
Back to \eqref{eq:proof-theta-fast-NC-SC-theta-step-1}, we have
\[
M(T_{k+1}, \theta_{k+1}) - M(T_{k+1}, \theta_k)
\ge  \tau (1-\frac{l \tau }{2}) \| \partial_\theta M(T_k, \theta_k) \|^2 
	+ \tau e_k \cdot \partial_\theta M(T_k, \theta_k).
\]
Note that 
\begin{align*}
-  e_k \cdot \partial_\theta M(T_k, \theta_k) 
& \le  \| e_k \| \| \partial_\theta M(T_k, \theta_k)\| 
 \le l \eta \| \partial_T M( T_k, \theta_k) \|_{L^2(P)} \| \partial_\theta M(T_k, \theta_k)\|  \\
& \le \frac{l \eta}{2} ( \| \partial_T M( T_k, \theta_k) \|_{L^2(P)}^2 + \| \partial_\theta M(T_k, \theta_k)\|^2),
\end{align*}
where the second inequality is by \eqref{eq:proof-theta-fast-NC-SC-eqn5}.
We then have
\begin{align}
M(T_{k+1}, \theta_{k+1}) - M(T_{k+1}, \theta_k)
& \ge  \tau (1-\frac{l \tau }{2} - \frac{l \eta}{2}) \| \partial_\theta M(T_k, \theta_k) \|^2 
	-  \tau \frac{l \eta}{2}  \| \partial_T M( T_k, \theta_k) \|_{L^2(P)}^2 \nonumber \\
& \ge \frac{\tau}{4} 	 \| \partial_\theta M(T_k, \theta_k) \|^2 
 	- \frac{\eta}{2} \| \partial_T M( T_k, \theta_k) \|_{L^2(P)}^2,
	\label{eq:proof-theta-fast-NC-SC-theta-step-2}
\end{align}
where the second inequality is by that $l \tau \le 1$,
and that $ l \eta \le \frac{1}{ 1+\kappa^2 } \le 1/2$ under \eqref{eq:proof-theta-fast-NC-SC-condition-eta-tau}.

In the $T$-step, by Lemma \ref{lemma:desdent-coordinate-M} ii),
\begin{align}
 M( T_{k+1},  \theta_k) - M( T_k,  \theta_k)
 &\ge   \langle \partial_T M(T_k, \theta_k),  T_{k+1} - T_k \rangle_{L^2(P)} 
 	- \frac{l}{2} \| T_{k+1} - T_k \|_{L^2(P)}^2 \nonumber \\
& = 	-\eta (1+ \frac{l \eta}{2}) \| \partial_T M(T_k, \theta_k) \|_{L^2(P)}^2 \nonumber \\
& \ge - \frac{5 \eta }{4}  \| \partial_T M(T_k, \theta_k) \|_{L^2(P)}^2 \nonumber,
 \end{align}
where in the last inequality we used $ l \eta \le 1/2$ again.
Together  with \eqref{eq:proof-theta-fast-NC-SC-theta-step-2}, we have
\begin{equation}\label{eq:proof-theta-fast-NC-SC-primal-descent-3}
M(T_{k+1}, \theta_{k+1}) - M(T_{k}, \theta_k)
\ge  \frac{\tau}{4} 	 \| \partial_\theta M(T_k, \theta_k) \|^2 
	 - \frac{7 \eta }{4}  \| \partial_T M(T_k, \theta_k) \|_{L^2(P)}^2.
\end{equation}

\vspace{2pt}
$\bullet$ Lyapunov function:
We define
\[
V_k: = V( T_k, \theta_k) = \Psi( T_k) + \alpha( \Psi(T_k) - M(T_k, \theta_k)),
\]
where $\alpha > 0$ to be determined. By \eqref{eq:proof-theta-fast-NC-SC-psi-descent-3}\eqref{eq:proof-theta-fast-NC-SC-primal-descent-3},
\begin{align*}
& V_k - V_{k+1}
 = (1+\alpha) (\Psi( T_k) - \Psi( T_{k+1}) ) - \alpha (M(T_k, \theta_k) - M(T_{k+1}, \theta_{k+1})) \\
& \ge 
	(1+\alpha) ( \frac{\eta}{2} \| \partial \Psi (T_k) \|_{L^2(P)}^2
 	- \frac{\eta}{2} \| \delta_k \|_{L^2(P)}^2 )
	+ \alpha (  \frac{\tau}{4} 	 \| \partial_\theta M(T_k, \theta_k) \|^2 
	 - \frac{7 \eta }{4}  \| \partial_T M(T_k, \theta_k) \|_{L^2(P)}^2 ) \\
& \ge 
	\eta \frac{1-6\alpha}{2} \| \partial \Psi (T_k) \|_{L^2(P)}^2
  	- \eta  \frac{1+8\alpha}{2} \| \delta_k \|_{L^2(P)}^2 
	+ \tau \frac{\alpha}{4} 	 \| \partial_\theta M(T_k, \theta_k) \|^2, 
\end{align*}
where in the second inequality we used that 
\[
 \| \partial_T M(T_k, \theta_k) \|_{L^2(P)}^2
 = \| \partial \Psi (T_k)  + \delta_k \|_{L^2(P)}^2
 \le 2 ( \| \partial \Psi (T_k) \|_{L^2(P)}^2 + \| \delta_k \|_{L^2(P)}^2 ).
\]
Inserting \eqref{eq:proof-theta-fast-NC-SC-eqn4}, we then have
\[
V_k - V_{k+1}
\ge  \frac{\eta( 1-6\alpha)}{2} \| \partial \Psi (T_k) \|_{L^2(P)}^2
	+ (  \frac{\tau \alpha}{4} -  \frac{ \eta \kappa^2(1+8\alpha)}{2}  )
			 \| \partial_\theta M(T_k, \theta_k) \|^2. 
\]

The rest of the proof takes the same steps as the part in the proof of Theorem \ref{thm:rate-T-fast}, 
swapping the role of $\theta$ and $T$ (and $\eta$ and $\tau$). 
To be complete, we set $\alpha = 1/8$, 
 $\tau = 1/l$, $\eta = {1}/{(40 \kappa^2 l)}$,
 which fulfills \eqref{eq:proof-theta-fast-NC-SC-condition-eta-tau} and also makes $\tau/32 = \frac{5}{4} \kappa^2 \eta$.
 Inserting these into the descent of $V_k$ gives that 
 \begin{equation}
V_k - V_{k+1}
\ge \frac{\eta}{8} ( \| \partial \Psi (T_k) \|_{L^2(P)}^2 
	+    {2 \kappa^2}   
			 \| \partial_\theta M(T_k, \theta_k) \|^2).
\end{equation}
By summing $k$ from 0 to $K-1$, and that 
\[
 V_0 - V_K
 \le V_0 - \Psi^*
 = \Psi(T_0) - \Psi^* + \alpha (\Psi(T_0) - M(T_0, \theta_0)) 
 := \Delta_0,
 \]
a constant depending on the initial value $(T_0, \theta_0)$,
 we have that 
\begin{equation*}
\frac{1}{K} \sum_{k=0}^{K-1} 
 ( \| \partial \Psi (T_k) \|_{L^2(P)}^2 
	+    {2 \kappa^2}   
			 \| \partial_\theta M(T_k, \theta_k) \|^2)
\le \frac{8 \Delta_0}{\eta K} 
= \frac{320 \kappa^2 l \Delta_0}{K}.
\end{equation*} 
As a result, when $\frac{320 \kappa^2 l \Delta_0}{K} \le \varepsilon^2$, there must be a $k \le K $ s.t.
\begin{equation}\label{eq:proof-theta-fast-NC-SC-find-k}
\| \partial \Psi (T_k) \|_{L^2(P)}^2 
	+    {2 \kappa^2}   
			 \| \partial_\theta M(T_k, \theta_k) \|^2 \le \varepsilon^2.
\end{equation}
This $k$ satisfies that  $\| \partial \Psi (T_k) \|_{L^2(P)}  \le \varepsilon$ as claimed in i).

To show that it also satisfies ii), 
we derive an upper bound of $\| \partial_T M (T_k, \theta_k) \|_{L^2(P)}$:
For any $T \in \calT$ and any $\theta \in \R^p$,
\[
\| \partial_T M( T, \theta)- \partial_T M( T, \theta^*[T]) \|_{L^2(P)}
\le l \| \theta -  \theta^*[T] \|
\le  \kappa \| \partial_\theta M(T, \theta)\|, 
\]
where the first inequality is by the  coordinate  $l$-smoothness of $M$, and the second inequality is by \eqref{eq:EB-M-T-theta}.
Apply to $(T,\theta)= (T_k, \theta_k)$, and use triangle inequality, we have
\begin{align*}
\| \partial_T M( T_k, \theta_k)  \|_{L^2(P)}
& \le
	\| \partial_T M( T_k, \theta^*[T_k]) \|_{L^2(P)}
	+  \kappa \| \partial_\theta M( T_k, \theta_k) \| \\
& = 	\| \partial \Psi( T_k ) \|_{L^2(P)} + \kappa \| \partial_\theta M( T_k, \theta_k) \|,
\end{align*}
which gives that
$
\| \partial_T M( T_k, \theta_k)  \|_{L^2(P)}^2
\le 2( \| \partial \Psi( T_k ) \|_{L^2(P)}^2 + \kappa^2 \| \partial_\theta M( T_k, \theta_k) \|^2).
$
Comparing to \eqref{eq:proof-theta-fast-NC-SC-find-k}, we have
\[
\varepsilon^2 
\ge  \frac{1}{2} \| \partial_T M( T_k, \theta_k)  \|_{L^2(P)}^2
	+    \kappa^2   
			 \| \partial_\theta M(T_k, \theta_k) \|^2,
\]
and this proves ii).
\end{proof}

\begin{proof}[Proof of Corollary \ref{cor:theta-fast-NC-SC}]
The corollary is a direct application of Theorem \ref{thm:theta-fast-NC-SC} where $M$ is $H$
and $\calT$ is $\calT_{\mu}^{\rm SC}$.
The  coordinate  $l$-smoothness of $H$ is by Lemma \ref{lemma:theta-fast-lsmooth-H},
and the needed conditions on $\calT$ are satisfied under Assumptions 
\ref{assump:l-C2},\ref{assump:SC-good-T},\ref{assump:Tgood-in-open-set-SC}.
\end{proof}

\subsection{Proofs in Section \ref{subsec:theta-fast-NC-NC}}
\begin{proof}[Proof of Lemma \ref{lemma:Hs-prox-diff}]
We first verify that  $T^+(T,\theta)$  is well-defined, namely $\min_{V \in L^2(P)} h( V; T, \theta)$ has unique minimizer.
We claim that for fixed $(T, \theta) \in L^2(P) \times \R^p$,
the field  
\begin{align}
\partial_V h(V) 
& = \partial_V h(V ; T, \theta) 
 = \partial_T H(V,\theta) +\frac{1}{s}(V-T)
    \label{eq:expression-partialVh-2} \\
& = - \partial_v \ell( \theta, V(x) ) + \frac{1}{\gamma} (V(x) - x)
	+ \frac{1}{s}(V(x) - T(x) ), \quad P-a.s.
	\label{eq:expression-partialVh-3}
\end{align}
as a mapping from $L^2(P)$ to $L^2(P)$ satisfies that 

a) it is continuous on $L^2(P)$,

b) it is coercive, namely $ \lim_{\| V \|_{L^2(P)} \to \infty} \frac{ \langle \partial_V h(V; T, \theta), V \rangle_{L^2(P)}}{\| V\|_{L^2(P)}} \to \infty$,

c) it is $(1/\gamma + 1/s - \rho)$- strongly monotone on $L^2(P)$.

If true, then $\partial_V h$ has a zero point $V^* \in L^2(P)$ by Browder-Minty \cite[Theorem 5.16]{brezis2011functional},  and $V^*$ is also the unique zero point due to the strongly monotonicity of $\partial_V h$. 
Then, since the strongly monotonicity of $\partial_V h$ also implies strong convexity of $h( \cdot ; T,\theta)$ on $L^2(P)$, one can verify that $V^*$ is the unique minimizer of $h(\cdot; T, \theta)$.
This $V^*$ is then defined as $T^+(T ,\theta)$.

To finish proving the well-definedness of the proximal mapping $T^+$, it remains to show a)b)c) above. 
We first prove c):
omitting the dependence on $(T, \theta)$ in the notation of $\partial_V h$, we want to show that 
\begin{equation}\label{eq:strong-monotone-partialVh-goal}
\langle \partial_V h(\tilde V) - \partial_V h(V), \tilde V  - V \rangle_{L^2(P)} \ge 
(1/\gamma + 1/s - \rho) \| \tilde V - V\|_{L^2(P)}^2,
\quad \forall \tilde V, V \in L^2(P).
\end{equation}
By \eqref{eq:expression-partialVh-3}, we have
\[
\partial_V h(\tilde V) - \partial_V h(V)
= \underbrace{ - \partial_v \ell( \theta, \tilde V(x) )  +   \partial_v \ell( \theta, V(x) ) }_{\textcircled{1}}
 + \underbrace{ (\frac{1}{\gamma} + \frac{1}{s})( \tilde V(x) - V(x))}_{\textcircled{2}}.
\]
The term $\textcircled{2}$ will contribute a $(1/\gamma+1/s) \| \tilde V - V\|_{L^2(P)}^2$ term after taking the $L^2$ inner-product with $\tilde V - V$.
The first term $\textcircled{1}$ satisfies that for each $x$,
\[
(- \partial_v \ell( \theta, \tilde V(x) )  +   \partial_v \ell( \theta, V(x) ) )\cdot (\tilde V(x) - V(x) ) 
\ge - \rho \| \tilde V(x) - V(x)\|^2
\]
due to that $\ell(\theta, v)$ is $\rho$-weakly concave in $v$.
This implies that 
\[
 \langle \textcircled{1}, \tilde V - V \rangle_{L^2(P)} \ge -\rho \| \tilde V - V\|_{L^2(P)}^2.
\]
Putting together, this proves \eqref{eq:strong-monotone-partialVh-goal}, so c) holds.

Next, we show that c) implies b):
Define $\nu := 1/\gamma + 1/s - \rho$,
by the $\nu$-strongly monotonicity of $\partial_V h$ on $L^2(P)$, we have
$
\langle \partial_V h( V) - \partial_V h( 0 ), V  \rangle_{L^2(P)} \ge \nu \| V \|_{L^2(P)}^2$,
and this means that 
\[
\langle \partial_V h( V) , V  \rangle_{L^2(P)} 
\ge \langle \partial_V h( 0) , V  \rangle_{L^2(P)}  + \nu \| V \|_{L^2(P)}^2.
\]
Meanwhile, by \eqref{eq:expression-partialVh-3},
$
\partial_V h( 0 )(x)
=  - \partial_v  \ell( \theta, 0) -  \frac{1}{\gamma} x
 - \frac{1}{s}  T(x)$
 is a field in $L^2(P)$ independent from $V$. Thus, 
 $
| \langle \partial_V h( 0) , V  \rangle_{L^2(P)}   |
\le \|  \partial_V h( 0) \|_{L^2(P)}   \| V \|_{L^2(P)}  $ which grows linearly in $\| V\|_{L^2(P)} $.
Thus, $\langle \partial_V h( V) , V  \rangle_{L^2(P)} $ will be dominated by the quadratic growth from $\nu \| V \|_{L^2(P)}^2$ as $\| V\|_{L^2(P)} \to \infty$, 
and this proves b).

Finally, for a), we show that $\partial_V h$ is Lipschitz in $V$: 
by \eqref{eq:expression-partialVh-3},
it suffices to show that $ \partial_v  \ell( \theta, V(x))$ as a field is  Lipschitz in $V$.
For any $x$, by that $\ell(\theta, v)$ is $l_0$-smooth on $\R^p \times \R^d$,
we have
$
\| \partial_v  \ell( \theta , \tilde V(x))
-  \partial_v  \ell( \theta , V(x)) \|
\le l_0  \| \tilde V(x) - V(x)\|$.
Thus,
$
\E_{x \sim P} \| \partial_v  \ell( \theta, \tilde V(x))
			-  \partial_v  \ell( \theta, V(x)) \|^2
\le l_0^2  \| \tilde V - V\|_{L^2(P)}^2$,
 which gives the Lipschitz continuity in $V$ on $L^2(P)$.

We have shown that $T^+$ is well-defined,
and to finish the proof of the lemma, we are to prove \eqref{eq:partial_T_Hs}\eqref{eq:partial_theta_Hs}.
The equation \eqref{eq:partial_T_Hs} follows from that $H_s$ can be viewed as the Moreau envelope of $H(\cdot, \theta)$ for fixed $\theta$.
Specifically, the second equality in  \eqref{eq:partial_T_Hs} is by the first-order optimality condition in the definition of $V^*=T^+(T, \theta)$,
and the first equality is by Danskin's type argument.
The equation \eqref{eq:partial_theta_Hs} again follows by Danskin's type argument.
\end{proof}

\begin{proof}[Proof of Lemma \ref{lemma:theta-fast-lsmooth-Hs}]
We consider the four inequalities in Definition \ref{def:coord-L-smoothness-M},
and denote the corresponding Lipschitz constants as $l_{11}$, $l_{12}$, $l_{21}$, $l_{22}$, respectively. 
We will derive the constants $l_{11}$, $l_{12}$, $l_{21}$, $l_{22}$, and then take $\bar l$ to be the maximum of the four.

Recall that $\forall V \in L^2(P)$, $\forall \theta \in \R^p$,
\begin{equation}\label{eq:partial_T_H-1}
\partial_T H(V,\theta)(x)
=  - \partial_v \ell( \theta, V(x) ) + \frac{1}{\gamma} (V(x) - x), \quad P-a.s.
\end{equation}
\begin{equation}\label{eq:partial_theta_H-1}
 \partial_\theta H( V, \theta) 
=- \E_{x \sim P} \partial_\theta \ell( \theta, V(x)).
 \end{equation}

\vspace{5pt}
\noindent
$\bullet$ $l_{11}$:
We want to show that 
\[
\| \partial_T H_s (\tilde T,  \theta) - \partial_T H_s ( T, \theta)  \|_{L^2(P)}
	  \le l_{11} \| \tilde T - T\|_{L^2(P)},
\quad \forall \tilde T, T \in L^2(P), \forall  \theta \in \R^p.
\]
By \eqref{eq:partial_T_Hs} in Lemma \ref{lemma:Hs-prox-diff},
 we have, denoting  $  T^+( T, \theta)$ as $ T^+$ and $T^+(\tilde T, \theta)$ as  $\tilde T^+$,
\[
\partial_T H_s (\tilde T,  \theta) - \partial_T H_s ( T, \theta) 
= \frac{1}{s}( \tilde T - \tilde T^+) - \frac{1}{s}( T - T^+),
\]
and thus
\begin{align}
& \| \partial_T H_s (\tilde T,  \theta) - \partial_T H_s ( T, \theta) \|_{L^2(P)}^2
 = \frac{1}{s^2}\| ( \tilde T - T )  - (\tilde T^+ - T^+ )\|^2  \nonumber \\
& =  \frac{1}{s^2}
\big( \| \tilde T   - T \|_{L^2(P)}^2
- 2 \langle \tilde T   - T , \tilde T^+ - T^+  \rangle_{L^2(P)}
 + \| \tilde T^+ - T^+ \|_{L^2(P)}^2 \big). \label{eq:l11-proof-1}
\end{align}
Recall that we have shown in \eqref{eq:strong-monotone-partialVh-goal} that $\partial_V h(\cdot; T, \theta)$ is 
$(1/\gamma + 1/s - \rho)$-strongly monotone on $L^2(P)$.
Then,
\begin{align*}
(\frac{1}{s}-\rho + \frac{1}{\gamma}) \| \tilde T^+ - T^+\|_{L^2(P)}^2
& \le \langle \partial_V h (\tilde T^+; T, \theta) - \partial_V h (T^+; T, \theta),  \tilde T^+ - T^+\rangle_{L^2(P)} \\
& = \langle \partial_T H( \tilde T^+, \theta) + \frac{1}{s} (\tilde T^+ - T) ,  \tilde T^+ - T^+\rangle_{L^2(P)} 
\end{align*}
where in the equality we used that $\partial_V h (T^+; T, \theta) = 0$
and the expression \eqref{eq:expression-partialVh-2}.
Meanwhile, \eqref{eq:partial_T_Hs} gives that 
\[
\partial_T H( \tilde T^+, \theta) 
=  \partial_T H( T^+(\tilde T, \theta), \theta) 
= \frac{1}{s} ( \tilde T -  \tilde T^+),
\]
and then we have
\[
(\frac{1}{s}-\rho + \frac{1}{\gamma}) \| \tilde T^+ - T^+\|_{L^2(P)}^2
\le 
\frac{1}{s}  \langle  \tilde T  - T ,  \tilde T^+ - T^+\rangle_{L^2(P)}. 
\]
this implies that
\begin{equation}\label{eq:l11-proof-2}
0 \le \| \tilde T^+ - T^+\|_{L^2(P)}^2 
\le  \frac{1}{1-s ( \rho - 1/\gamma)}  \langle  \tilde T  - T ,  \tilde T^+ - T^+\rangle_{L^2(P)}, 
\end{equation}
and furtherly, by Cauchy-Schwarz,
\begin{equation}\label{eq:dT+-bound-by-dT-1}
\| \tilde T^+ - T^+ \|_{L^2(P)}
=\| T^+( \tilde T, \theta ) - T^+( T, \theta ) \|_{L^2(P)}
\le \frac{1}{1 - s ( \rho - 1/\gamma) } \| \tilde T - T\|_{L^2(P)}.
\end{equation}
Inserting \eqref{eq:l11-proof-2} back to \eqref{eq:l11-proof-1}, we have that
$\| \partial_T H_s (\tilde T,  \theta) - \partial_\theta H_s ( T, \theta) \|_{L^2(P)}^2$
is upper bounded by 
\begin{equation}\label{eq:l11-proof-3}
  \frac{1}{s^2}
\Big( \| \tilde T   - T \|_{L^2(P)}^2
+( \frac{1}{1-s ( \rho - 1/\gamma) } - 2)
 \langle \tilde T   - T , \tilde T^+ - T^+  \rangle_{L^2(P)}
  \Big).
\end{equation}
There are two cases:

Case 1. $\frac{1}{1 - s ( \rho - 1/\gamma) } - 2 \le 0 $: then \eqref{eq:l11-proof-3} is upper bounded by 
$ \frac{1}{s^2}  \|  \tilde T - T   \|^2 $,
because $\langle \tilde T   - T , \tilde T^+ - T^+  \rangle_{L^2(P)} \ge 0$.

Case 2. $\frac{1}{1 - s ( \rho - 1/\gamma) } - 2 > 0 $: 
then by Cauchy-Schwarz and \eqref{eq:dT+-bound-by-dT-1},  \eqref{eq:l11-proof-3} is upper bounded by
$
\frac{1}{s^2}( 1 -\frac{1}{1-s ( \rho - 1/\gamma) } )^2 \| \tilde T   - T \|_{L^2(P)}^2$.

Combining the two cases, we have that the r.h.s. of \eqref{eq:l11-proof-3} is always upper-bounded by 
\[
\big( \frac{1}{s^2} \vee (\frac{  \rho - 1/\gamma }{1 - s ( \rho - 1/\gamma)  })^2   \big)
\|  \tilde T - T \|^2,
\]
and this means that we can set $l_{11} = \frac{1}{s} \vee \frac{ | \rho - 1/\gamma|}{1 - s ( \rho - 1/\gamma)  }  $.

\vspace{5pt}
\noindent
$\bullet$ $l_{12}$:
We want to show that 
\[
\| \partial_T H_s (T, \tilde \theta) - \partial_T H_s ( T, \theta)  \|_{L^2(P)} 
	\le l_{12} \| \tilde \theta - \theta\|,
\quad \forall T \in L^2(P), \forall \tilde \theta, \theta \in \R^p.
\]
 By \eqref{eq:partial_T_Hs} in Lemma \ref{lemma:Hs-prox-diff},
 we have 
 \[
  \partial_T H_s (T, \tilde \theta) - \partial_T H_s ( T, \theta)
  = - \frac{1}{s}( T^+( T, \tilde \theta ) -  T^+( T, \theta) ),
  \]
 and also
 \begin{align*}
 T^+( T, \tilde \theta ) -  T^+( T, \theta)
 & = -s (   \partial_T H_s (T, \tilde \theta) - \partial_T H_s ( T, \theta) ) \\
 & = -s ( \partial_T H( T^+(T, \tilde \theta), \tilde \theta ) -  \partial_T H( T^+(T, \theta), \theta )  ).
 \end{align*}
 Note that by the expression \eqref{eq:partial_T_H-1}, one can verify that $\forall \tilde V, V \in L^2(P)$,
 \begin{align}
&~~~ \langle \partial_T H( \tilde V,  \theta ) -  \partial_T H( V, \theta ), \tilde V - V \rangle_{L^2(P)} \nonumber \\
& = \E_{x \sim P}
	(- \partial_v \ell( \theta, \tilde V(x) ) +  \partial_v \ell( \theta, V(x) ))\cdot( \tilde V(x) - V(x))
	 + \frac{1}{\gamma} \|\tilde V(x) -  V(x)\|^2 \nonumber \\
& \ge (\frac{1}{\gamma} - \rho )	  \| \tilde V - V \|_{L^2(P)}^2. 
	 \label{eq:-rho-tildeV-V-bound-1}
 \end{align}
 We then have
 \begin{align*}
&   \|  T^+( T, \tilde \theta ) -  T^+( T, \theta)\|_{L^2(P)}^2
 =  -s \langle   \partial_T H( T^+(T, \tilde \theta), \tilde \theta ) -  \partial_T H( T^+(T, \theta), \theta )  ,
  	T^+( T, \tilde \theta ) -  T^+( T, \theta)
  	\rangle_{L^2(P)} \\
& = 	-s \langle   \partial_T H( T^+(T, \tilde \theta), \theta ) -  \partial_T H( T^+(T, \theta), \theta )  ,
  	T^+( T, \tilde \theta ) -  T^+( T, \theta)
  	\rangle_{L^2(P)} \\
&~~~	-s \langle   \partial_T H( T^+(T, \tilde \theta), \tilde \theta ) -  \partial_T H( T^+(T, \tilde \theta), \theta )  ,
  	T^+( T, \tilde \theta ) -  T^+( T, \theta)
  	\rangle_{L^2(P)} \\
& \le 	s ( \rho - 1/\gamma)   \|  T^+( T, \tilde \theta ) -  T^+( T, \theta)\|_{L^2(P)}^2
	+ s l_0 \|\tilde \theta - \theta \|\|  T^+( T, \tilde \theta ) -  T^+( T, \theta)\|_{L^2(P)},
 \end{align*}
 where in the last inequality we applied \eqref{eq:-rho-tildeV-V-bound-1} with 
  $\tilde V = T^+(T, \tilde \theta)$, $V = T^+(T, \theta)$,
  and we also used that 
 \[
\|  \partial_T H( T^+(T, \tilde \theta), \tilde \theta ) -  \partial_T H( T^+(T, \tilde \theta), \theta ) \|
\le l_0 \|\tilde \theta - \theta\|
 \]
 due to  \eqref{eq:coord-smooth-L-21}, which holds when replacing $L$ with $H$ since $H = -L$.
Because $s( \rho - 1/\gamma)  < s\rho < 1$, we then have
\begin{equation}\label{eq:tildeT+-T+-bound-by-dtheta}
\| T^+( T, \tilde \theta) - T^+(T, \theta)\|_{L^2(P)}
\le l_0 \frac{s}{1-s ( \rho - 1/\gamma)  } \| \tilde \theta  - \theta\|.
\end{equation}
As a result,
\[
\|  \partial_T H_s (T, \tilde \theta) - \partial_T H_s ( T, \theta) \|_{L^2(P)}
 = \frac{1}{s} \| T^+( T, \tilde \theta) - T^+(T, \theta)\|_{L^2(P)}
 \le \frac{l_0}{1-s ( \rho - 1/\gamma)  } \| \tilde \theta  - \theta\|.
\]
This means that we can set $l_{12} = l_0/(1-s ( \rho - 1/\gamma)  )$.

\vspace{5pt}
\noindent
$\bullet$ $l_{21}$:
We want to show that 
\[
\| \partial_\theta H_s (\tilde T,  \theta) - \partial_\theta H_s ( T, \theta)  \|  
	\le l_{21} \| \tilde T - T\|_{L^2(P)},
\quad \forall \tilde T, T \in L^2(P), \forall  \theta \in \R^p.
\]
  By \eqref{eq:partial_theta_Hs} in  Lemma \ref{lemma:Hs-prox-diff}  and \eqref{eq:partial_theta_H-1},
  \[
  \partial_\theta H_s (\tilde T,  \theta) - \partial_\theta H_s ( T, \theta) 
  = \partial_\theta H( T^+(\tilde T, \theta), \theta) - \partial_\theta H(T^+ (T, \theta), \theta).
  \]
  By \eqref{eq:coord-smooth-L-12} (which holds for $H$ in place of $L$ due to that $H = -L$), we have
  \[
\|   \partial_\theta H( T^+(\tilde T, \theta), \theta) - \partial_\theta H( T^+ (T, \theta), \theta) \|
\le l_0 \| T^+(\tilde T, \theta) - T^+ (T, \theta)\|_{L^2(P)}.
  \]
  Combined with \eqref{eq:dT+-bound-by-dT-1}, 
  we have 
  \[
\|   \partial_\theta H( T^+(\tilde T, \theta), \theta) - \partial_\theta H( T^+ (T, \theta), \theta) \|
\le \frac{l_0}{1- s ( \rho - 1/\gamma) } \| \tilde T -T \|_{L^2(P)},
  \]
which means that $l_{21}$ can be chosen to be $l_0/(1-s ( \rho - 1/\gamma) )$.

\vspace{5pt}
\noindent
$\bullet$ $l_{22}$:
We want to show that 
\[
\| \partial_\theta H_s (T, \tilde \theta) - \partial_\theta H_s ( T, \theta)  \|  
	\le l_{22} \| \tilde \theta - \theta\|,
\quad \forall T \in L^2(P), \forall \tilde \theta, \theta \in \R^p.
\]
 By \eqref{eq:partial_theta_Hs} in  Lemma \ref{lemma:Hs-prox-diff}  and \eqref{eq:partial_theta_H-1},
 \begin{align*}
 \partial_\theta H_s (T, \tilde \theta) - \partial_\theta H_s ( T, \theta)
& = \partial_\theta H( T^+(T, \tilde \theta), \tilde \theta) - \partial_\theta H ( T^+(T, \theta), \theta) \\
& = \E_{x \sim P} ( - \partial_\theta \ell( \tilde \theta, T^+(T, \tilde \theta)(x) ) 
			       +  \partial_\theta \ell( \theta, T^+(T, \theta)(x) ) ).
 \end{align*}
 Note that by the $l_0$-smoothness of $\ell( \theta, v)$, for any $P$-a.s. $x$,
 \[
 \|  \partial_\theta \ell( \tilde \theta, T^+(T, \tilde \theta)(x) ) 
  -  \partial_\theta \ell( \theta, T^+(T, \theta)(x) ) \|^2
  \le l_0^2 
  	( \| \tilde \theta - \theta\|^2 + \| T^+(T, \tilde \theta)(x) -  T^+(T, \theta)(x)   \|^2).
 \]
 Then, we have
  \begin{align*}
 \|  \partial_\theta H_s (T, \tilde \theta) - \partial_\theta H_s ( T, \theta)\|
& \le \E_{x \sim P} \| \partial_\theta \ell( \tilde \theta, T^+(T, \tilde \theta)(x) ) 
			       -  \partial_\theta \ell( \theta, T^+(T, \theta)(x) ) \| \\
& \le l_0 ( \| \tilde \theta - \theta\|^2 + \| T^+(T, \tilde \theta) - T^+(T, \theta) \|_{L^2(P)}^2)^{1/2} \\
& \le l_0 \| \tilde \theta - \theta\| (1 +  ( l_0 \frac{s}{1-s ( \rho - 1/\gamma) })^2  )^{1/2},
 \end{align*}
 where in the last inequality we used  \eqref{eq:tildeT+-T+-bound-by-dtheta}.
Thus, we can set $l_{22} = l_0 ( 1+ l_0 \frac{s}{1-s( \rho - 1/\gamma) })$.
\end{proof}

\begin{proof}[Proof of Lemma \ref{lemma:TNC-SC-Hs}]
Introducing the notation of perturbation $\theta$ to $\theta + \delta \theta$,
and $\delta T = 0$,
we want to show that for any local perturbation $\delta \theta$,
\begin{equation}\label{eq:proof-barmu-SC-goal}
( - \partial_\theta H_s( T, \theta + \delta \theta) +   \partial_\theta H_s( T, \theta ) )\cdot \delta \theta 
\ge \bar \mu \| \delta \theta\|^2.
 \end{equation}
 By \eqref{eq:partial_theta_Hs} in Lemma \ref{lemma:Hs-prox-diff},
 and denoting  $T^+(T, \theta) = T^+$,  $T^+(T, \theta + \delta \theta) = T^+ + \delta T^+$,
we have that 
 \begin{align}
&~~~  - \partial_\theta H_s( T, \theta + \delta \theta) +   \partial_\theta H_s( T, \theta )  \nonumber \\
& =   - \partial_\theta H(  T^+ + \delta T^+, \theta + \delta \theta) 
	+   \partial_\theta H( T^+, \theta )   \nonumber  \\
& = \E_{x \sim P}
 	\big( \partial_\theta \ell( \theta+\delta \theta, (T^+ + \delta T^+)(x))
			- \partial_\theta \ell( \theta , T^+ (x)) \big)  \nonumber \\
& = 	 \E_{x \sim P}
	\big( \partial_{\theta \theta}^2 \ell( \theta, 	 T^+ (x)) \delta \theta
	+ \partial_{\theta v}^2 \ell( \theta, T^+(x) ) \delta T^+(x) \big),	
	\label{eq:proof-barmu-SC-1}
 \end{align}
where in the second equality we used \eqref{eq:partial_theta_H-1}.

 The expression of $\delta T^+(x)$ can be derived using \eqref{eq:partial_T_Hs} in Lemma \ref{lemma:Hs-prox-diff}:
For $P$-a.s. $x$, we have
\[
\frac{1}{s} (T(x) - T^+(x) ) = \partial_T H(  T^+, \theta )(x)
= -\partial_v \ell(\theta, T^+(x)) + \frac{1}{\gamma}( T^+(x) - x).
\]
Take perturbation on both sides,  the above equation implies that
 \[
\delta T^+(x)
 =  \big[ ( ( \frac{1}{\gamma}+\frac{1}{s}) I_d - \partial^2_{vv} \ell )^{-1}   \partial^2_{v \theta} \ell \big] |_{(\theta, T^+(x))}
 	 \delta \theta,
	 \quad P-a.s.
 \]
 Inserting back to \eqref{eq:proof-barmu-SC-1}, we have
\begin{align*}
&~~~
 - \partial_\theta H_s( T, \theta + \delta \theta) +   \partial_\theta H_s( T, \theta ) \\
& = \E_{x \sim P}
	 \big[ 
	\partial_{\theta \theta}^2 \ell + 
	 \partial_{\theta v}^2 \ell  ( ( \frac{1}{\gamma}+\frac{1}{s}) I_d - \partial^2_{vv} \ell )^{-1}   \partial^2_{v \theta} \ell \big] |_{(\theta, T^+(x))}
 	 \delta \theta.	
\end{align*}
 This means that \eqref{eq:proof-barmu-SC-goal} holds as long as the matrix
 $
 \E_{x \sim P}
	 \big[ 
	\partial_{\theta \theta}^2 \ell + 
	 \partial_{\theta v}^2 \ell  ( ( \frac{1}{\gamma}+\frac{1}{s}) I_d - \partial^2_{vv} \ell )^{-1}   \partial^2_{v \theta} \ell \big] |_{(\theta, T^+(x))}
	 \succeq
	\bar \mu I
 $, which is guaranteed when  $T \in \calT_{\bar \mu}^{\rm NC}$. 
 \end{proof}

\begin{proof}[Proof of Corollary \ref{cor:theta-fast-NC-NC}]
By Lemma \ref{lemma:Hs-prox-diff}, the update \eqref{eq:damped-PPM} is equivalent to 
\begin{equation*}
\begin{cases}
T_{k+1} 		\leftarrow T_k - \eta \partial_T H_s( T_k, \theta_k)  \\
\theta_{k+1}  	\leftarrow  \theta_k + \tau \partial_\theta H_s(T_k, \theta_k) , \\
\end{cases}
\end{equation*}
which is the GDA scheme \eqref{eq:GDA-M-theta-fast} applied to $M = H_s$. 
This allows us to prove the corollary by applying Theorem \ref{thm:theta-fast-NC-SC}
with $M = H_s$  and $ \calT = \calT_{\bar \mu}^{\rm NC}$,
and we verify that the needed assumptions on $M$ and $\calT$ are satisfied:
the  coordinate  $\bar l$-smoothness of $H_s$ is by Lemma \ref{lemma:theta-fast-lsmooth-Hs}, 
the $\bar \mu$-strongly concavity of $H_s$ in $\theta $ when $T \in  \calT_{\bar \mu}^{\rm NC}$  is by Lemma \ref{lemma:TNC-SC-Hs},
and the other requirements of $\calT$ and $T_k$ are ensured by Assumptions \ref{assump:NC-good-T},\ref{assump:Tgood-in-open-set-NC}.
\end{proof}

\section{Experimental details}\label{app:exp-more}

\subsection{2D regression example}

The data samples $x \sim P  = {\rm Unif}( [-1,1 ] \times [-1,1])$, and the true response 
$y^*(x) =e^{-\| x\|^2/2\sigma_y^2}$, $\sigma_y=0.5$.
An illustration of the data and true response is given in Figure \ref{fig:2d-regression-data}.

\begin{figure}[t]
\centering
\begin{minipage}{0.275\textwidth}
\includegraphics[height=.9\linewidth]{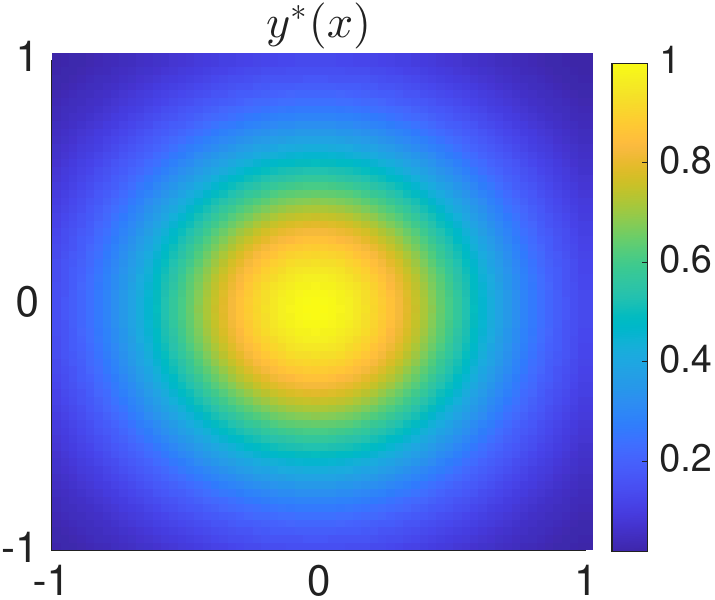}\subcaption{}
\end{minipage}\hspace{20pt}
\begin{minipage}{0.275\textwidth}
\includegraphics[height=.9\linewidth]{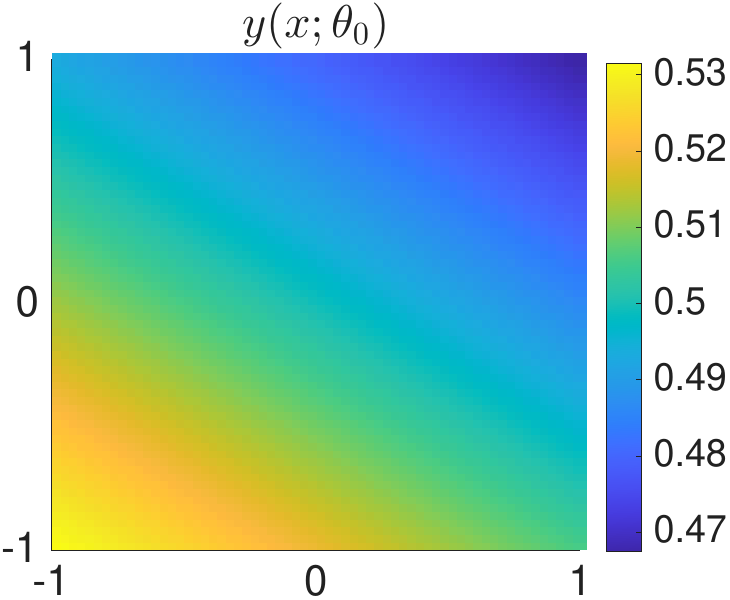}\subcaption{}
\end{minipage}\hspace{20pt}
\begin{minipage}{0.275\textwidth}
\includegraphics[height=.9\linewidth]{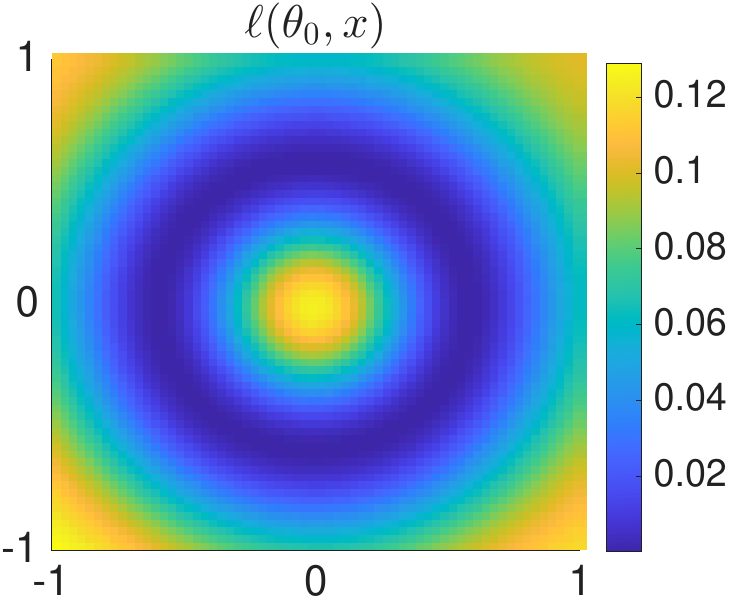}\subcaption{}
\end{minipage}
\caption{
$L^2$ regression loss  on 2D data.
(a) True response.
(b) Model function at initial $\theta_0$.
(c) Loss  function $\ell $ at initial $\theta_0$.
}
\label{fig:2d-regression-data}
\end{figure}

\subsection{General setup for image experiments}

The classifier $\theta$ and the transport maps are all trained in the latent space,
and we take $\theta^0$ as a pretrained classifier. 
For the $y$-specific transport maps $T_{y}$,
we parameterize them as a single neural network that implements label-conditioning in the architecture, sharing the trainable parameters $\varphi$ across $y \in [K]$.
We always initialize $T^0$ as the identity map, omitting the dependence on $y$ in the notation. This is implemented by computing $T_\varphi(x) = x + R_\varphi(x)$,
where $R_\varphi$ is a neural network parametrized by $\varphi$ and zero-initialized (by setting the final layer to be zero).

\paragraph{VAE latent space.}
We train convolutional VAEs to obtain smooth latent representations. The network architectures are adapted from \cite{rombach2022high}. The encoder and decoder follow a symmetric ResNet-style architecture with group normalization and SiLU activations. Downsampling uses average pooling and upsampling uses nearest-neighbor interpolation. More details are provided in each dataset-specific subsection.

\paragraph{Classifier $\theta$ and initialization.}
The decision model $\theta$ is an MLP classifier mapping a $d$-dimensional latent code $x$ to class logits $f_y(\theta, x)$ for $y\in[K]$. The architecture contains two hidden layers of width $2d$ with SiLU activations. We initialize $\theta$ by training it on the full training split of the corresponding image dataset using Adam with weight decay $\omega$. The resulting pretrained model serves as $\theta^0$ in our minimax optimization, where the same weight decay $\omega$ is applied inside the loss $\ell_y(\theta, x)$. Dataset-specific hyperparameters are given in the subsections below.

\paragraph{Class-specific loss $\ell$.}
We denote by $P_y$ the distribution of $x$ conditioned on class $y$, and the loss $\ell$ has the expression
$$
\ell_y(\theta,x) = -\log \frac{\exp(f_{y}(\theta, x))}{\sum_{j=1}^K  \exp(f_{j}(\theta, x))} + \frac{ \omega }{2} \|\theta\|^2,
$$ 
where $f_y( \theta, \cdot )$ is the logit of class $y$,
and $\omega$ is the weight decay parameter. 
Our methodology can be extended to solve for $y$-specific  $T_y$ for each class. 
Specifically, we solve for $\{ \theta, \, T_y, \, y \in [K]\}$ via the following minimax problem 
$$
  \min_{\theta\in\mathbb{R}^m} \max_{T_y \in L^2(P_y),y\in [K ]} 
 \sum_{y =1 }^K p_y
 \mathbb{E}_{x\sim P_y} \big[\ell_y(\theta, T_y(x)) - \frac{1}{2\gamma} \|T_y(x) - x\|^2\big],
$$
where $p_y$ is the (population) proportion of class $y$ in the dataset. With finite samples, the expectation in the above expression is naturally replaced by empirical averages, possibly on batches.  

\paragraph{Training of $T_y$ by matching loss.} The neural transport map $T_\varphi$ is trained using the matching loss \eqref{eq:loss-matching}, optimized using Adam. To represent $K$ class-specific maps $\{T_y\}_{y\in[K]}$ with one network, we condition $T_\varphi$ on the class label using a learnable label embedding concatenated to $x$. The resulting vector is fed into an MLP with two hidden layers of width $2d$ and SiLU activations, producing a residual added back to $x$.

\paragraph{GDA with momentum.}
We set the momentum parameter $\nu_{\rm m}$ and store over iterations indexed by $k$ the velocity vectors:
$h^k$ for gradient in $\theta$,
and $\{ g_i^k \}_i$ for gradients in $v_i$.
The update scheme is 
$$
\begin{gathered}
{g}_{i}^{k+1} \gets \nu_{\rm m} {g}_{i}^k + 
    \big( \partial_v  \ell( \theta^k, v_i^k) - \frac{1}{\gamma} (v_i^k - x_i) \big), 
    \quad v_i^{k+1} \gets v_i^k + \eta  {g}_{i}^{k+1}, \quad i\in B;
    \\
h^{k+1} \gets \nu_{\rm m}  h^k + \frac{1}{m}\sum_{i \in B} \partial_\theta  \ell( \theta^k, v^k_i), \quad \theta^{k+1} \gets  \theta^k  - \tau  h^{k+1}.
\end{gathered}
$$
Here, $g_i^k$ and $v_i^k$ are only changed if $i \in B$, similarly as in \eqref{eq:vi-GDA-3}.

\subsection{MNIST}

We provide  in Table \ref{tab:mnist-hyperparams} the detailed hyperparameter setup,
including the neural network architectures,
the training configurations of the VAE, classifier $\theta$, and neural transport map $T_\varphi$. 

\paragraph{VAE architecture.}
The encoder begins with a $3{\times}3$ convolution mapping the input to 64 channels, followed by three stages that progressively reduce spatial resolution. The first two stages each apply a \texttt{ResBlock} and then average pooling, reducing $28{\times}28 \rightarrow 14{\times}14 \rightarrow 7{\times}7$. At $7{\times}7$, a third \texttt{ResBlock} is followed by GroupNorm, SiLU, and a $4{\times}4$ convolution that produces a $4{\times}4$ feature map. A final \texttt{ResBlock}, GroupNorm, SiLU, and a $1{\times}1$ convolution output the mean and log-variance of the latent code of shape (2, 4, 4). Each \texttt{ResBlock} uses two $3{\times}3$ convolutions, GroupNorm with 16 groups, and SiLU activations.

The decoder mirrors this structure. A $1{\times}1$ convolution first lifts the latent tensor of shape $(2,4,4)$ to 64 channels, followed by a \texttt{ResBlock}. GroupNorm and SiLU precede a $4{\times}4$ transposed convolution that maps $4{\times}4$ back to $7{\times}7$ while reducing the channel width to 64. A \texttt{ResBlock} is applied at $7{\times}7$, followed by two nearest-neighbor upsampling steps ($7{\times}7 \rightarrow 14{\times}14 \rightarrow 28{\times}28$), each paired with a \texttt{ResBlock}. A final normalization, SiLU activation, and a $3{\times}3$ convolution produce the reconstructed grayscale image.

\begin{figure}[h]
\centering
\begin{subfigure}{0.32\linewidth}
\hspace{-8pt}
\includegraphics[width=\linewidth]{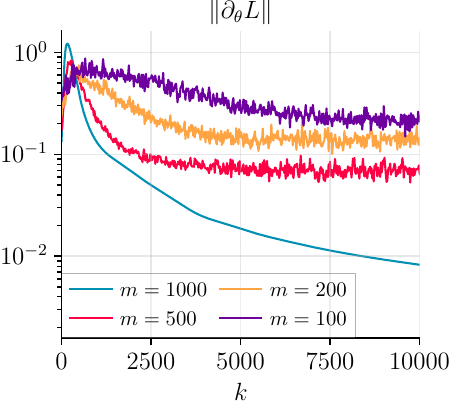}
\subcaption{}
\end{subfigure}
\begin{subfigure}{0.32\linewidth}
\hspace{-8pt}
\includegraphics[width=\linewidth]{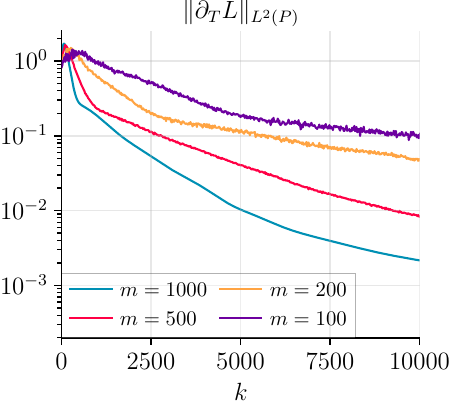}
\subcaption{}
\end{subfigure}
\begin{subfigure}{0.32\linewidth}
\hspace{-8pt}
\includegraphics[width=\linewidth]{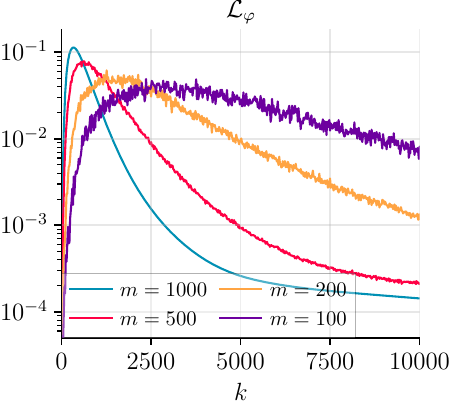}
\subcaption{}
\end{subfigure}
\caption{
GDA without momentum on MNIST. 
Same plots as in Figure \ref{fig:mnist-curves-momentum}, with all the other setups the same.}
\label{fig:mnist-curves-without-momentum}
\end{figure}

\begin{figure}[h]
\centering
\begin{subfigure}{0.32\linewidth}
\hspace{-8pt}
\includegraphics[width=\linewidth]{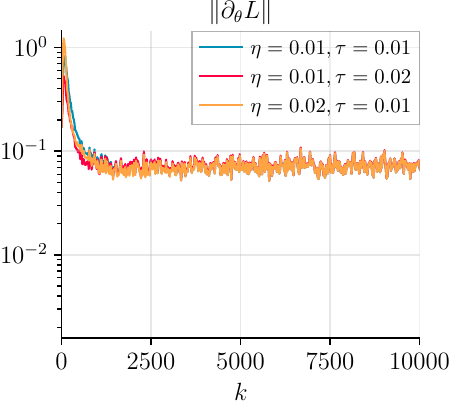}
\subcaption{}
\end{subfigure}
\begin{subfigure}{0.32\linewidth}
\hspace{-8pt}
\includegraphics[width=\linewidth]{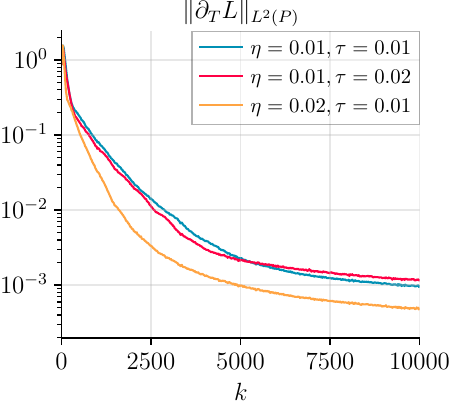}
\subcaption{}
\end{subfigure}
\begin{subfigure}{0.32\linewidth}
\hspace{-8pt}
\includegraphics[width=\linewidth]{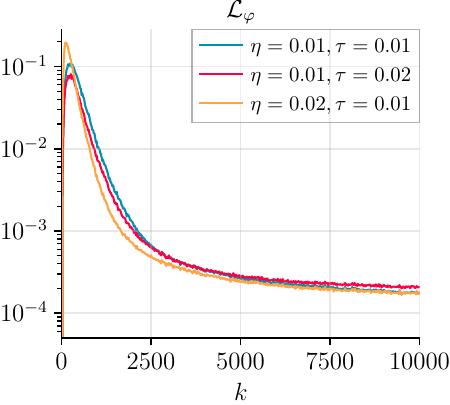}
\subcaption{}
\end{subfigure}
\caption{GDA with momentum and varying step sizes $\eta$ and $\tau$, with
$m=500$. Same plots as in Figure \ref{fig:mnist-curves-momentum}.}
\label{fig:mnist-curves-step-size}
\end{figure}

\begin{figure}[h]
    \centering
    \includegraphics[width=\textwidth]{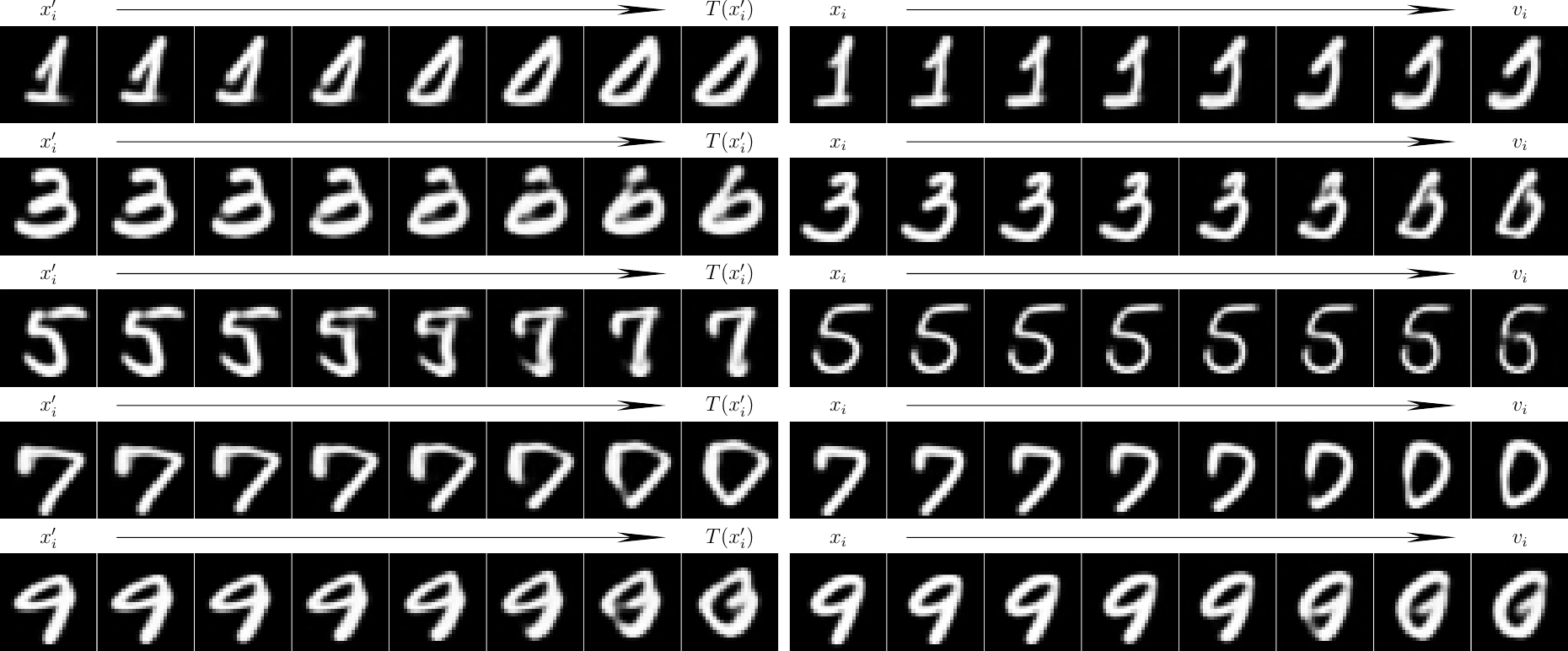}
    \caption{
    Interpolation trajectories
    computed on testing samples $x_i'$
    and the corresponding training sample $x_i$ (the nearest neighbor of $x_i'$ in the training set).
    Each row shows the deformation of $x_i'$ (from $x_i'$ to $T_\varphi(x_i')$) on the left, 
    and  that of $x_i$ (from $x_i$ to the optimized particle $v_i$) on the right.
    The $x_i'$s are selected in the same way as in Figure \ref{fig:interpolation-panels}(a), and the left block consists of the same plots as therein. 
    }
    \label{fig:train_vs_test}
\end{figure}

\paragraph{Ablation of momentum and varying step sizes.} 
Figure \ref{fig:mnist-curves-without-momentum} reports the same set of MNIST experiments as Figure \ref{fig:mnist-curves-momentum} but without momentum. All three metrics show a decreasing trend after a short warm-up phase, indicating that the GDA dynamics remain convergent. Compared with $\mathrm{momentum}=0.7$, the descent of all three metrics is slower. This confirms the role of momentum in accelerating convergence.

The same plots with varying step sizes  $\eta$ and $\tau$ are shown in Figure \ref{fig:mnist-curves-step-size}, with momentum 0.7. The GDA dynamics converge under the different choices, and the found minimax solutions and the learned neural transport map $T_\varphi$ are also similar (not shown).

\paragraph{Training vs. test interpolation trajectories.} For each test sample, we identify its nearest training particle in the latent space and compare their respective interpolation trajectories. The results are shown for $m=500$, $k=20{,}000$, the same as Figure \ref{fig:interpolation-panels}(a).
The visual alignment between the two paths in each row illustrates that the learned transport $T_\varphi$ generalizes continuously to unseen data.

\begin{table}[h]
\centering
\caption{MNIST hyperparameters for VAE, classifier, and transport map.}
\label{tab:mnist-hyperparams}

\begin{tabular}{p{3.5cm} p{4.5cm} p{4.5cm}}
\toprule
\textbf{Component} & \textbf{Hyperparameter} & \textbf{Value} \\
\midrule
\multirow{11}{*}{\textbf{VAE}}
& Latent dimension $d$             & 32 \\
& Hidden channels                 & [64, 64, 64, 64] \\
& Normalization layers             & GroupNorm with 16 groups \\
& Reconstruction loss              & $L^2$ \\
& KL regularization weight         & $10^{-2}$ \\
& Data augmentation                & Random affine transforms \\
& Training epochs                  & 200 \\
& Batch size                       & 600 \\
& Optimizer                        & AdamW \\
& Learning rate                    & $5\times10^{-3}$ (cosine schedule) \\
& Weight decay                     & $10^{-4}$ \\
& EMA decay                        & 0.999 \\
\midrule
\multirow{6}{*}{\textbf{Classifier $\theta$}}
& MLP hidden width                 & 64 \\
& Weight decay $\omega$            & $10^{-2}$ \\
& Optimizer for $\theta_0$         & Adam \\
& Learning rate for $\theta_0$     & $10^{-3}$ \\
& Training batches for $\theta_0$  & 20,000 \\
& Batch size for $\theta_0$        & 600 \\
\midrule
\multirow{5}{*}{\textbf{Transport Map $T_\varphi$}}
& MLP hidden width                 & 64 \\
& Embedding size for $y$           & 64 \\
& Optimizer                        & Adam \\
& Learning rate                    & $10^{-4}$ \\
& Weight decay                     & $10^{-5}$ \\
\bottomrule
\end{tabular}
\end{table}

\subsection{CIFAR-10}

\begin{table}[h]
\centering
\caption{CIFAR-10 hyperparameters for VAE, classifier, and transport map.}
\label{tab:cifar-hyperparams}

\begin{tabular}{p{3.5cm} p{4.5cm} p{4.5cm}}
\toprule
\textbf{Component} & \textbf{Hyperparameter} & \textbf{Value} \\
\midrule

\multirow{11}{*}{\textbf{VAE}}
& Latent dimension $d$               & 256 \\
& Hidden channels                    & [128, 256, 512, 512] \\
& Normalization layers               & GroupNorm with 32 groups \\
& Reconstruction loss                & $L^1$ \\
& Perceptual loss                    & LPIPS (weight 0.2) \\
& KL regularization weight           & $10^{-2}$ \\
& Data augmentation                  & Horizontal flip, color jitter \\
& Training epochs                    & 1,000 \\
& Batch size                         & 250 \\
& Optimizer                          & AdamW \\
& Learning rate                      & $10^{-4}$ (cosine schedule) \\ 
& Weight decay                       & $10^{-4}$ \\
& EMA decay                          & 0.999 \\
\midrule

\multirow{6}{*}{\textbf{Classifier $\theta$}}
& MLP hidden width                   & 512 \\
& Weight decay $\omega$              & $10^{-3}$ \\
& Optimizer for $\theta_0$           & Adam \\
& Learning rate for $\theta_0$       & $10^{-3}$ \\
& Training batches for $\theta_0$    & 20,000 \\
& Batch size for $\theta_0$          & 500 \\
\midrule

\multirow{5}{*}{\textbf{Transport Map $T_\varphi$}}
& MLP hidden width                   & 512 \\
& Embedding size for $y$             & 512 \\
& Optimizer                          & Adam \\
& Learning rate                      & $10^{-4}$ \\
& Weight decay                       & $10^{-5}$ \\
\bottomrule
\end{tabular}
\end{table}

The detailed hyperparameter setup can be found in Table \ref{tab:cifar-hyperparams}. 

\paragraph{VAE architecture.}
The encoder begins with a $3{\times}3$ convolution mapping the input to 128 channels, followed by three stages that reduce the spatial resolution from $32{\times}32 \rightarrow 16{\times}16 \rightarrow 8{\times}8 \rightarrow 4{\times}4$. Each stage applies two \texttt{ResBlock} layers and then average pooling. The \texttt{ResBlock} design matches that used in MNIST but uses GroupNorm with 32 groups, and a $1{\times}1$ shortcut whenever the input and output channels differ (e.g., $128\!\to\!256$, $256\!\to\!512$). After reaching $4{\times}4$, an additional \texttt{ResBlock}, followed by GroupNorm, SiLU, and a $1{\times}1$ convolution, outputs the mean and log-variance of the latent code of shape $(16,4,4)$.

In the decoder, a $1{\times}1$ convolution first lifts the latent tensor to 512 channels, followed by a \texttt{ResBlock}. Three nearest-neighbor upsampling steps then increase the spatial resolution ($4{\times}4 \rightarrow 8{\times}8 \rightarrow 16{\times}16 \rightarrow 32{\times}32$), each followed by two \texttt{ResBlock} layers whose output channels decrease according to the encoder widths. A final GroupNorm, SiLU, and a $3{\times}3$ convolution produce the reconstructed RGB image.

\end{document}